\documentclass{article}
\usepackage{iclr2027_conference,times}

\usepackage[T1]{fontenc}
\usepackage{microtype}
\usepackage{amsmath,amssymb,amsthm,mathtools}
\usepackage{newtxmath}

\usepackage{graphicx}
\usepackage{array,booktabs,tabularx}
\usepackage{algorithm,algorithmic}
\usepackage{float,placeins}
\usepackage{enumitem}
\graphicspath{{figures/}}

\usepackage{xcolor}
\usepackage{url}
\usepackage{hyperref}
\usepackage[capitalize,noabbrev]{cleveref}
\definecolor{CitationTeal}{HTML}{43A1AE}
\definecolor{LinkOrange}{HTML}{FDA614}
\hypersetup{
  colorlinks=true,
  citecolor=CitationTeal,
  linkcolor=LinkOrange,
  urlcolor=LinkOrange,
  filecolor=LinkOrange
}

\newtheorem{theorem}{Theorem}
\newtheorem{proposition}{Proposition}
\newtheorem{lemma}{Lemma}
\newtheorem{corollary}{Corollary}
\theoremstyle{definition}

\newcommand{\FedASTRA}{\textsc{FedASTRA}}

\newcommand{\appendixsectionentry}[1]{%
  \par\noindent
  \hyperref[#1]{\textbf{\makebox[2.15em][l]{\ref*{#1}}\nameref*{#1}}}%
  \nobreak\dotfill\nobreak
  \hyperref[#1]{\pageref*{#1}}\par
}
\newcommand{\appendixsubsectionentry}[1]{%
  \par\noindent\hspace*{1.55em}%
  \hyperref[#1]{\makebox[3.2em][l]{\ref*{#1}}\nameref*{#1}}%
  \nobreak\dotfill\nobreak
  \hyperref[#1]{\pageref*{#1}}\par
}

\title{More Is Better? Exploring Latent Correspondence in Federated Hypergraph Learning}

\author{
  \textbf{Bo Yan}$^1$, \textbf{Siming Zhang}$^2$, \textbf{Siqi Li}$^1$, \textbf{Hang Yu}$^2$\\
  $^1$Tsinghua University, $^2$Shanghai University
}

\iclrfinalcopy 
\begin{document}
\maketitle

\begin{abstract}
Federated hypergraph learning with latent cross-client hyperedge correspondence poses fundamental identifiability and computational challenges, as the same local release can correspond to multiple global hypergraph structures. 
We characterize the tractability of recovering these latent correspondences, showing tractable pairwise assignment but generally intractable multi-client hyperedge matching.
Building on this characterization, we derive a cross-entropy decomposition that reveals when local prediction can outperform predictors fitted with additional cross-client information. 
Motivated by this principle, we introduce \textbf{Fed}erated \textbf{A}daptive \textbf{S}pectral \textbf{T}raining with \textbf{R}epresentation \textbf{A}lignment (\textbf{FedASTRA}), which learns correspondence-invariant diffusion representations through adaptive heat diffusion and masked latent prediction while communicating model parameters only. We establish its asymptotic optimality with respect to the permitted local information and prove global population-risk convergence under explicit conditions on recovery, representation, and update alignment. Experiments on eight real-world datasets demonstrate the effectiveness of FedASTRA against identifier-free federated graph and hypergraph methods. 
On ogbn-papers100M, FedASTRA reaches the performance-communication Pareto frontier, achieving 50.07 Accuracy and 29.73 Macro-F1 while communicating only 59.60 MiB per round.
\end{abstract}
\section{Introduction}
\label{sec:introduction}

Federated learning enables collaborative model training without sharing data records \citep{mcmahan2017communication}. For applications involving group-wise interactions, hypergraphs provide a natural representation of memberships in care teams, coauthor sets, and biochemical interactions \citep{feng2019hypergraph,chien2022allset,wang2023edhnn}.
However, when hypergraph data are anonymized across clients, local memberships may remain observable while their cross-client correspondence is lost, making the underlying global hypergraph potentially non-identifiable from local observations alone.
Consider two hospitals that share care episodes, as illustrated in Figure \ref{fig:concept-overview}. After episode identifiers are removed, Hospital A retains anonymous records of backache and BP/ECG, while Hospital B retains records of surgery and headache/drug. Same-order matching yields compatible episodes of sizes two and four, whereas cross-order matching yields two episodes of size three. Thus, even this minimal two-client example admits distinct global compatible structures consistent with the same local observations, raising a natural question: \textit{which global hypergraph properties remain identifiable under latent hyperedge correspondence?}

We answer this question by establishing sharp identification bounds and computational boundaries under latent hyperedge correspondence. 
For two clients, the compatible global structures reduce to bipartite assignments, yielding sharp identification bounds and a double-centering criterion for point identification.
Beyond two clients, correspondence becomes a multiway hypergraph-matching problem. 
It is NP-complete for every fixed client count above two, yet fixed-parameter tractable jointly in client count and matching size. 
When the client count is unbounded, it becomes W[1]-complete when parameterized only by the matching size.

\begin{figure}[t!]
  \centering
  \includegraphics[width=0.96\textwidth]{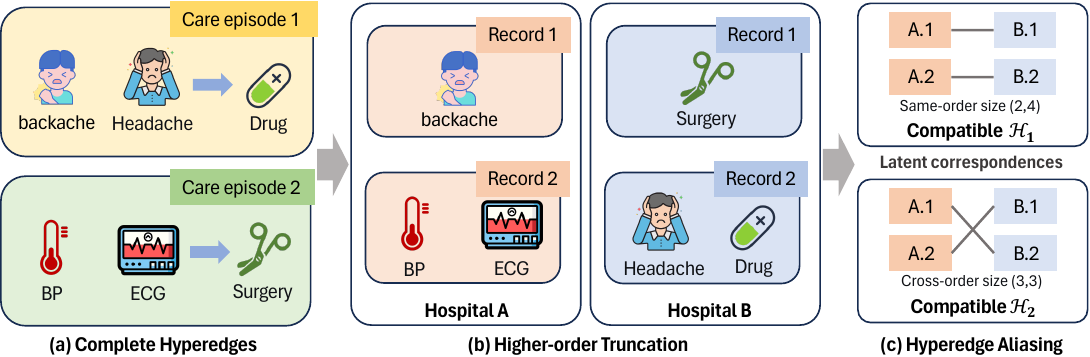}
  \caption{\textbf{From complete hyperedges to latent hyperedge correspondences.} (a) A complete care episode is represented as a hyperedge connecting its clinical events, e.g., blood pressure (BP), electrocardiogram (ECG), and surgery. (b) A episode may span multiple hospitals, e.g., BP measurement and ECG examination may be performed at hospital A, while the subsequent surgery is carried out at hospital B. Each hospital observes only a local fragment of the original hyperedge, leading to \textbf{higher-order truncation}. (c) Once cross-hospital record correspondences are latent, the same local releases can admit multiple compatible global matchings, resulting in \textbf{hyperedge aliasing}.}
  \label{fig:concept-overview}
\end{figure}

This characterization raises another question: \textit{when does additional cross-client information improve prediction, and when can exploiting it be detrimental?}
Additional information may provide useful label signal, but uncertain correspondence can also increase prediction variability and fitting error. We decompose cross-entropy risk into information gain and fitting error, yielding a sufficient condition under which a local predictor outperforms a fitted predictor with additional information under the same evaluation law.

Motivated by this tradeoff, together with the communication and privacy costs of richer cross-client exchanges \citep{shi2025scalemia}, we introduce FedASTRA, which learns correspondence-invariant local diffusion representations through adaptive heat filtering and masked latent prediction \citep{lecun2022path,assran2023ijepa}, with a client-local exponential moving-average (EMA) target and parameter-only FedAvg \citep{mcmahan2017communication}.
We show that effectively exploiting the task-relevant local information can be sufficient without resolving latent cross-client correspondences or completing the global hypergraph, yielding asymptotic optimality relative to the permitted local information.
Under explicit representation and update-alignment conditions, we further establish global population-risk convergence along the implemented federated updates.

Our contributions are fourfold. \textbf{First,} we establish sharp identification bounds and a computational hierarchy for latent hyperedge correspondence: polynomial-time two-client assignment, NP-complete yet fixed-parameter tractable matching for any fixed client count above two, and W[1]-complete matching when the client count is unbounded. \textbf{Second,} we characterize when additional cross-client information helps or hurts prediction through a cross-entropy decomposition that connects information gain, fitting error, prediction variation, and assignment geometry. \textbf{Third,} we develop FedASTRA and establish asymptotic optimality relative to the permitted local information, together with global population-risk convergence under explicit representation and update-alignment conditions. \textbf{Fourth,} experiments on eight real-world datasets demonstrate consistent improvements over identifier-free federated graph and hypergraph methods; on ogbn-papers100M, FedASTRA reaches the performance-communication Pareto frontier.

\section{Related Work}
\label{sec:related}

\paragraph{Federated graph learning.}
Existing federated graph learning methods differ in whether cross-client graph correspondence is available. 
Identifier-free methods leave cross-client node identities latent \citep{zhang2021fedsage, DBLP:conf/icml/BaekJJYH23, DBLP:journals/pvldb/LiWZZLW23, fu2025fedath}, as in FedSage \citep{zhang2021fedsage} operating on local subgraphs, while FedSage+ \citep{zhang2021fedsage} generates cross-client neighbor features. 
Methods with explicit correspondence, such as FedGCN \citep{yao2023fedgcn} and FedCog \citep{lei2023fedcog}, exploit shared cross-client node or edge identities to incorporate additional graph information. For hypergraph, FedHGL \citep{luo2025fedhgl} instead assumes shared hyperedge identifiers and utilzes local differential privacy (LDP) \citep{DBLP:journals/siamcomp/KasiviswanathanLNRS11} to support structure completion. 
Despite progress, these methods trade cross-client communication and privacy exposure for predictive performance. 
In contrast, we study an identifier-free setting in which cross-client hyperedge correspondence is latent and only local incidence is available.

\paragraph{Hypergraph representation learning.}
Most hypergraph neural networks assume observed node-hyperedge memberships and focus on how to encode higher-order structure, including spectral propagation on hypergraph operators \citep{feng2019hypergraph,wang2023edhnn}, graph-reduction approaches that approximate hyperedges with pairwise structures \citep{yadati2019hypergcn}, and direct higher-order aggregation through message passing or permutation-invariant set functions \citep{dong2020hnhn,chien2022allset,huang2021unignn}.
Beyond supervised learning, self-supervised methods learn transferable hypergraph representations through multi-level contrastive learning \citep{DBLP:conf/aaai/LeeS23} and generative pretext tasks \citep{DBLP:conf/iclr/KimKBLYS24}.
More recently, large-scale hypergraph pretraining and foundation models have emerged for cross-task and cross-domain transfer \cite{DBLP:journals/pami/GaoFLHDWH26}. 
FedASTRA targets a more realistic federated scenario by learning effective representations from anonymous client-local incidence when cross-client hyperedge correspondence remains latent.

\section{Learning Difficulty under Latent Hyperedge Correspondence}
\label{sec:method}

\subsection{Observation model}
\label{sec:preliminaries}

Let $\mathcal H=(V,\mathcal E,X,r,w)$ be a finite node-attributed, typed, weighted multihypergraph. 
Here $V$ is the node set, $\mathcal E$ is the multiset of hyperedge instances, $X$ is the node-feature matrix, and $r$ and $w$ collect hyperedge types and weights. 
hyperedge instance $e\in\mathcal E$ has a support node set $A_e\subseteq V$, a type $r_e$ in the finite type set $\mathcal R$, and a positive weight $w_e$. 
Multiple instances may share the same support and type, and a type label carries no instance identity. 
For an integer $1\le q\le |V|$, let $\binom Vq$ be the collection of $q$-node subsets. Aggregating all hyperedge instances with the same support $A$ and type $r$, we define
\begin{align}
  w_r(A)&=\sum_{e:A_e=A,\,r_e=r}w_e,
  \label{eq:aggregate_weight}\\
  M_{q,r}^{\mathcal H}(S)&=\sum_{A\supseteq S}w_r(A),
  \qquad r\in\mathcal R,\quad S\in\binom Vq.
  \label{eq:q_measure}
\end{align}
Thus $M_{q,r}^{\mathcal H}(S)$ is the weighted inclusion count of the $q$-node subset $S$ among hyperedges of type $r$.
The full collection $\{M_{q,r}^{\mathcal H}(S)\}$ is the invertible Boolean-zeta transform of the aggregate weights \{$w_r(A)$\}. 
In particular, the $q=2$ and $q=3$ coordinates capture observable pairwise and triple co-membership patterns before client partitioning. The notation collection is provided in Appendix~\ref{app:notation}.

A fixed node-to-client assignment $z:V\to[K]$, with $k\in[K]:=\{1,\ldots,K\}$, partitions $V$ into local node set $V_k=\{v\in V:z(v)=k\}$ for each client k and induces local hyperedge fragments $F_{e,k}=A_e\cap V_k$. 
At a visibility threshold $c\ge 1$, client $k$ retains each fragment $F_{e,k}$ with $|F_{e,k}|\ge c$, together with its type $r_e$, weight $w_e$, and a status indicator $\upsilon_{e,k}\in\{0,1\}$ specifying whether hyperedge $e$ is fully contained in client $k$. Crucially, cross-client hyperedge identifiers are withheld.
Let $\Gamma$ collect the public assumptions governing the observation process, including the visibility rule and any exact constraints on the number of retained fragments. 
Write $[T_k]_{\rm row}$ for the equivalence class of client $k$'s retained fragment table under arbitrary row permutations. 
The analyst release is then defined by the deterministic observation map:
\begin{equation}
Y_{\rm A}
=
\mathcal O_{z,\Gamma}(\mathcal H)
=
\left(
\left(
(V_k,X_{V_k}),
\left[(F_{e,k},r_e,w_e,\upsilon_{e,k})
:\ |F_{e,k}|\ge c\right]_{\rm row}
\right)_{k=1}^K,
\Gamma
\right).
\label{eq:observation}
\end{equation}
Uppercase $Y_{\rm A}$ denotes the analyst release as a random element and lowercase $y_{\rm A}$ a realized release. 
Within the latent hypergraph class $\mathcal M_\Gamma$ permitted by the public assumptions $\Gamma$, define the analyst compatibility class
\begin{equation}
  \mathcal F_{\rm A}(y_{\rm A},\Gamma)=
  \{\mathcal H'\in\mathcal M_\Gamma:
  \mathcal O_{z,\Gamma}(\mathcal H')=y_{\rm A}\}.
  \label{eq:compatible_set}
\end{equation}
In contrast to the richer analyst release, the learner, representing the actual training algorithm, receives a coarsened release.
Accordingly, let $B_k$ be the binary node-fragment incidence matrix induced by $T_
k$, with one column for each row of $T_k$. 
Define $[B_k]_{\rm col}=\{B_k\Pi:\Pi\text{ is a column permutation}\}$ and thus arbitrary row permutations of $T_k$ correspond to column permutations $B_k\mapsto B_k\Pi$. 
The learner coarsened release is therefore
\begin{equation}
  Y_{\rm L}=\rho(Y_{\rm A})=
  \bigl((X_k,[B_k]_{\rm col})\bigr)_{k=1}^K,
  \qquad X_k:=X_{V_k}.
  \label{eq:learner_observation}
\end{equation}
It retains only local node features and anonymous incidence structure, omitting hyperedge types, weights, status indicators, and row identities. 
For a fixed training protocol $\mathcal P$ containing the released labels, task split, participation schedule, and hyperparameters, and algorithmic randomness $\xi$, any learner operating under this observation model is a measurable function of $(Y_{\rm L},\mathcal P,\xi)$.
In native comparisons, the same realized release is provided to all learners.

For each $q$-node subset $S$, the client partition determines whether its co-membership information is retained. For $q\ge c$, we define
\begin{equation}
\widetilde M_{q,r}(S,z)
=M_{q,r}^{\mathcal H}(S)\mathbf 1\{z\text{ is constant on }S\},
\label{eq:mask_law}\\
\end{equation}
so the contribution of $S$ is retained exactly when all nodes in $S$ are assigned to the same client. We quantify the erased fraction by
\begin{equation}
    \tau_q(\mathcal H,z)
  =1-\frac{\sum_{r\in\mathcal R}\|\widetilde M_{q,r}(\cdot,z)\|_1}
  {\sum_{r\in\mathcal R}\|M_{q,r}^{\mathcal H}\|_1},
  \label{eq:tau}
\end{equation}
with $\tau_q(\mathcal H,z)=0$ for a zero denominator. We call the erased mass
\emph{higher-order truncation}, as illustrated in Figure~\ref{fig:concept-overview}. 
Independently permuted client fragment tables may also admit multiple compatible cross-client matchings. We call this ambiguity \emph{hyperedge aliasing}. 
These two effects can alter global higher-order statistics. We therefore consider a class of linear co-membership targets $\Phi_{q,\psi}$. 
For any finite type-indexed test function $\psi_{q,r}$, define
\begin{equation}
  \Phi_{q,\psi}(\mathcal H)=\sum_{e\in\mathcal E}w_e
  \sum_{S\in\binom{A_e}{q}}\psi_{q,r_e}(S,X_S).
  \label{eq:linear_q_target}
\end{equation}
Here, $\psi_{q,r}(S,X_S)$ specifies the contribution of each $q$-node subset $S$, allowing $\Phi_{q,\psi}$ to cover both total order-$q$ co-membership mass and more general feature-dependent higher-order statistics.

Equation~\eqref{eq:tau} quantifies the co-membership mass loss under a fixed partition. Appendices~\ref{app:framework}, \ref{app:proof-zeta}--\ref{app:proof-ht} give the Boolean-zeta inversion and show how the co-membership quantities can be estimated under a randomized partition using Horvitz-Thompson estimation \citep{horvitz1952generalization}.

\subsection{Exact two-client compatibility sets}
\label{sec:identifiability}

Two clients form the smallest nontrivial setting with an exact correspondence analysis. 
We use the richer analyst release with observed status indicators $\upsilon_{e,k}$, and exact type $r$ and weight $\omega$, together with $c=1$ and equal numbers of retained fragment rows across the two clients. Thus, each cross-client hyperedge leaves one nonempty row per client.
Let $m_h$ denote the common row count in stratum $h=(r,\omega)$.
Within stratum $h$, fragments $F_{1i}^h\subseteq V_1$ and $F_{2j}^h\subseteq V_2$, indexed by $i,j\in[m_h]$, are matched by an unobserved $\varpi_h\in\operatorname{Sym}(m_h)$, the symmetric group on $[m_h]$.
Let $\Phi_{\rm obs}$ collect the client-local and within-fragment contributions that are invariant across all compatible matchings. For a candidate cross-client pair $(F^h_{1i}, F^h_{2j})$, the mixed matching-dependent contribution is
\begin{equation}
  C^h_{ij}=\omega\!\!\sum_{\substack{S\in\binom{F_{1i}^h\cup F_{2j}^h}{q}:\\
      S\cap F_{1i}^h\ne\varnothing,\ S\cap F_{2j}^h\ne\varnothing}}
  \psi_{q,r}(S,X_S).
  \label{eq:completion_cost}
\end{equation}
The cost matrix $C^h$ exposes exactly which part of the target $\Phi_{q,\psi}$ depends on the missing correspondence. 
Define the identified set of $\Phi_{q,\psi}$ under the analyst release as its attainable values over all compatible latent hypergraphs:
$\mathcal I_{q,\psi}(y_{\rm A})=
  \{\Phi_{q,\psi}(\mathcal H'):
  \mathcal H'\in\mathcal F_{\rm A}(y_{\rm A},\Gamma)\}$.

\begin{theorem}[Compatible target values under an unobserved permutation]
  \label{thm:completion_polytope}
  Assume $K=2$, $c=1$, observed status indicators, exactly one nonempty
  cross fragment at each client for every cross-client hyperedge, exact
  type-and-weight strata, and equal row counts within every stratum.  Then the compatible values and their sharp endpoints are
  \begin{align}
    \mathcal I_{q,\psi}(y_{\rm A})
    &=\left\{\Phi_{\rm obs}+\sum_h\sum_{i=1}^{m_h}
    C^h_{i,\varpi_h(i)}:
    \varpi_h\in\operatorname{Sym}(m_h)\right\},
    \label{eq:identified_functional_set}\\
    \Phi^-_{q,\psi}
    &=\Phi_{\rm obs}+\sum_h\min_{\Lambda\in\mathbb B_{m_h}}
    \langle C^h,\Lambda\rangle_F,
    \quad
    \Phi^+_{q,\psi}
    =\Phi_{\rm obs}+\sum_h\max_{\Lambda\in\mathbb B_{m_h}}
    \langle C^h,\Lambda\rangle_F,
    \label{eq:birkhoff_endpoints}
  \end{align}
  where $\mathbb B_m=\{\Lambda\in\mathbb R_+^{m\times m}: \Lambda\mathbf1=\mathbf1,\Lambda^\top\mathbf1=\mathbf1\}$ is the Birkhoff polytope and $\langle A,B\rangle_F=\operatorname{tr}(A^\top B)$ is the Frobenius inner product. Each condition below characterizes point identification within a stratum, independently of the probability distribution over matchings compatible with the observed release:
  \begin{align}
    J_{m_h}C^hJ_{m_h}&=0,
    \label{eq:centered_cost}\\[-1mm]
    C^h_{ij}&=\alpha_i^h+\beta_j^h,
    \label{eq:additive_cost}\\[-1mm]
    C^h_{ij}+C^h_{i'j'}-C^h_{ij'}-C^h_{i'j}&=0
    \quad\text{for all }i,i',j,j'\in[m_h],
    \label{eq:cycle_identification}
  \end{align}
  with centering matrix $J_m=I_m-\mathbf1\mathbf1^\top/m$ and vectors $\alpha^h,\beta^h\in\mathbb R^{m_h}$ in Equation~\eqref{eq:additive_cost}.
\end{theorem}

Each compatible assignment selects one entry from every row and column of $C^h$, thereby forming a permutation matrix. 
By Birkhoff-von Neumann integrality \cite{birkhoff1946tres}, both endpoints are attainable.  
Double centering in Equation~\eqref{eq:centered_cost} removes additive row and column effects, leaving only the interaction that can vary with the latent correspondence, in a subspace of dimension at most $(m_h-1)^2$. 
Equation~\eqref{eq:cycle_identification} checks whether any such interaction remains. 
Appendix~\ref{app:proof-completion} gives the proof and dual certificates. Appendix~\ref{app:two-client-boundary} further discusses weaker observation assumptions.

Coarsening the analyst release enlarges the compatible class. 
Hence, if $\phi_-$ and $\phi_+$ are attainable endpoints under the richer release, any randomized scalar estimator using the coarser learner release has worst-case absolute risk at least $(\phi_+-\phi_-)/2$ (Appendix~\ref{app:proof-two-point}). 
This diagnostic uses two clients, while model training uses $K\in\{3,6,9\}$ and correspondence-invariance result holds for any $K$.
Prediction can still be accurate when the local label signal survives unresolved global structure.

\paragraph{Graph matching versus hypergraph matching.}
Total order-$q$ mass admits sharp sorting bounds in $O(\sum_hm_h\log m_h)$ time
(Appendices~\ref{app:proof-qmass}--\ref{app:proof-sorting}). To isolate
relation size, consider $K$ clients with $m$ distinct singleton rows each and explicitly listed
rewarded tuples $\mathcal E_+\subseteq[m]^K$. Each listed tuple has reward one, all others
zero. A completion with at least $b$ rewarded relations exists exactly when these candidate
$K$-node hyperedges contain a matching of size $b$. For $K=2$, this is polynomial bipartite
graph matching. For fixed $K\ge3$, it is complete for nondeterministic polynomial time
(NP-complete) but fixed-parameter tractable (FPT) in $Kb$, meaning time
$f(Kb)\operatorname{poly}(L_{\rm in})$ for input length $L_{\rm in}$. With unbounded $K$, it is
complete for W[1], the first level of the W hierarchy, in $b$
(Appendix~\ref{app:three-client-boundary}). This separates pairwise from multiway correspondence
under the same observation and reward rules, independently of the optimization difficulty
of a neural classifier.

\section{Benefit Bounds of Additional Information}
\label{sec:information-comparison}

Given these fundamental identifiability and computational challenges under hyperedge correspondence, a natural thought is to exploit more information beyond local observations for better prediction. Yet a subtle tradeoff arises here: additional information can reveal label signal, while correspondence uncertainty can induce prediction variation and fitting error. In other words, its benefit is bounded. In this section, we characterize these benefit bounds through cross-entropy decomposition.

\paragraph{Information value and fitting cost.}
For class $y\in[C]=\{1,\ldots,C\}$, $C\ge2$, predictor $F$ uses $\mathcal S_0$:
client attributes, anonymous incidence, and the parameter-only transcript through validation,
excluding held-out labels. Predictor $G$ additionally uses $A_{\rm aux}$ and its derived
training states, giving $\mathcal S_1=\mathcal S_0\vee\sigma(A_{\rm aux})$.
Let $p_i=\Pr(y=\cdot\mid\mathcal S_i)$, $i=0,1$, and $q_F,q_G$ be their predictions.
Kullback--Leibler (KL) divergence $D_{\rm KL}(p\|q)=\sum_c p_c\log(p_c/q_c)$ yields
the information and fitting terms
\begin{equation}
  \mathcal I_{\rm add}=\mathbb E D_{\rm KL}(p_1\|p_0),\qquad
  \mathcal E_F=\mathbb E D_{\rm KL}(p_0\|q_F),\qquad
  \mathcal E_G=\mathbb E D_{\rm KL}(p_1\|q_G),
  \label{eq:main-information-cost}
\end{equation}
Then $\mathcal I_{\rm add}$ can measure new label information and $\mathcal E_F,\mathcal E_G$
measure fitting error to the respective optima. For measurable predictions with finite
log risk $\mathcal R_{\log}(q)=\mathbb E[-\log q_y]$, conditional entropy yields the exact
risk decomposition below, whose derivation is given in Appendix~\ref{app:auxiliary-risk}:
\begin{equation}
  \mathcal R_{\log}(q_G)-\mathcal R_{\log}(q_F)
  =\underbrace{\mathcal E_G-\mathcal E_F}_{\text{difference in fitting error}}
  -\underbrace{\mathcal I_{\rm add}}_{\text{additional label information}}.
  \label{eq:main-local-advantage}
\end{equation}

Thus $F$ has strictly lower risk exactly when $\mathcal E_G-\mathcal E_F>\mathcal I_{\rm add}$,
with the reverse strict inequality favoring $G$. Additional information weakly improves
ideal risk because it can be ignored \citep{blackwell1953experiments}. For fitted models,
however, its gain competes with approximation, estimation, and optimization errors.
When $p_1=p_0$, fitting quality alone determines which predictor has lower risk.

\paragraph{How correspondence creates a prediction cost.}
To quantify this cost, remove the class mean from enriched logits $s_G\in\mathbb R^C$
and set $\bar s_G=\mathbb E[s_G\mid\mathcal S_0]$. For cross-entropy risk
$\mathcal R_{\rm CE}(s)=\mathcal R_{\log}(\operatorname{softmax}(s))$
and finite-risk local logits $s_F$, define
$V_{\rm aux}=\mathbb E\|s_G-\bar s_G\|_2^2$,
$\Delta_{\rm label}=\mathbb E\|p_1-p_0\|_2^2$, and
$\delta_F=\mathcal R_{\rm CE}(s_F)-\mathcal R_{\rm CE}(\bar s_G)$.
Under $\|s_G\|_\infty\le B_{\rm logit}<\infty$ almost surely, Theorem~\ref{thm:auxiliary-ce-risk} gives
\begin{equation}
  \boxed{
  \frac{e^{-2B_{\rm logit}}}{2C}V_{\rm aux}
  >\sqrt{\Delta_{\rm label}V_{\rm aux}}+\delta_F
  \quad\Longrightarrow\quad
  \mathcal R_{\rm CE}(s_G)>\mathcal R_{\rm CE}(s_F).
  }
  \label{eq:main-strict-local-criterion}
\end{equation}
Hence local prediction has lower risk when variation cost exceeds the label-alignment
bound plus the signed fitting gap. The criterion is sufficient, since Cauchy--Schwarz
bounds label alignment. When $\Delta_{\rm label}=0$ and $\delta_F\le0$, any
$V_{\rm aux}>0$ certifies a strict local advantage.
Moreover, Proposition~\ref{prop:assignment-logit-variance}
computes this variation from the double-centered matching costs of Theorem~\ref{thm:completion_polytope}
for affine logits under independent uniform matchings, linking the identification result
to this prediction-risk comparison.

\paragraph{Communication and privacy boundary.}
Additional exchanges incur communication costs and can expose sensitive information
\citep{shi2025scalemia}. Parameter-only exchange avoids these auxiliary disclosures,
although model updates still require protection against information leakage.

\section{FedASTRA: Learning from Local Structure}
\label{sec:fedastra}

Section~\ref{sec:information-comparison} yields a design principle: local prediction can outperform
a fitted richer-information model when its fitting cost exceeds the label-information gain.
Guided by this principle, \FedASTRA{} stays within the observable local release: it learns diffusion
scales from anonymous incidence, predicts masked representations, and communicates parameters only,
eliminating the need for cross-client pairing. Under probabilistic recovery and approximation conditions, FedASTRA is asymptotically optimal relative
to permitted local information. Under representation and task-alignment conditions, its actual updates
achieve global population-risk convergence with an $O(1/T)$ optimization term.

\paragraph{Learning diffusion scales.}
For binary incidence $B_k$ with nonempty columns, let
$D_{v,k}=\operatorname{diag}(B_k\mathbf1)$, $D_{e,k}=\operatorname{diag}(B_k^\top\mathbf1)$,
and $I_{0,k}=\operatorname{diag}(\mathbf1\{(D_{v,k})_{vv}=0\})$.
Using the Moore--Penrose inverse $\dagger$, define the node-to-node transition on these
observed local memberships by
\begin{equation}
  P_k=D_{v,k}^{\dagger}B_kD_{e,k}^{-1}B_k^\top+I_{0,k}.
  \label{eq:local_transition}
\end{equation}
This walk averages over incident relations and their members, leaving isolated nodes
fixed \citep{zhou2006learning}. For latent width $d$, encode
$U_k=E_{\theta_{\rm enc}}(X_k)\in\mathbb R^{|V_k|\times d}$ and use $R\ge0$ positive-time
channels with increments $\eta\in\mathbb R^{R\times d}$. The context parameters
$\vartheta=\operatorname{ctx}(\theta)=(\theta_{\rm enc},\eta)$ belong to the online state
$\theta$. The positive floor $\epsilon_{\rm ord}>0$ orders each coordinate's diffusion times:
\begin{gather}
  \mathsf K_k(t)
  =\exp\{t(P_k-I)\}
  =\sum_{\ell=0}^{\infty}e^{-t}\frac{t^\ell}{\ell!}P_k^\ell,
  \qquad t\ge0,
  \label{eq:heat_semigroup}\\
  t_{bj}
  =\sum_{a=1}^{b}\{\operatorname{softplus}(\eta_{aj})
  +\epsilon_{\rm ord}\},
  \qquad b\in[R],\quad j\in[d],
  \label{eq:ordered_heat_times}\\
  [\mathsf K_k(\boldsymbol t_b)U_k]_{:j}
  :=\mathsf K_k(t_{bj})U_{k,:j},\qquad
  W_k(X_k\mid\vartheta)
  =[\,U_k,\mathsf K_k(\boldsymbol t_1)U_k,\ldots,
  \mathsf K_k(\boldsymbol t_R)U_k\,].
  \label{eq:heat_basis}
\end{gather}
Each channel mixes walk lengths with Poisson weights \citep{kondor2002diffusion},
while $R=0$ retains only $U_k$. Each $t_{bj}$ is the expected walk length for
coordinate $j$. Learning these times adapts the extent of local mixing, while the
undiffused channel retains the original encoded attributes.
Thus $R$ controls scales, independently of observation
order $q$. Write $W_{k,\theta}(X)=W_k(X\mid\operatorname{ctx}(\theta))$.
Column relabeling preserves both features and coupled training updates.
Appendices~\ref{app:proof-heat-stability}--\ref{app:observable-dynamics}
prove stability and give a shared recurrence with Poisson-tail error control
for each learned diffusion time.

\paragraph{Local prediction and federation.}
Two masked views predict a full-view exponential moving average (EMA) target through
$\mathcal L_{\rm pred,k}$, combining cosine alignment, variance--covariance regularization,
and view agreement (Appendix~\ref{app:algorithm}).
For classifier $c_{\theta_k}:\mathbb R^{(R+1)d}\to\mathbb R^C$, context
$\vartheta_k=\operatorname{ctx}(\theta_k)$, labeled set $\mathcal V_k^{\rm tr}$ of size
$n_k^{\rm tr}>0$, and predictive weight $\lambda_{\rm pred}\ge0$, the local training objective is
\begin{equation}
  \mathcal L_k=
  (n_k^{\mathrm{tr}})^{-1}\sum\nolimits_{v\in\mathcal V_k^{\rm tr}}
  \operatorname{CE}(c_{\theta_k}(W_k(X_k\mid\vartheta_k)_v),y_v)
  +\lambda_{\rm pred}\mathcal L_{\mathrm{pred},k}.
  \label{eq:training_objective}
\end{equation}
Here $\operatorname{CE}$ is cross-entropy for $y_v\in[C]$. Algorithm~\ref{alg:fedastra} updates
target $\bar\vartheta_k$ with momentum $\mu\in[0,1)$ and transmits only online parameters,
weighted by each client's number of labeled training nodes.

\begin{algorithm}[tbp]
  \caption{FedASTRA with an ordered client-local diffusion basis}
  \label{alg:fedastra}
  \small
  \begin{algorithmic}[1]
    \REQUIRE Rounds $T$, steps $S_{\rm loc}$, clients $\mathcal C_u\subseteq[K]$, counts $n_k^{\rm tr}$,
    weight $\lambda_{\rm pred}$, momentum $\mu$.
    \STATE Initialize $\theta^0$ and client-local operators $P_k$.
    \FOR{$u=0,\ldots,T-1$}
      \FOR{$k\in\mathcal C_u$ in parallel}
        \STATE Set $\theta_k\leftarrow\theta^u$ and $\bar\vartheta_k\leftarrow\operatorname{ctx}(\theta_k)$.
        \STATE Repeat $S_{\rm loc}$ times: update $\theta_k$ on Equation~\eqref{eq:training_objective}, then set
        $\bar\vartheta_k\leftarrow\mu\bar\vartheta_k+(1-\mu)\operatorname{ctx}(\theta_k)$.
        \STATE Return $\theta_k^{u+1}\leftarrow\theta_k$ and discard the target.
      \ENDFOR
      \STATE $\textstyle
        \theta^{u+1}\leftarrow
        \frac{\sum_{k\in\mathcal C_u}n_k^{\mathrm{tr}}\theta_k^{u+1}}
             {\sum_{k\in\mathcal C_u}n_k^{\mathrm{tr}}}$.
    \ENDFOR
    \STATE Select one validation checkpoint and evaluate test data once.
  \end{algorithmic}
\end{algorithm}

\paragraph{Local information and optimal prediction.}
Under the rank, attribute, and incidence-recovery conditions of
Theorem~\ref{thm:local-heat-sufficiency} in Appendix~\ref{app:local-sufficiency}, the optimal cross-entropy gap between the complete
heat representation and local information is $O(\varepsilon_n\log(1/\varepsilon_n))$,
where $\varepsilon_n\to0$ bounds joint failure on $n$ nodes.
On successful recovery, the complete heat representation determines the permitted
local release. The residual risk bound accounts for decoder failure through its
binary entropy and the finite label alphabet.
Thus any local prediction rule can be reproduced on the event of successful recovery.

\paragraph{From client risks to federated risk.}
Let $w_k^{\rm pop}$ be test-client probabilities and $\mathcal R_{\rm fed}$ their weighted
cross-entropy. For the model $\widehat\theta$ selected after $T$ rounds on
$N=\sum_kn_k^{\rm tr}$ labels, Theorem~\ref{thm:client-population-risk} gives
\begin{equation}
  \mathbb E\!\left[\mathcal R_{\rm fed}(\widehat\theta)
    -\sum_kw_k^{\rm pop}\mathcal R^*_{{\rm loc},k}\right]
  \le\mathbb E\!\left[\sum_kw_k^{\rm pop}(d_k+a_k)+a_{\rm share}\right]
  +2u_N+e_{N,T}.
  \label{eq:main-federated-rate}
\end{equation}
Here $\mathcal R^*_{{\rm loc},k}$ is client $k$'s optimal local-information risk.
The terms $d_k,a_k,a_{\rm share}$ measure information loss, local approximation error,
and shared-parameter cost. Finally, $u_N$ bounds expected uniform training-to-population discrepancy,
and $e_{N,T}$ bounds empirical optimization error, addressed next.

\paragraph{Global convergence along the implemented updates.}
Let $J_N$ be evaluation-mode training cross-entropy on domain $\Theta$, convex in
the final affine head at fixed features. After round $u_0$,
assume each representation admits a head within distance $D_{\rm cls}>0$ of the current
head and loss at most $\inf_\Theta J_N+a_{\rm ctx}$, with $a_{\rm ctx}\ge0$.
If the actual-update decrease bound (Proposition~\ref{prop:actual-update-risk}) is at least
$\gamma>0$ times the squared head-gradient norm, Theorem~\ref{thm:actual-global-convergence}
bounds $e_{N,T}=\mathbb E[J_N(\widehat\theta)-\inf_\Theta J_N]$ by
\begin{equation}
  e_{N,T}\le\mathbb E\!\left[
  a_{\rm ctx}+\frac{D_{\rm cls}^2}{\gamma(T-u_0)}
  +s_{N,T}^{\rm sel}\right],\qquad T>u_0 .
  \label{eq:main-explicit-global-rate}
\end{equation}
Here $s_{N,T}^{\rm sel}=[J_N(\widehat\theta)-J_N(\theta^T)]_+$ accounts for validation
selection, with $[x]_+=\max\{x,0\}$. Substitution into Equation~\eqref{eq:main-federated-rate}
gives the population-risk bound. For fixed $u_0$ and uniformly bounded
$\mathbb E[D_{\rm cls}^2/\gamma]$, the optimization term is $O(1/T)$.
Vanishing residual terms then give local-information Bayes
consistency. Appendix~\ref{app:stochastic-convergence} also covers imperfect update alignment.

\section{Experiments}
\label{sec:experiments}

We evaluate FedASTRA through four research questions (RQ1-RQ4) concerning predictive performance, mechanism validation, information access, and scalability:

\begin{itemize}[nosep]
  \item \textit{RQ1: How effective is FedASTRA on identifier-free native hypergraphs?}
  \item \textit{RQ2: Do the empirical results support the proposed analysis and FedASTRA's design?}
  \item \textit{RQ3: How does FedASTRA compare under different information-access settings?}
  \item \textit{RQ4: Can FedASTRA scale to large real-world hypergraphs without sampling?}
\end{itemize}
Additional experimental settings and results are provided in Appendices~\ref{app:protocols} and~\ref{app:results}, respectively.

\noindent\textbf{Datasets.}
We evaluate two groups of datasets. 
(1) Native hypergraph datasets including CoraCA coauthorship, DBLP4k bibliographic, IMDB4k movie, and Trivago hotel-session data \citep{yadati2019hypergcn,sun2011pathsim,fu2020magnn,feng2019hypergraph,gao2023hgnnplus,li2026dhgbench}. 
(2) Graph datasets including Cora and CiteSeer citation networks, the Facebook Page-Page social network
\citep{yang2016revisiting,rozemberczki2021multiscale}, and the large-scale ogbn-papers100M dataset \citep{hu2020ogb}. 
Appendix~\ref{app:datasets_relation_channels} 
details the datasets.

\noindent\textbf{Baselines.}
We consider two groups of baselines. (1) Federated adaptations (FedAvg \citep{mcmahan2017communication}) of native-hypergraph methods, including  
Fed-HGNN (HGNN \citep{feng2019hypergraph}), Fed-HNHN (HNHN \citep{dong2020hnhn}), Fed-HyperGCN (HyperGCN \citep{yadati2019hypergcn}), Fed-AllDeepSets (DeepSets \citep{zaheer2017deepsets}), Fed-AllSetTransformer (AllSet \citep{chien2022allset}), Fed-UniGCNII (UniGCNII \citep{huang2021unignn}) and Fed-ED-HNN (ED-HNN \citep{wang2023edhnn}). (2) Federated graph learning methods with different levels of cross-client information access, including FedSage with no access and FedSage+ with generated information \citep{zhang2021fedsage}, FedGCN \citep{yao2023fedgcn} and FedCog \citep{lei2023fedcog} with direct access to correspondence. Details of baselines are presented in Appendix ~\ref{app:federated-protocol}.

\noindent\textbf{Implementation Details.}
Fed-HNHN uses the configuration reported by FedHGL \citep{luo2025fedhgl},
and Fed-HyperGCN uses the local HyperGCN projection. FedASTRA uses two masks,
a local EMA target, and parameter-only FedAvg
\citep{mcmahan2017communication,kingma2015adam}.
Results are averaged over five runs, except three for Fed-HNHN and all Trivago-related experiments.
We report node classification accuracy (Acc.) and Macro-F1.
The nodes are partitioned across $K\in\{3,6,9\}$ clients using class-conditioned Dirichlet sampling \citep{luo2025fedhgl}. 
All methods follow the same data splits and training and evaluation protocols.
Appendix~\ref{app:protocols} details the experimental protocols. 

\noindent\textbf{RQ1: performance on identifier-free native hypergraphs.}
As shown in Table~\ref{tab:heat-native-main}, FedASTRA achieves the best Accuracy and Macro-F1 across all four hypergraph datasets from distinct domains, especially on CoraCA, outperforming the baseline by 6.5 and 8.5 percentage points, respectively. The consistent cross-domain gains support FedASTRA's correspondence-invariant design, which avoids global structure recovery while approaching the optimum permitted by local information under the stated conditions. In contrast, direct FedAvg adaptations of traditional hypergraph models may under-exploit local information when global structure is fragment. More detailed results under different number of clients are reported in Appendix~\ref{app:results-native}.

\begin{table}[htbp]
  \centering
  \caption{Performance on identifier-free native hypergraphs. Results are averaged over $K\in\{3,6,9\}$.}
  \label{tab:heat-native-main}
  \small
  \setlength{\tabcolsep}{3pt}
  \resizebox{\textwidth}{!}{%
    \begin{tabular}{lcccccccc}
      \toprule
      Method & \multicolumn{2}{c}{CoraCA} & \multicolumn{2}{c}{DBLP4k} & \multicolumn{2}{c}{IMDB4k} & \multicolumn{2}{c}{Trivago} \\
      & Acc. (\%) & Macro-F1 (\%) & Acc. (\%) & Macro-F1 (\%) & Acc. (\%) & Macro-F1 (\%) & Acc. (\%) & Macro-F1 (\%) \\
      \midrule
      Fed-HGNN & \underline{68.64 $\pm$ 1.45} & \underline{64.35 $\pm$ 1.71} & \underline{77.63 $\pm$ 1.35} & \underline{76.98 $\pm$ 1.18} & \underline{48.85 $\pm$ 1.29} & \underline{48.59 $\pm$ 1.23} & 21.85 $\pm$ 0.19 & 4.27 $\pm$ 0.20 \\
      Fed-HNHN & 54.73 $\pm$ 1.17 & 50.42 $\pm$ 1.76 & 69.58 $\pm$ 1.18 & 68.79 $\pm$ 1.12 & 44.96 $\pm$ 1.50 & 43.05 $\pm$ 1.24 & \underline{25.68 $\pm$ 1.02} & \underline{7.72 $\pm$ 0.24} \\
      Fed-HyperGCN & 64.02 $\pm$ 1.64 & 58.07 $\pm$ 1.95 & 72.69 $\pm$ 0.77 & 71.46 $\pm$ 0.86 & 44.14 $\pm$ 1.16 & 40.86 $\pm$ 2.38 & 9.17 $\pm$ 0.20 & 0.13 $\pm$ 0.02 \\
      Fed-AllDeepSets & 39.63 $\pm$ 1.83 & 27.08 $\pm$ 2.04 & 39.35 $\pm$ 0.55 & 32.45 $\pm$ 1.71 & 40.32 $\pm$ 1.15 & 34.12 $\pm$ 2.20 & 10.88 $\pm$ 0.15 & 0.27 $\pm$ 0.01 \\
      Fed-AllSetTransformer & 30.22 $\pm$ 0.00 & 6.63 $\pm$ 0.00 & 30.99 $\pm$ 1.29 & 21.81 $\pm$ 0.86 & 37.04 $\pm$ 0.05 & 18.10 $\pm$ 0.07 & 8.99 $\pm$ 0.00 & 0.10 $\pm$ 0.00 \\
      Fed-UniGCNII & 59.54 $\pm$ 1.31 & 50.11 $\pm$ 3.86 & 77.35 $\pm$ 2.26 & 75.98 $\pm$ 2.60 & 41.62 $\pm$ 0.74 & 30.57 $\pm$ 3.32 & 9.15 $\pm$ 0.13 & 0.15 $\pm$ 0.03 \\
      Fed-ED-HNN & 59.62 $\pm$ 1.65 & 53.28 $\pm$ 2.30 & 72.10 $\pm$ 2.54 & 71.36 $\pm$ 2.55 & 43.46 $\pm$ 0.89 & 36.58 $\pm$ 1.05 & 9.19 $\pm$ 0.29 & 0.13 $\pm$ 0.03 \\
      FedASTRA & \textbf{75.17 $\pm$ 0.42} & \textbf{72.88 $\pm$ 0.60} & \textbf{79.22 $\pm$ 0.98} & \textbf{78.61 $\pm$ 0.80} & \textbf{51.13 $\pm$ 0.97} & \textbf{51.06 $\pm$ 0.77} & \textbf{29.09 $\pm$ 1.19} & \textbf{7.96 $\pm$ 0.10} \\
      \bottomrule
    \end{tabular}%
  }
\end{table}

\noindent\textbf{RQ2: validation on proposed analysis and mechanism.}
We first validate the proposed identification analysis through two-client audits. Among 120 evaluated cells, no higher-order co-membership target is point identified, and the median endpoint-hull width increases from 36.37\% at $q=2$ to 93.71\% at $q=5$. Type coarsening preserves endpoint containment in all 40 audited cells, while 32 column permutations leave $P_k$, heat features, and logits unchanged. These results confirm the predicted ambiguity under latent correspondence and the correspondence-invariant property of the learner representation. Higher-order truncation exhibits different predictive trends across datasets, indicating that its impact depends on how the lost higher-order structure aligns with task-relevant signal. Complete audit configurations and results are reported in Appendix~\ref{app:results-mechanism}.

For mechanism ablations, we evaluate heat-basis width $R\in\{0,1,2,4,8\}$ and the masked prediction objective. Figure~\ref{fig:ablation_main} shows that removing diffusion ($R=0$) degrades performance, while moderate heat widths improve accuracy. Removing the predictive objective ($\lambda_{pred}=0$) also consistently reduces performance, demonstrating its complementary role in FedASTRA. Complete results are reported in Appendix~\ref{app:results-mechanism}.

\begin{figure}[H]
  \centering
  \includegraphics[width=\textwidth]{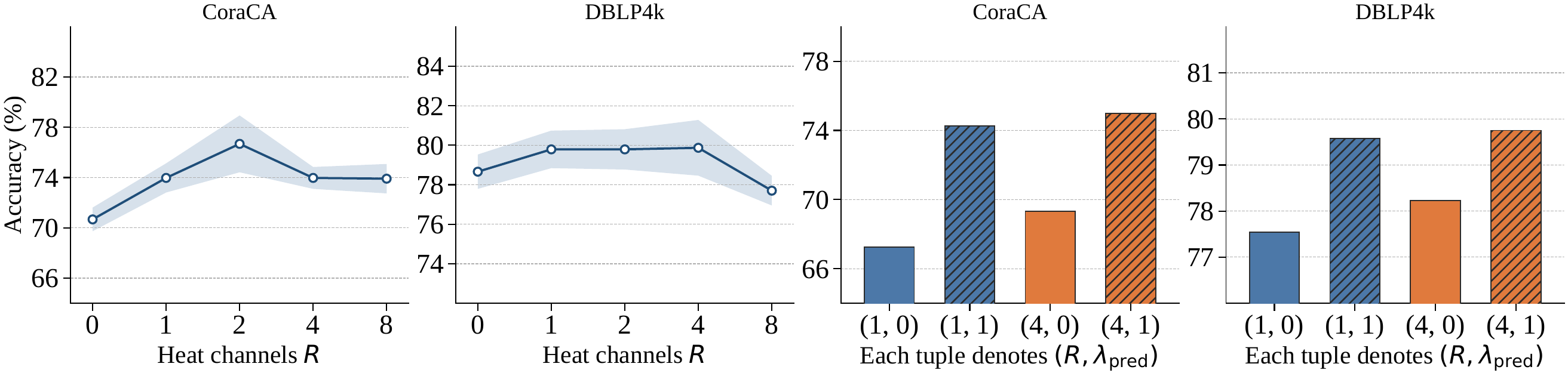}
  \caption{Ablation study on the adaptive diffusion (left) and the masked prediction (right).}
  \label{fig:ablation_main}
\end{figure}

\noindent\textbf{RQ3: performance against baselines with varying information access.}
We empirically contrast the information-access regimes discussed in Section~\ref{sec:information-comparison}, as depicted in Table~\ref{tab:heat-graph-main}. FedASTRA outperforms FedSage and FedGCN (0-hop), which use only local information, and also surpasses FedSage+, despite the latter using additional communication to generate cross-client information. Although FedGCN (2-hop) and FedCog are stronger under a different setting with known correspondence, the gap is substantially narrowed, suggesting that the remaining gap is not solely missing cross-client information, but also how effectively the local structure is represented and exploited. The stable performance of FedASTRA as the client number grows further indicates that stronger local modeling can partially offset increasing structural information loss. Appendix~\ref{app:results-graph} provides complete metrics, information-access details, and traffic accounting.

\begin{table}[htbp]
  \centering
  \caption{Performance against federated graph learning methods. Communication traffic (MiB) is averaged across datasets and client counts, with model traffic per round and auxiliary traffic per run.}
  \label{tab:heat-graph-main}
  \scriptsize
  \setlength{\tabcolsep}{2.2pt}
  \resizebox{\textwidth}{!}{%
    \begin{tabular}{lccccccccccc}
      \toprule
      Method & \multicolumn{3}{c}{Cora} & \multicolumn{3}{c}{CiteSeer} & \multicolumn{3}{c}{Facebook} & \multicolumn{2}{c}{Average communication traffic} \\
      & $K=3$ & $K=6$ & $K=9$ & $K=3$ & $K=6$ & $K=9$ & $K=3$ & $K=6$ & $K=9$ & Model/round & Auxiliary/run \\
      \midrule
      FedSage & 70.35 $\pm$ 1.96 & 63.18 $\pm$ 0.97 & 58.65 $\pm$ 1.85 & 65.08 $\pm$ 0.96 & 63.73 $\pm$ 0.74 & 62.93 $\pm$ 1.72 & 72.63 $\pm$ 0.73 & 65.64 $\pm$ 1.81 & 60.41 $\pm$ 1.61 & 4.82 & 0 \\
      FedGCN (0-hop) & 74.07 $\pm$ 1.57 & 65.90 $\pm$ 2.05 & 62.07 $\pm$ 1.79 & 68.68 $\pm$ 1.50 & 65.83 $\pm$ 0.56 & 65.14 $\pm$ 1.80 & 75.14 $\pm$ 0.95 & 67.87 $\pm$ 1.71 & 64.30 $\pm$ 1.22 & 2.41 & 0 \\
      FedSage+ & 69.67 $\pm$ 0.71 & 63.11 $\pm$ 1.66 & 61.02 $\pm$ 1.40 & 63.16 $\pm$ 1.17 & 62.53 $\pm$ 1.66 & 62.95 $\pm$ 1.62 & 68.17 $\pm$ 0.33 & 61.85 $\pm$ 0.95 & 58.56 $\pm$ 1.33 & 4.82 & 17314.45 \\
      FedGCN (2-hop) & 84.05 $\pm$ 1.17 & 84.05 $\pm$ 1.16 & 84.12 $\pm$ 0.92 & 72.20 $\pm$ 1.44 & 72.27 $\pm$ 1.53 & 72.38 $\pm$ 1.50 & 83.21 $\pm$ 0.76 & 82.84 $\pm$ 0.50 & 82.73 $\pm$ 0.85 & 2.41 & 673.79 \\
      FedCog & 84.29 $\pm$ 1.50 & 84.34 $\pm$ 1.06 & 83.92 $\pm$ 0.78 & 72.21 $\pm$ 1.49 & 72.38 $\pm$ 1.04 & 72.51 $\pm$ 1.42 & 82.46 $\pm$ 0.69 & 82.50 $\pm$ 0.84 & 82.15 $\pm$ 0.94 & 2.41 & 752.27 \\
      FedASTRA & 77.54 $\pm$ 1.08 & 74.45 $\pm$ 1.60 & 73.33 $\pm$ 1.22 & 67.61 $\pm$ 2.32 & 68.00 $\pm$ 1.86 & 67.92 $\pm$ 1.88 & 76.88 $\pm$ 1.44 & 75.57 $\pm$ 0.94 & 75.12 $\pm$ 1.49 & 13.53 & 0 \\
      \bottomrule
    \end{tabular}%
  }
\end{table}

\noindent\textbf{RQ4: scalability on large-scale data.}
We evaluate FedASTRA on ogbn-papers100M after local one-hop lifting, as summarized in Table~\ref{tab:heat-papers100m}.
We also include Fed-MLP, Fed-GCN \citep{kipf2017semi}, and Fed-SIGN \citep{frasca2020sign} as baselines. FedASTRA outperforms the strongest baseline by up to 9.69 percentage points under comparable communication overhead. Its scalability stems from the correspondence-free local diffusion design and 
the strong performance further suggests that adaptive diffusion can better exploit local structures.
Appendix~\ref{app:results-scale} details the results.

\begin{table}[htbp]
  \centering
  \caption{Performance on complete ogbn-papers100M dataset. Nodes are partitioned across 8 clients. Traffic counts bidirectional trainable-parameter per round (MiB). BatchNorm buffers stay local.}
  \label{tab:heat-papers100m}
  \scriptsize
  \setlength{\tabcolsep}{3pt}
  \begin{tabular}{lcccc}
    \toprule
    Method & Acc. (\%) & Macro-F1 (\%) & \#Parameters & Parameter trafffic/round \\
    \midrule
    Fed-MLP & 33.61 & 13.74 & 983,380 & 60.02 \\
    Fed-SIGN & 38.42 & 16.56 & 984,442 & 60.09 \\
    Fed-HGNN & 31.04 & 9.97 & 982,672 & 59.98 \\
    Fed-HNHN & 35.93 & 13.46 & 34,672 & 2.12 \\
    Fed-HyperGCN & \underline{40.38} & \underline{21.44} & 984,442 & 60.09 \\
    Fed-GCN & 39.70 & 17.40 & 983,380 & 60.02 \\
    FedASTRA & \textbf{50.07} & \textbf{29.73} & 976,504 & 59.60 \\
    \bottomrule
  \end{tabular}
\end{table}
\section{Conclusion}
\label{sec:conclusion}

We establish identification and computational boundaries for latent hyperedge correspondence
and quantify when fitting costs outweigh the benefit of additional information.
FedASTRA combines invariant local diffusion, masked prediction, and parameter-only communication.
Under stated assumptions, information-preservation and global risk bounds connect observable
memberships to prediction. Results on native hypergraphs, federated subgraphs, and large-scale 
ogbn-papers100M data support this approach under explicit information-access and communication constraints.

\bibliography{refs}
\bibliographystyle{iclr2027_conference}

\clearpage
\appendix
\raggedbottom
\phantomsection
\pdfbookmark[0]{Appendix}{appendix-contents}
\begingroup
\hypersetup{linkcolor=black}
\begin{center}
  {\LARGE\bfseries Appendix}
\end{center}
\bigskip
{\Large\bfseries Contents}\par
\medskip
\setlength{\parskip}{0.42em}
\appendixsectionentry{app:more-related-work}
\appendixsectionentry{app:formal-results}
\appendixsubsectionentry{app:notation}
\appendixsubsectionentry{app:framework}
\appendixsubsectionentry{app:two-client-boundary}
\appendixsectionentry{app:proofs}
\appendixsubsectionentry{app:proof-completion}
\appendixsubsectionentry{app:proof-qmass}
\appendixsubsectionentry{app:proof-sorting}
\appendixsubsectionentry{app:three-client-boundary}
\appendixsubsectionentry{app:proof-zeta}
\appendixsubsectionentry{app:proof-ht}
\appendixsubsectionentry{app:proof-two-point}
\appendixsubsectionentry{app:proof-heat-stability}
\appendixsubsectionentry{app:observable-dynamics}
\appendixsubsectionentry{app:auxiliary-risk}
\appendixsubsectionentry{app:local-sufficiency}
\appendixsubsectionentry{app:stochastic-convergence}
\appendixsectionentry{app:protocols}
\appendixsubsectionentry{app:algorithm}
\appendixsubsectionentry{app:partitions}
\appendixsubsectionentry{app:datasets_relation_channels}
\appendixsubsectionentry{app:federated-protocol}
\appendixsectionentry{app:results}
\appendixsubsectionentry{app:results-native}
\appendixsubsectionentry{app:results-graph}
\appendixsubsectionentry{app:results-mechanism}
\appendixsubsectionentry{app:results-scale}
\endgroup

\clearpage
\section{More Related Work}
\label{app:more-related-work}

\paragraph{Latent correspondence and partial identification.}
Unknown correspondence is classical in record linkage
\citep{fellegi1969record}, where latent pairings between records must be
inferred from incomplete identifying information. Partial identification
instead retains the set of target values compatible with the observations
when point identification is unavailable \citep{manski2003partial}.
Our setting combines these perspectives at the level of higher-order
relations: client-local hyperedge fragments are observed, but their
cross-client instance correspondence is not. Rather than selecting or
estimating one global linkage, we characterize the full compatibility class induced by the released fragments and derive sharp target bounds over all compatible alignments.

\paragraph{Alignment under latent correspondence.}
When object identities are not shared, relational alignment is often formulated through latent couplings. Gromov--Wasserstein and fused
Gromov--Wasserstein formulations compare structured objects without requiring
a common coordinate system \citep{peyre2016gw,vayer2019fgw}, while recent
partial formulations relax full-mass or full-matching requirements
\citep{bai2025partialgw,ICLR2025_cfa8440d}.
For higher-order structures, HyperCOT develops transport-based comparisons of
hypergraphs \citep{chowdhury2024hypercot}, and recent partitioned-network
geometry unifies graph and hypergraph alignment within a common metric
framework \citep{zhang2026partitioned}.
These approaches optimize an alignment or distance under a specified geometry.
But more specifically here, we are more concerned with which target values remain invariant to latent cross-client correspondence and which vary within the compatibility class induced by a fixed federated release.

\paragraph{Hypergraph reconstruction from coarsened observations.}
Some existing papers study whether higher-order structure can be reconstructed
from lower-order or otherwise coarsened observations. Early work develops
statistical reconstruction from network data \citep{young2021hypergraph}.
Wang and Kleinberg characterize information loss under hypergraph projection
and study reconstruction with additional structural knowledge
\citep{wang2024graphs}. For random uniform hypergraphs, recent work gives information-theoretic and algorithmic thresholds for exact recovery
\citep{pmlr-v247-bresler24a}, followed by a characterization of both partial
and exact recovery across hyperedge sizes and multiplicity regimes
\citep{bresler2025recovery}. Multiplicity-aware supervised reconstruction
provides a complementary algorithmic approach \citep{lee2025marioh}, while
recent learning theory has also formulated relational learning itself through
a hypergraph-recovery model \citep{pmlr-v235-chen24l}.
These works primarily ask when a latent source hypergraph is recoverable under projection or a specified generative model. Our observation is however different: clients retain higher-order fragments themselves, but their cross-client
instance identities are suppressed. We therefore retain all compatible global
hypergraphs and characterize what is identifiable without directly assuming that one
global completion can be recovered.

\paragraph{Assignment and multiway matching complexity.}
Two-way correspondence is closely related to classical linear assignment and
bipartite matching \citep{birkhoff1946tres,kuhn1955hungarian}, whereas
multidimensional assignment and three-dimensional matching exhibit
fundamentally different computational behavior
\citep{pierskalla1968multidimensional,karp1972reducibility}.
Parameterized algorithms for set packing and matching draw on tools such as
color coding, while the W-hierarchy distinguishes fixed-parameter
tractability from parameterized intractability
\citep{DBLP:journals/jacm/AlonYZ95,DBLP:journals/tcs/DowneyF95}.
Recent partial graph matching still exposes tractable assignment structure in
important pairwise settings \citep{ICLR2025_cfa8440d}.
Our result places these phenomena within one latent-correspondence problem:
under the same observation and reward rules, two-client correspondence is
polynomial, fixed-client multiway matching is NP-complete but
fixed-parameter tractable in the stated joint parameter, and unbounded client
count is W[1]-complete in matching size.

\section{Formal Results and Identification Boundaries}
\label{app:formal-results}

This section records the observation assumptions used in the proofs: which fragment
attributes are retained, when partition censoring admits unbiased correction, and when
unknown correspondence is exactly a bipartite assignment.

\subsection{Important Notation}
\label{app:notation}

\begin{table}[H]
  \centering
  \small
  \caption{Core notation for the observation model, identification results,
    algorithm, and experiments.}
  \begin{tabularx}{\textwidth}{@{}lX@{}}
    \toprule
    Symbol & Meaning \\
    \midrule
    $\mathcal H=(V,\mathcal E,X,r,w)$ & Finite node-attributed, typed, weighted multihypergraph. \\
    $A_e,r_e,w_e$ & Support, type in $\mathcal R$, and positive weight of relation instance $e\in\mathcal E$. \\
    $w_r(A)$ & Aggregate weight of type-$r$ instances with support $A$. \\
    $M_{q,r}^{\mathcal H}(S)$ & Weighted type-$r$ order-$q$ co-membership count for $S\in\binom Vq$. \\
    $\mathsf Z$, $z$, $V_k$ & Random node-to-client assignment, one realization, and the nodes assigned to client $k$. \\
    $c$, $\upsilon_{e,k}$ & Minimum visible fragment size and the analyst's local/cross status mark. \\
    $\Gamma$, $\mathcal M_\Gamma$ & Public observation assumptions and its allowed latent-hypergraph class. \\
    $Y_{\rm A},y_{\rm A}$ & Analyst release as a random element and one realized value. \\
    $Y_{\rm L},y_{\rm L}$ & Coarsened learner release $Y_{\rm L}=\rho(Y_{\rm A})$ and one realized value. \\
    $\mathcal F_{\rm A}$, $\mathcal F_{\rm L}$ & Analyst and learner compatibility classes defined in Equation~\eqref{eq:compatible_set} and Corollary~\ref{cor:learner-fiber-lower-bound}. \\
    $\widetilde M_{q,r}(\cdot,z)$, $\tau_q$ & Retained type-$r$ coordinates and the type-aggregated erased-mass fraction. \\
    $\pi_S$, $\pi_q$ & Inclusion probability of a $q$-set and its common value under an exchangeable randomized partition design. \\
    $\mathcal I_{q,\psi}(y_{\rm A})$ & Values of $\Phi_{q,\psi}$ over $\mathcal F_{\rm A}(y_{\rm A},\Gamma)$. \\
    $h=(r,\omega)$, $C^h$, $\mathbb B_{m_h}$ & Exact type-and-weight stratum, its assignment cost, and the Birkhoff polytope. \\
    $B_k,D_{v,k},D_{e,k},I_{0,k}$ & Client incidence matrix, node and relation degrees, and isolated-node indicator. \\
    $P_k$, $P_k^{\rm sym}$ & Observable Markov operator and its symmetric representative on positive-degree nodes. \\
    $\mathsf K_k(t)$ & Continuous-time local diffusion $\exp\{t(P_k-I)\}$ at time $t\ge0$. \\
    $R$, $t_{bj}$, $\epsilon_{\rm ord}$ & Number of positive heat channels, channel-$b$ time for feature $j$, and its positive increment floor. \\
    $W_k(X_k\mid\vartheta)$, $W_{k,\theta}(X_k)$ & Heat-basis representation and its shorthand with $\vartheta=\operatorname{ctx}(\theta)$. \\
    $L_{\rm tail}$, $\varepsilon_{\rm tail}$ & Shared recurrence degree and its deterministic Poisson-tail tolerance. \\
    $\theta$, $\vartheta$, $\bar\vartheta$ & Complete online state, its context coordinates $(\theta_{\rm enc},\eta)$, and the client-local EMA target context. \\
    $p_\theta$, $c_\theta$, $y_v$ & Rowwise predictor, rowwise classifier, and class label of node $v$. \\
    $\gamma_{\rm view}$, $\lambda_{\rm pred}$, $\mu$ & View-agreement weight, predictive-loss weight, and target momentum. \\
    $\mathcal P$, $\xi$, $\mathcal T_T$ & Fixed training protocol, algorithmic randomness, and the transcript through round $T$. \\
    \bottomrule
  \end{tabularx}
\end{table}

\paragraph{Probability and conditioning convention.}
Unless a statement says otherwise, $\mathcal H$, $z$, $\Gamma$, and the realized release are
fixed. Expectations and variances in Theorem~\ref{thm:ht} are taken over the declared random
partition $\mathsf Z$. Corollary~\ref{cor:matching} explicitly averages over a uniform latent
matching. Algorithmic laws average over $\xi\sim\nu$, where the fixed law $\nu$ is independent
of $\mathcal H$. The auxiliary variable $N_t^{\rm aux}$ below only expresses deterministic
Poisson coefficients and tail bounds and is absent from implementation sampling.
Appendix~\ref{app:auxiliary-risk} separately defines its downstream evaluation law.

\subsection{Observation Channel and Design-based Recovery}
\label{app:framework}

\subsubsection{Aggregate support--type weights and Boolean-zeta coordinates}

The analyst release retains each fragment occurrence with its exact type, weight, and
local/cross status. Assignment strata use these individual rows, including duplicates.
Co-membership coordinates instead depend on $w_r(A)$, the total weight of rows with support
$A$ and type $r$. Aggregation loses the allocation of weight among duplicate occurrences,
so it cannot replace the row-level table in the matching analysis. Erasing types further
reduces the observation to $M_q=\sum_r M_{q,r}$. The following inclusion describes the
resulting uncertainty on the same latent class $\mathcal M_\Gamma$.

\begin{corollary}[Coarsened or missing type marks enlarge identified sets]
  \label{cor:type_coarsening}
  Let $\chi:\mathcal R\to\mathcal R_\chi$ merge or suppress type labels, and let
  $\chi_\#$ apply this map to the release while preserving its other fields.
  For an exact-type release $y_{\rm ex}$, set $y_\chi=\chi_\#y_{\rm ex}$ and
  $\mathcal O_\chi=\chi_\#\circ\mathcal O_{z,\Gamma}$, with public type
  restrictions pushed forward accordingly. On the common class $\mathcal M_\Gamma$, define
  \begin{equation}
    \mathcal F_{\rm ex}(y_{\rm ex},\Gamma)=\{\mathcal H\in\mathcal M_\Gamma:
    \mathcal O_{z,\Gamma}(\mathcal H)=y_{\rm ex}\},\qquad
    \mathcal F_\chi(y_\chi,\Gamma)=\{\mathcal H\in\mathcal M_\Gamma:
    \mathcal O_\chi(\mathcal H)=y_\chi\}.
    \label{eq:type_coarsening_fibers}
  \end{equation}
  For a scalar target $\phi:\mathcal M_\Gamma\to\mathbb R$, write
  $\mathcal I_{\phi,{\rm ex}}=\phi(\mathcal F_{\rm ex}(y_{\rm ex},\Gamma))$ and
  $\mathcal I_{\phi,\chi}=\phi(\mathcal F_\chi(y_\chi,\Gamma))$. Then
  \begin{equation}
    \mathcal F_{\rm ex}(y_{\rm ex},\Gamma)\subseteq\mathcal F_\chi(y_\chi,\Gamma),\qquad
    \mathcal I_{\phi,{\rm ex}}\subseteq\mathcal I_{\phi,\chi}.
    \label{eq:type_coarsening_inclusion}
  \end{equation}
  When the scalar endpoints are finite, this inclusion gives
  \begin{equation}
    \inf\mathcal I_{\phi,\chi}\le\inf\mathcal I_{\phi,{\rm ex}}
    \le\sup\mathcal I_{\phi,{\rm ex}}\le\sup\mathcal I_{\phi,\chi}.
    \label{eq:type_coarsening_endpoints}
  \end{equation}
  Thus type coarsening can only preserve or widen the sharp endpoint interval.
\end{corollary}

\begin{proof}
  Every hypergraph reproducing the exact marked table also reproduces its
  $\chi$-coarsening, which proves the first inclusion.  Taking the image under
  $\phi$ proves the second, and nested scalar sets give the endpoint inequalities.
\end{proof}

\subsubsection{Independent and fixed-capacity designs}

For fixed $\mathcal H$ and $q\ge c$, a coordinate $S\in\binom Vq$ survives when all its
nodes share a client. Independent assignments with $p_{vk}=\Pr\{\mathsf Z(v)=k\}$ give
\begin{equation}
  \pi_S=\Pr\{\mathsf Z|_S\text{ is constant}\}
  =\sum_{k=1}^K\prod_{v\in S}p_{vk}.
  \label{eq:node_dependent_pi}
\end{equation}
For a uniformly random partition with fixed capacities
$(n_1,\ldots,n_K)$, exchangeability gives
\begin{equation}
  \pi_q=\sum_{k=1}^K\frac{\binom{n_k}{q}}{\binom{|V|}{q}}.
  \label{eq:fixed_capacity_pi}
\end{equation}
Positive inclusion probabilities permit coordinatewise Horvitz--Thompson weighting
\citep{horvitz1952generalization}. Its unbiasedness concerns repeated randomized partitions.
Proportions from a fixed institutional partition alone actually fail to specify the required assignment law.

\subsubsection{Conditions for coordinate-wise partition censoring}

Equation~\eqref{eq:mask_law} requires retained multiplicity, exact weight, type, and $q\ge c$.
Deduplication would change the count, so the data interface preserves duplicate and singleton
columns. For $c>1$, some nonempty fragments disappear. Completion analysis then requires
constraints on original hyperedge counts or orders. The exact two-client result uses $c=1$.

\subsection{Two-client Assumptions and Weight-stratum Boundary}
\label{app:two-client-boundary}

Theorem~\ref{thm:completion_polytope} assumes two clients, $c=1$, observed local/cross status,
one nonempty row at each client per cross-client hyperedge, and exact type-and-weight strata
$h=(r,\omega)$ with $m_{1h}=m_{2h}=:m_h$. Each compatible completion is then a bijection
within every stratum. Missing fragments, uncertain counts, or further clients change this
feasible set. Weight bins also enlarge it: the public model must specify the allowed
underlying weights, and endpoints must be optimized over that range rather than evaluated
at an assumed bin midpoint.

Under the exact assumptions, general two-client targets admit assignment bounds, while
total order-$q$ mass admits sorting bounds. General three-client targets are NP-hard.
For any client count, anchor-additive costs remain tractable because the objective
separates into $K-1$ independent bipartite assignments
(Proposition~\ref{prop:star_separable_multiclient}).

\section{Proofs}
\label{app:proofs}

We prove the structural bounds first, then connect observable heat features to prediction
risk and the shared-model federated objective.

\subsection{Two-client Completion and Point Identification}
\label{app:proof-completion}

Fix a stratum $h=(r,\omega)$. Under the exact two-client assumptions, every
$\varpi\in\operatorname{Sym}(m_h)$ is compatible. Each $q$-set in
$F_{1i}^h\cup F_{2,\varpi(i)}^h$ lies within one fragment or intersects both.
Summing the within-fragment terms over $i$ is matching-invariant, while the mixed terms
give $C^h_{i,\varpi(i)}$. Adding known local contributions proves
Equation~\eqref{eq:identified_functional_set}. By Birkhoff--von Neumann integrality,
the vertices of $\mathbb B_m$ are exactly the permutation matrices
\citep{birkhoff1946tres}. Both linear-program endpoints therefore attain compatible
matchings and are sharp. Each costs $O(m_h^3)$ by Hungarian assignment
\citep{kuhn1955hungarian}.

The assignment dual also provides a certificate for the computed bound. For each stratum,
the lower-endpoint dual is
\begin{equation}
  \max_{u,v}\ \sum_i u_i+\sum_jv_j
  \quad\text{s.t.}\quad u_i+v_j\le C^h_{ij}.
  \label{eq:assignment_min_dual}
\end{equation}
The upper endpoint has a minimization dual with $u_i+v_j\ge C^h_{ij}$.
Dual feasibility and equality of primal and dual values certify an endpoint,
with equality on every pair in an optimal matching.

If $C_{ij}=\alpha_i+\beta_j$, every permutation sums to
$\sum_i\alpha_i+\sum_j\beta_j$. Conversely, equality of all permutation sums implies
equality after exchanging any two assigned columns, hence
\begin{equation}
  C_{ij}+C_{i'j'}-C_{ij'}-C_{i'j}=0.
  \label{eq:cycle_swap_proof}
\end{equation}
When all such rectangular cross-differences vanish, set
$\alpha_i=C_{i1}$ and $\beta_j=C_{1j}-C_{11}$.  The cycle using row 1 and column
1 gives $C_{ij}=\alpha_i+\beta_j$.  For $m=1$ this representation is immediate.
Finally, $J_mCJ_m=0$ if and only if $C$ lies in the kernel of double centering,
which is exactly the row--column-additive subspace.  Thus the quotient of cost
matrices by the kernel of assignment variation is represented uniquely by
$J_mCJ_m$. Since $J_m$ has rank $m-1$, this image has dimension $(m-1)^2$.
The quotient maps injectively into the following $(m-1)^2$ base four-cycle coordinates:
\begin{equation}
  \left(C_{ij}-C_{i1}-C_{1j}+C_{11}\right)_{i,j=2}^{m}
  \label{eq:base_cycle_coordinates}
\end{equation}
Zero base coordinates imply row--column additivity, so these $(m-1)^2$ coordinates are
independent on the quotient. The dimension is attainable with $q=2$ and unit-weight
singleton fragments by setting $\psi_{2,r}(\{a_i,b_j\},X_{\{a_i,b_j\}})=C_{ij}$.
Repeated fragments or restricted test functions may reduce the realized dimension, without
changing the endpoint argument. Since strata vary independently, point identification
requires additivity in every stratum. \hfill$\square$

\subsection{\texorpdfstring{Order-$q$}{Order-q} Mass Identification}
\label{app:proof-qmass}

For a general target, the preceding proof requires an assignment solver. Total order-$q$
mass has more structure: with $q\ge2$, its mixed cost depends only on fragment sizes.
This allows both endpoints to be computed by sorting. Within
stratum $h=(r,\omega)$, set $a_i^h=|F_{1i}^h|$, $b_j^h=|F_{2j}^h|$, and $\omega_h=\omega$.
Vandermonde's identity and its discrete mixed difference give
\begin{gather}
  \kappa_q(a,b)=\binom{a+b}{q}-\binom aq-\binom bq
  =\sum_{\ell=1}^{q-1}\binom a\ell\binom b{q-\ell},
  \label{eq:vandermonde_mixed}\\[-1mm]
  \kappa_q(a,b)+\kappa_q(a',b')-\kappa_q(a,b')-\kappa_q(a',b)
  =\sum_{x=a}^{a'-1}\sum_{y=b}^{b'-1}\binom{x+y}{q-2}\ge0
  \label{eq:reverse_monge}
\end{gather}
for $a\le a'$ and $b\le b'$. Same-order pairing maximizes the total mass at
$T_q^+$, while opposite-order pairing minimizes it at $T_q^-$. Writing the
released local mass as $M_q^{\rm obs}$ gives sharp erased-mass bounds:
\begin{equation}
  \tau_q^-=1-M_q^{\rm obs}/T_q^-,\qquad
  \tau_q^+=1-M_q^{\rm obs}/T_q^+.
  \label{eq:tau_bounds}
\end{equation}
The zero-total convention and explicit formulas for $T_q^\pm$ appear in
Appendix~\ref{app:proof-sorting}.  Point identification holds exactly when
the mixed difference vanishes throughout every stratum.  Sorting computes both
attainable endpoints in $O(\sum_hm_h\log m_h)$ time by sorting fragment sizes within each stratum.

To determine when these bounds coincide, take the constant-one test. Its stratum cost is
$C^h_{ij}=\omega_h\kappa_q(a_i^h,b_j^h)$. Equation~\eqref{eq:reverse_monge} yields the exact
rectangular $2\times2$ cross-difference within each fixed type-and-weight stratum:
\begin{equation}
  C^h_{ij}+C^h_{i'j'}-C^h_{ij'}-C^h_{i'j}
  =\omega_h\sum_{x=a_i^h}^{a_{i'}^h-1}
  \sum_{y=b_j^h}^{b_{j'}^h-1}
  \binom{x+y}{q-2}.
  \label{eq:qmass_circulation_proof}
\end{equation}
Theorem~\ref{thm:completion_polytope} makes vanishing of all these quantities
necessary and sufficient for point identification. If both lists vary, take
$a_{i'}^h=\max_i a_i^h$ and $b_{j'}^h=\max_j b_j^h$ with strictly smaller
$a_i^h,b_j^h$. When $a_{i'}^h+b_{j'}^h\ge q$, the summand at
$(x,y)=(a_{i'}^h-1,b_{j'}^h-1)$ is positive, so point identification fails. For
$q=2$ every nonempty summand equals one. Thus, at $q=2$, all
cross-differences vanish exactly when at least one size list is constant in each
stratum. For general $q$, if one list is constant then all rectangles are
empty in that direction. If both vary but
$\max_i a_i^h+\max_j b_j^h<q$, every mixed $q$-node-subset count is zero. Together with
the preceding strict argument, these are the only point-identified cases.
\hfill$\square$

\subsection{Structured Tractable Class: Anti-Monge Bounds}
\label{app:proof-sorting}

Vandermonde's identity gives Equation~\eqref{eq:vandermonde_mixed} by removing the two
within-fragment terms. For $b_q(n)=\binom nq$, Pascal's identity gives
$b_q(n+2)-2b_q(n+1)+b_q(n)=\binom n{q-2}$. Summing this difference over
$x=a,\ldots,a'-1$ and $y=b,\ldots,b'-1$ proves Equation~\eqref{eq:reverse_monge},
since row-only and column-only terms cancel.

Within each stratum, sort $a_1^h\le\cdots\le a_{m_h}^h$. If a proposed matching pairs $a_i^h\le
  a_{i'}^h$ with $b_j^h>b_{j'}^h$, the anti-Monge inequality shows that swapping the two targets
weakly increases the total mixed count. Repeating adjacent swaps proves that equal-order
pairing maximizes it. Applying the same argument to the reversed $b$ order proves the minimum.
Non-negative retained weights preserve the inequality within each fixed stratum, and strata
add. With invariant local mass $M_q^{\rm obs}$, the completed mass under
$\boldsymbol\varpi=(\varpi_h)_h$ is $T_q(\boldsymbol\varpi)=M_q^{\rm obs}+
  \sum_h\omega_h\sum_i\kappa_q(a_i^h,b_{\varpi_h(i)}^h)$. Since $1-M_q^{\rm obs}/T_q$ is
nondecreasing in $T_q$ for $T_q>0$, the claimed truncation bounds follow. If $T_q^-=0$, then
$M_q^{\rm obs}=0$: the zero-total convention gives the lower endpoint zero, while the upper
endpoint is one if $T_q^+>0$ and zero otherwise, exactly as in Equation~\eqref{eq:tau_bounds}.
Sorting the fragment-size lists computes both sharp endpoints in $O(m\log m)$ time.
\hfill$\square$

\subsection{Multi-client Computational Boundary}
\label{app:three-client-boundary}

For three clients, anchoring the first fragment index leaves two permutations
$\varpi,\varsigma\in\operatorname{Sym}(m)$ and the objective
\begin{equation}
  \Phi_{\rm obs}+\sum_{i=1}^{m}C_{i,\varpi(i),\varsigma(i)}.
  \label{eq:three_index_assignment}
\end{equation}
For any finite tensor $T$, create unit-weight
singleton fragments $\{a_i\},\{b_j\},\{c_k\}$ in one type/weight stratum, set $q=3$, fix one
type $r$, define $\psi_{3,r}(\{a_i,b_j,c_k\},X_{\{a_i,b_j,c_k\}})=T_{ijk}$, and set the test
function to zero on every other triple. Each completed hyperedge has exactly one $q$-node
subset, so its cost tensor is exactly $T$. The $m^3$ test-function table is an explicit
polynomial-size part of the endpoint instance. For a three-dimensional-matching instance,
choose $T_{ijk}=0$ on an allowed triple and $T_{ijk}=1$ otherwise. A completion with value zero
exists if and only if the three-dimensional-matching instance has a perfect matching. The
latter is NP-complete \citep{karp1972reducibility}. Hence the lower endpoint is NP-hard for
three clients. More clients contain this construction by adding singleton fragment lists and
defining the $q=3$ test to be zero on every triple involving an added client. Their
permutations leave the objective unchanged. Together, all these prove hardness for general linear
co-membership targets.

For a sparse binary-reward target, correspondence reduces to hypergraph matching. The
following result separates the effects of client count and desired matching size.

\begin{theorem}[Parameterized complexity of hyperedge matching]
  \label{thm:span-parameterized}
  In one exact all-client stratum, let each of $K\ge2$ clients have $m$
  distinct singleton fragments, with one common type and unit weight. Set $q=K$.
  The input lists favorable tuples $\mathcal E_+\subseteq[m]^K$ and assigns
  the order-$K$ query value one to them and zero to all other tuples.
  This sparse default-zero encoding includes all client-row lists, so
  $Km$ is bounded by the input length.
  For permutations $\varpi_2,\ldots,\varpi_K$ of $[m]$, define
  \begin{equation}
    \Phi_+(\mathcal E_+)=
    \max_{\varpi_2,\ldots,\varpi_K}
    \sum_{i=1}^m
    \mathbf1\{(i,\varpi_2(i),\ldots,\varpi_K(i))\in\mathcal E_+\}.
    \label{eq:sparse-reward-endpoint}
  \end{equation}
  For $1\le b\le m$, deciding $\Phi_+\ge b$ is polynomial-time solvable
  when $K=2$ and NP-complete for every fixed $K\ge3$.
  It is fixed-parameter tractable in $Kb$, and hence in $b$ for fixed $K$.
  When $K$ is unbounded it is W[1]-complete parameterized by $b$.
  A direct randomized algorithm has one-sided error at most
  $\delta_{\rm cc}\in(0,1)$ and time
  $O((2e)^{Kb}\operatorname{poly}(L_{\rm in})
    \log(2/\delta_{\rm cc}))$, where $L_{\rm in}$ is the input length.
\end{theorem}

\begin{proof}
  Treat client-row pairs $(k,i)$ as vertices in $K$ disjoint parts.
  A favorable tuple is a $K$-uniform hyperedge containing one vertex from
  each part. Rewarded tuples in a completion are pairwise disjoint.
  Conversely, any $b$ disjoint favorable tuples extend to a full completion:
  each part has $m-b$ unused rows, which can be paired arbitrarily.
  Thus the endpoint decision is exactly the existence of a $b$-edge
  matching in the explicitly supplied $K$-partite hypergraph.
  For $K=2$ this is bipartite matching. For $K=3,b=m$ it is perfect
  three-dimensional matching \citep{karp1972reducibility}. For larger fixed
  $K$, copy the first coordinate into each new coordinate.
  This preserves disjointness and proves NP-hardness. A list of $b$ tuples
  is a polynomially checkable witness, establishing NP membership.

  For fixed-parameter tractability, put $s=Kb$ and color every client-row independently and
  uniformly with $s$ colors. The $s$ distinct rows of any fixed witness all receive different
  colors with probability $s!/s^s\ge e^{-s}$. Discard tuples whose vertices repeat a color. Use a
  Boolean dynamic program over color subsets, starting from the empty subset and adjoining the
  $K$ colors of any tuple disjoint from the current subset. Reaching all $s$ colors certifies
  exactly $b$ disjoint tuples. Disjoint color masks cannot share a row, so every reported witness
  is valid. One trial costs $O(2^s|\mathcal E_+|K+Km)$. Repeating $\lceil e^s\log(2/\delta_{\rm
      cc})\rceil$ times yields the claimed one-sided error and running time. Perfect-hash families
  derandomize color-coding to give deterministic fixed-parameter tractability
  \citep{DBLP:journals/jacm/AlonYZ95}.

  For W[1] membership, form the conflict graph with one vertex per favorable tuple and an edge
  between tuples sharing any client-row. The desired matching is an independent set of size $b$,
  a parameter-preserving reduction to Independent Set \citep{DBLP:journals/tcs/DowneyF95}. For hardness, take
  an Independent Set instance $H=(U,F_H)$ with $N=|U|$ vertices and target $b$. Create one
  private client and one client for each edge of $H$, adding a harmless client if needed to make
  $K\ge2$. Every client has $m=N$ rows. For each $v\in U$, list one favorable tuple $t_v$. Its
  private-client row is unique to $v$. At the client for edge $\{u,v\}$, tuples $t_u,t_v$ share a
  row, and all other tuples use distinct rows different from that shared row. Exactly $N-1$ rows
  are used there, which is feasible. A harmless client assigns distinct rows to all tuples. Two
  favorable tuples share a row precisely when their vertices in $H$ are adjacent. The
  construction is polynomial and preserves $b$. Therefore an independent set of size $b$ exists
  precisely when the constructed endpoint is at least $b$, proving W[1]-hardness and
  completeness.
\end{proof}

Here candidate hyperedges have rank $K$. The fixed-rank and variable-rank classifications
concern this sparse reward problem, including its supplied query function. They do not
classify arbitrary dense costs or neural-network training. For any client count,
separation across a designated anchor gives the following tractable exception.

\begin{proposition}[Additive multi-client costs with one designated anchor]
  \label{prop:star_separable_multiclient}
  Consider an exact all-client stratum with $m$ rows per client, one row from each
  compatible original hyperedge at every client, and independently feasible
  one-to-one matchings anchored at client 1.  A fully additive tensor
  $C_{i_1,\ldots,i_K}=\sum_k a^{(k)}_{i_k}$ is point identified.  When the cost
  adds across the anchor--client pairs,
  $C_{i_1,\ldots,i_K}=a_{i_1}+\sum_{k=2}^K C^{(k)}_{i_1i_k}$, each sharp endpoint
  is the constant $\sum_i a_i$ plus the corresponding optima of $K-1$ independent
  bipartite assignments.  Thus each endpoint costs
  $O(\sum_h(K-1)m_h^3)$ by solving one independent bipartite assignment
  for each anchor-client pair within every stratum.
\end{proposition}

\begin{proof}[Proof of Proposition~\ref{prop:star_separable_multiclient}]
  Anchoring client 1 makes the feasible set $\operatorname{Sym}(m)^{K-1}$.
  For a fully additive cost,
  \begin{align}
    \sum_{i=1}^m C_{i,\varpi_2(i),\ldots,\varpi_K(i)}
    &=\sum_{i=1}^m a_i^{(1)}
    +\sum_{k=2}^K\sum_{i=1}^m a_{\varpi_k(i)}^{(k)}\nonumber\\
    &=\sum_{k=1}^K\sum_{j=1}^m a_j^{(k)},
    \label{eq:fully_additive_identified}
  \end{align}
  because each $\varpi_k$ is a permutation, proving identification.
  Anchor-additive costs instead give
  \begin{equation}
    \sum_{i=1}^m a_i+\sum_{k=2}^K
    \sum_{i=1}^m C^{(k)}_{i,\varpi_k(i)}.
    \label{eq:star_assignment_factorization}
  \end{equation}
  Each summand has its own feasible permutation, so the total optimum is the sum
  of independent assignment optima. Their integrality proves sharpness and gives
  $O(\sum_h(K-1)m_h^3)$ time. Interactions among non-anchor indices would break
  this separation.
\end{proof}

\subsection{Boolean-zeta Inversion and Partition Censoring}
\label{app:proof-zeta}

The assignment results describe uncertainty about how fragments fit together. We next
justify the co-membership coordinates used to measure the information removed by
partitioning. Before partitioning, the full collection determines aggregate support weights.
After partitioning, each coordinate is either retained exactly or erased, under the
observation assumptions stated in Appendix~\ref{app:framework}.

\begin{proposition}[Invertible Boolean-zeta representation and censoring]
  \label{prop:zeta_mask}
  For every type $r\in\mathcal R$, the complete nonempty collection
  $M_{|S|,r}^{\mathcal H}(S)$
  determines the aggregate support--type weights through
  \begin{equation}
    w_r(A)=\sum_{B\supseteq A}(-1)^{|B|-|A|}
    M_{|B|,r}^{\mathcal H}(B).
    \label{eq:mobius}
  \end{equation}
  If multiplicity, exact retained weight, and type are observed, then
  Equation~\eqref{eq:mask_law} holds for every $q\ge c$.
\end{proposition}

\begin{proof}

  Fix a type $r$ and regard $w_r$ as a function on the Boolean lattice $2^V$, with zero mass on
  the empty relation. Equation~\eqref{eq:q_measure} is its upper zeta transform:
  $M_{|S|,r}^{\mathcal H}(S)=\sum_{A\supseteq S}w_r(A)$. The M\"obius function on the interval
  $[S,A]$ is $(-1)^{|A|-|S|}$, so Boolean-lattice M\"obius inversion gives
  Equation~\eqref{eq:mobius}. This inversion applies separately to each observed relation type.

  For the mask, a copy of $S$ contributed by relation instance $e$ appears in a local fragment if
  and only if all nodes in $S$ have the same client label. If they do, exactly one client
  contains that copy. Otherwise none does. Summing over all $A_e\supseteq S$ proves
  Equation~\eqref{eq:mask_law}. The complete and retained order-$q$ masses therefore satisfy the
  two identities below for every observed type $r$ and order $q$:
  \begin{align}
    \|M_{q,r}^{\mathcal H}\|_1
    &=\sum_{e:r_e=r}w_e\binom{|A_e|}{q},\nonumber\\
    \|\widetilde M_{q,r}(\cdot,z)\|_1
    &=\sum_{e:r_e=r}w_e\sum_{k=1}^K\binom{|A_e\cap V_k|}{q}.
    \label{eq:mass_counting}
  \end{align}
  Their ratio is exactly the retained order-$q$ co-membership mass
  in~\eqref{eq:tau} for type $r$.
\end{proof}

\subsection{Horvitz--Thompson Estimation}
\label{app:proof-ht}

The masking identity immediately suggests inverse-probability weighting under a known
random partition design. It also reveals the cost of this correction: coordinates that
rarely remain within one client acquire large variance. The theorem below gives both
statements for each coordinate and then sums their squared errors.

\begin{theorem}[Known randomized assignment]
  \label{thm:ht}
  Fix a latent hypergraph $\mathcal H$, an order $c\le q\le|V|$, a type
  $r\in\mathcal R$, and a node subset $S\in\binom Vq$. Let
  $\mathsf Z:V\to[K]$ follow the declared randomized partition design and assume
  $\pi_S=\Pr\{\mathsf Z|_S\text{ is constant}\}>0$. The coordinate estimator
  satisfies
  \begin{align}
    \widehat M_{q,r}(S)&=\widetilde M_{q,r}(S,\mathsf Z)/\pi_S,
    &\mathbb E_{\mathsf Z}\widehat M_{q,r}(S)&=M_{q,r}^{\mathcal H}(S),
    \label{eq:ht}\\
    \operatorname{Var}_{\mathsf Z}[\widehat M_{q,r}(S)]
    &=M_{q,r}^{\mathcal H}(S)^2(\pi_S^{-1}-1).&&
    \label{eq:ht_variance}
  \end{align}
  If every q-set has positive inclusion probability, the mean squared Frobenius error is
  \begin{equation}
    \mathbb E_{\mathsf Z}
    \|\widehat M_{q,r}-M_{q,r}^{\mathcal H}\|_F^2
    =\sum_{S\in\binom Vq}M_{q,r}^{\mathcal H}(S)^2(\pi_S^{-1}-1).
    \label{eq:ht_mse}
  \end{equation}
  Under uniform independent assignment, the relative coordinate variance of a
  nonzero coordinate is $K^{q-1}-1$.  If $\pi_S=0$, no estimator measurable from
  the censored order-$q$ co-membership array is design-unbiased for that coordinate
  over all type-$r$ hypergraphs with size-$q$ relation supports.
\end{theorem}

\begin{proof}
  Let $I_S(\mathsf Z)=\mathbf1\{\mathsf Z|_S\text{ is constant}\}$. Then
  $I_S(\mathsf Z)\sim\mathrm{Bernoulli}(\pi_S)$ and
  $\widetilde M_{q,r}(S,\mathsf Z)=
    I_S(\mathsf Z)M_{q,r}^{\mathcal H}(S)$, which gives
  \eqref{eq:ht}--\eqref{eq:ht_variance}.  Linearity of expectation applied to the
  sum of squared coordinate errors gives~\eqref{eq:ht_mse}.  Dependence among the
  indicators is allowed.  Uniform assignment has
  $\pi_S=K(1/K)^q=K^{1-q}$.  For $\pi_S=0$, compare a q-uniform hypergraph with
  and without a weight-$a$ edge on $S$.  Their censored order-$q$ co-membership arrays agree
  almost surely but the target coordinate differs by $a$.
\end{proof}

\subsection{Observational Equivalence and a Two-point Bound}
\label{app:proof-two-point}

Randomized-design unbiasedness concerns repeated releases of a fixed structure. With one
release held fixed, the question is different: can two compatible structures have different
target values? If they can, any estimator based on that release has the same output law
under both structures and must incur error on at least one. We first state the resulting
bound for the learner's coarsened release, then prove the underlying two-point argument.

\begin{corollary}[Identification-to-learning bridge under release coarsening]
  \label{cor:learner-fiber-lower-bound}
  Let $y_{\rm L}=\rho(y_{\rm A})$ and hold the released training labels
  $Y^{\rm tr}$ fixed across all candidate structures. For the same fixed
  $z$ and $\Gamma$, define
  $\mathcal F_{\rm L}(y_{\rm L},\Gamma)=
    \{\mathcal H\in\mathcal M_\Gamma:
    \rho(\mathcal O_{z,\Gamma}(\mathcal H))=y_{\rm L}\}$.
  Deterministic coarsening gives
  $\mathcal F_{\rm A}(y_{\rm A},\Gamma)\subseteq
    \mathcal F_{\rm L}(y_{\rm L},\Gamma)$. Let
  $\phi:\mathcal M_\Gamma\to\mathbb R$ be a scalar structural functional whose
  attainable endpoints on $\mathcal F_{\rm A}(y_{\rm A},\Gamma)$ are
  $\phi_-$ and $\phi_+$. The minimax absolute risk over randomized rules
  $\widehat\phi=a(y_{\rm L},Y^{\rm tr},\xi)$, with algorithmic randomness
  $\xi$ independent of $\mathcal H$, satisfies this bound over the compatible class:
  \begin{equation}
    \inf_a\ \sup_{\mathcal H\in\mathcal F_{\rm L}(y_{\rm L},\Gamma)}
    \mathbb E_\xi|\widehat\phi-\phi(\mathcal H)|
    \ge(\phi_+-\phi_-)/2.
    \label{eq:coarsened_release_minimax}
  \end{equation}
  The proof uses the two endpoint structures, which remain indistinguishable after
  coarsening and therefore induce the same conditional output law.

\end{corollary}

\begin{proposition}[Fixed-release observational equivalence]
  \label{thm:identifiability}
  Fix $z$, $\Gamma$, and a metric space $(\Theta,d)$. A functional
  $\phi:\mathcal M_\Gamma\to\Theta$ is point identified from the deterministic
  analyst release if and only if it is constant on every nonempty class
  $\mathcal F_{\rm A}(y_{\rm A},\Gamma)$. If
  $\mathcal H_0,\mathcal H_1\in\mathcal F_{\rm A}(y_{\rm A},\Gamma)$ and
  $d(\phi(\mathcal H_0),\phi(\mathcal H_1))=\Delta$, every randomized estimator
  $\widehat\phi=a(y_{\rm A},\xi)$ with $\xi\sim\nu$ independent of the latent
  hypergraph obeys
  \begin{equation}
    \max_{b\in\{0,1\}}\mathbb E_{\xi}
    d(\widehat\phi,\phi(\mathcal H_b))\ge\Delta/2.
    \label{eq:two_point}
  \end{equation}
\end{proposition}

\begin{proof}

  If $\phi=g\circ\mathcal O_{z,\Gamma}$, equal analyst releases imply equal values of $\phi$, so
  $\phi$ is constant on each compatible class. Conversely, if $\phi$ is constant on every class,
  define $g(y_{\rm A})$ as that common value for any $\mathcal H\in\mathcal F_{\rm A}(y_{\rm
        A},\Gamma)$. Constancy makes this definition independent of the representative.

  For the lower bound, $\widehat\phi=a(y_{\rm A},\xi)$ has the same distribution at both
  compatible structures. Pointwise, the triangle inequality yields the following deterministic
  relation between the two endpoint errors:
  \begin{equation}
    \Delta\le d(\widehat\phi,\phi(\mathcal H_0))+
    d(\widehat\phi,\phi(\mathcal H_1)).
  \end{equation}
  Taking expectation over the common law $\nu$ of $\xi$ shows that one endpoint
  risk reaches at least $\Delta/2$.
\end{proof}

\begin{proof}[Proof of Corollary~\ref{cor:learner-fiber-lower-bound}]
  If $\mathcal H\in\mathcal F_{\rm A}(y_{\rm A},\Gamma)$, applying $\rho$ to its
  analyst release gives $y_{\rm L}$, hence
  $\mathcal H\in\mathcal F_{\rm L}(y_{\rm L},\Gamma)$.
  Theorem~\ref{thm:completion_polytope} supplies compatible structures attaining
  $\phi_-$ and $\phi_+$. A rule measurable from
  $\sigma(Y_{\rm L},Y^{\rm tr},\xi)$ has one conditional output law under
  both endpoint structures because their training labels are held fixed.  Their
  two absolute errors sum to at least $\phi_+-\phi_-$, so one endpoint risk
  reaches $(\phi_+-\phi_-)/2$. This holds for every measurable rule, and taking the infimum gives
  Equation~\eqref{eq:coarsened_release_minimax}.
\end{proof}

\begin{corollary}[Matching under hyperedge aliasing]
  \label{cor:matching}
  In the two-client construction of Section~\ref{sec:identifiability}, assume
  each side's $m$ cross-status fragment rows are pairwise distinct and nonempty,
  share uninformative type and weight marks, and have no shared cross-client hyperedge
  identifier. A latent matching
  $\varpi\sim\operatorname{Unif}(\operatorname{Sym}(m))$ is
  recovered completely with probability at most $1/m!$.  An estimator returning
  $s$ proposed matches has at most $s/m$ correct matches in expectation.
\end{corollary}

\begin{proof}
  The observation is invariant under all $m!$ permutations.  Conditional on it,
  each full permutation is equiprobable.  A fixed proposed pair is correct with
  probability $1/m$.  Linearity of expectation yields $s/m$ without an
  independence assumption across all returned proposals.
\end{proof}

\subsection{Released-data Factorization and Observable Spectral Approximation}
\label{app:proof-heat-stability}

We establish that FedASTRA's complete trajectory factors through the learner release,
then characterize the stability and expressiveness of its observable heat basis.

\begin{proposition}[Permutation invariance and heat stability]
  \label{prop:heat-release-stability}
  For every relation-column permutation $\Pi$, $P_k(B_k\Pi)=P_k(B_k)$ and
  $W_{k,\theta}$ is unchanged. On the nonisolated block,
  $P_k^{\rm sym}=D_{v,k}^{1/2}P_kD_{v,k}^{-1/2}$ has spectrum in $[0,1]$.
  Hence $\|e^{t(P_k^{\rm sym}-I)}\|_2\le1$ for $t\ge0$, and two such
  symmetric representatives on the same nodes satisfy
  \begin{equation}
    \|e^{t(P^{\rm sym}-I)}-e^{t(\widetilde P^{\rm sym}-I)}\|_2
    \le t\|P^{\rm sym}-\widetilde P^{\rm sym}\|_2.
    \label{eq:heat_stability}
  \end{equation}
\end{proposition}

\begin{proof}
  For $B'=B\Pi$, $D'_e=\Pi^\top D_e\Pi$, while $D_v$ and $I_0$ are fixed. Thus
  \begin{equation}
    B'(D'_e)^{-1}(B')^\top
    =B\Pi(\Pi^\top D_e\Pi)^{-1}\Pi^\top B^\top=BD_e^{-1}B^\top,
    \label{eq:heat-column-invariance}
  \end{equation}
  proving invariance of $P$ and all heat features. On the nonisolated block,
  \begin{equation}
    P^{\rm sym}=D_v^{-1/2}BD_e^{-1}B^\top D_v^{-1/2}=GG^\top,
    \qquad G=D_v^{-1/2}BD_e^{-1/2}.
    \label{eq:heat-symmetric-similarity}
  \end{equation}
  Positive semidefiniteness and similarity to the row-stochastic $P$ give
  spectrum in $[0,1]$ and contraction. Isolated rows are absorbing.
  Duhamel's identity gives
  \begin{equation}
    e^{t(P^{\rm sym}-I)}-e^{t(\widetilde P^{\rm sym}-I)}
    =\int_0^t e^{(t-s)(P^{\rm sym}-I)}
    (P^{\rm sym}-\widetilde P^{\rm sym})
    e^{s(\widetilde P^{\rm sym}-I)}\,ds,
    \label{eq:heat-duhamel}
  \end{equation}
  whose norm is bounded by Equation~\eqref{eq:heat_stability}.
\end{proof}

For one release, contraction in symmetric coordinates is equivalent to contraction in
$\|x\|_{D_v}=\|D_v^{1/2}x\|_2$; the perturbation bound compares symmetric representatives.
In the two-client model, joining $F_{1i}^h$ to $F_{2,\varpi_h(i)}^h$ gives
\begin{equation}
  \{\mathcal H_{\boldsymbol\varpi}:\boldsymbol\varpi\in
  \textstyle\prod_h\operatorname{Sym}(m_h)\}
  \subseteq\mathcal F_{\rm A}(y_{\rm A},\Gamma)
  \subseteq\mathcal F_{\rm L}(\rho(y_{\rm A}),\Gamma).
  \label{eq:shared_repairing_fiber}
\end{equation}
These matchings range over Equation~\eqref{eq:identified_functional_set} while preserving
all local fragment multisets and hence the learner release.

\begin{theorem}[Factorization of the complete federated learning trajectory]
  \label{thm:fedastra-transcript-factorization}
  Let $\mathcal P$ fix the task split, training and validation labels,
  participation schedule, training-count aggregation weights, checkpoint rule, and
  hyperparameters; let $\xi$ collect initialization, masking, minibatch, and optimizer
  randomness. Under exact arithmetic, for any client count and horizon $T$, the transcript
  $\mathcal T_T$ of all broadcasts, online, EMA and optimizer states, uploads, selected
  checkpoint, and final logits admits a measurable map
  \begin{equation}
    \mathcal T_T(\mathcal H,\xi)
    =F_T\!\left(\rho(\mathcal O_{z,\Gamma}(\mathcal H)),\mathcal P,\xi\right).
    \label{eq:fedastra-transcript-factorization}
  \end{equation}
  Thus compatible completions have identical transcripts under coupled $\xi$, and identical
  conditional laws under a common $\mathcal H$-independent distribution for $\xi$.
\end{theorem}

\begin{proof}
  Equation~\eqref{eq:heat-column-invariance} makes $P_k$ a measurable function of
  $[B_k]_{\rm col}$. Initialization, masked attributes, heat features, losses and gradients
  therefore depend only on $(Y_{\rm L},\mathcal P,\xi)$ and current states. An optimizer
  step and the EMA recursion preserve this dependence. Induction over local steps gives
  the uploads; training-count-weighted aggregation gives the next broadcast. Induction
  over rounds, followed by the fixed validation rule and final forward pass, proves the
  factorization. Coupling, then integration over $\xi$, gives the two equalities.
\end{proof}

\begin{corollary}[Downstream invariance on learner compatibility classes]
  \label{cor:completion-invariance}
  Under the preceding assumptions, equal learner releases imply, throughout training,
  \begin{equation}
    \theta^u(\mathcal H)=\theta^u(\mathcal H'),\qquad
    \operatorname{logit}^{\,u}_{k,v}(\mathcal H)
    =\operatorname{logit}^{\,u}_{k,v}(\mathcal H').
    \label{eq:downstream_fiber_invariance}
  \end{equation}
  This includes intermediate local and server states, validation selection, and final
  predictions, with equal conditional laws after averaging over $\xi$.
\end{corollary}
\begin{proof}
  Apply Theorem~\ref{thm:fedastra-transcript-factorization}; fixed evaluation labels also
  give identical Accuracy and Macro-F1.
\end{proof}

\begin{corollary}[Density within observable spectral filters]
  \label{cor:heat-spectral-universality}
  For symmetric $P^{\rm sym}(Y_{\rm L})$ with spectrum in $[0,1]$, every
  $f\in C([0,1])$, $\delta>0$, and $\Delta>\epsilon_{\rm ord}$ admit a finite
  $M_f$ and coefficients depending only on $(f,\delta,\Delta)$ such that
  \begin{equation}
    \left\|f(P^{\rm sym})-a_0I-
    \sum_{b=1}^{M_f}a_b e^{b\Delta(P^{\rm sym}-I)}\right\|_2<\delta.
    \label{eq:heat_spectral_density}
  \end{equation}
  The grid $t_b=b\Delta$ is feasible in Equation~\eqref{eq:ordered_heat_times}.
\end{corollary}
\begin{proof}
  The change of variable $x=e^{\Delta(\lambda-1)}$ maps $[0,1]$ homeomorphically
  to $[e^{-\Delta},1]$. Polynomial approximation of $f(1+\log(x)/\Delta)$ gives
  $p(x)=\sum_{b=0}^{M_f}a_bx^b$ with uniform error below $\delta$.
  The spectral theorem yields Equation~\eqref{eq:heat_spectral_density}; choosing each
  softplus increment as $\Delta-\epsilon_{\rm ord}>0$ realizes the grid.
\end{proof}

\begin{corollary}[Information-boundary compliance and observable spectral approximation]
  \label{cor:combined-learning-boundary}
  On a positive-degree client block, define
  $f(P)=D_v^{-1/2}f(P^{\rm sym})D_v^{1/2}$ and
  $G_{M_f}(P)=a_0I+\sum_{b=1}^{M_f}a_be^{b\Delta(P-I)}$ using the preceding coefficients.
  For every conformable $U$,
  \begin{equation}
    \|(f(P)-G_{M_f}(P))U\|_{D_v,F}\le\delta\|U\|_{D_v,F},
    \qquad \|U\|_{D_v,F}=\|D_v^{1/2}U\|_F.
    \label{eq:observable-filter-feature-bound}
  \end{equation}
  Both these transforms and the complete training transcript are shared across the
  learner compatibility class, for fixed $\mathcal P$ and $\xi$.
\end{corollary}
\begin{proof}
  Conjugate Equation~\eqref{eq:heat_spectral_density} by $D_v^{1/2}$ and use
  $\|AU\|_F\le\|A\|_2\|U\|_F$. Invariance follows from
  Theorem~\ref{thm:fedastra-transcript-factorization}.
\end{proof}
Density concerns the union of finite basis widths. It gives neither fixed-$R$ universality
nor a prediction or optimization guarantee for a fitted classifier.

\paragraph{Certified evaluation.}
All times share the recurrence $U,PU,P^2U,\ldots$. For the auxiliary analytic variable
$N_t^{\rm aux}\sim\operatorname{Poisson}(t)$, deterministic tail-collapsed evaluation is
\begin{equation}
  \widehat{\mathsf K}_{t,L}U=
  \sum_{\ell=0}^{L-1}e^{-t}\frac{t^\ell}{\ell!}P^\ell U+
  \left(1-\sum_{\ell=0}^{L-1}e^{-t}\frac{t^\ell}{\ell!}\right)P^LU.
  \label{eq:tail-collapsed-evaluator}
\end{equation}
For active times $\mathcal T_{\rm heat}=\{t_{bj}:b\in[R],j\in[d]\}$, choose the
smallest integer $L_{\rm tail}$ satisfying
\begin{equation}
  \max_{t\in\mathcal T_{\rm heat}}\Pr(N_t^{\rm aux}>L_{\rm tail})\le\varepsilon_{\rm tail},
  \qquad
  \max_{t\in\mathcal T_{\rm heat}}\{\Pr(N_t^{\rm aux}=L_{\rm tail})
  +2\Pr(N_t^{\rm aux}>L_{\rm tail})\}\le\varepsilon_{\rm tail}.
  \label{eq:poisson_tail_certificate}
\end{equation}
These bound forward and time-derivative coefficient masses. We use
$\varepsilon_{\rm tail}=10^{-6}$, float64 log-space probabilities, and record both bounds
and $L_{\rm tail}$ for each validation-selected native checkpoint. Online and EMA branches
use their own times. At $R=0$, bypass the recurrence and set $L_{\rm tail}=0$; configurations
requiring $L_{\rm tail}>4096$ are refused. Floating-point propagation and accumulation
errors are separate from this coefficient certificate.

\subsection{Observable Diffusion and Predictive Dynamics}
\label{app:observable-dynamics}

Diffusion time evolves features at fixed incidence and encoder; local training steps evolve
the encoder and its EMA target. The following results distinguish these two processes.

\begin{theorem}[Continuous-time diffusion under observational equivalence]
  \label{thm:fiber-heat-orbit}
  Fix $y_{\rm L}$, client $k$, and $U_k^0\in\mathbb R^{|V_k|\times d}$. Put
  $d_{k,v}=(D_{v,k})_{vv}$ and $V_{k,+}=\{v:d_{k,v}>0\}$. Let $D_{k,+}$ be
  $D_{v,k}$ restricted to these nodes, and use $+$ likewise for other restrictions.
  For $t\ge0$, the trajectory
  \begin{equation}
    U_k(t)=e^{t(P_k-I)}U_k^0
    \label{eq:fiber-heat-orbit}
  \end{equation}
  uniquely solves $\dot U_k=(P_k-I)U_k$, $U_k(0)=U_k^0$. For the incidence-induced
  components $\mathcal G_k$ on $V_{k,+}$, define
  \begin{equation}
    (\mathsf Q_{k,+}U_{k,+}^0)_v=
    \frac{\sum_{u\in C}d_{k,u}U_{k,u}^0}{\sum_{u\in C}d_{k,u}},
    \qquad v\in C\in\mathcal G_k.
    \label{eq:component-stationary-projection}
  \end{equation}
  Writing $A=U_{k,+}(t)-\mathsf Q_{k,+}U_{k,+}^0$ and
  $\|A\|_{D_{k,+},F}^2=\operatorname{tr}(A^\top D_{k,+}A)$ gives
  \begin{align}
    \frac{d}{dt}\frac12\|A\|_{D_{k,+},F}^2
    &=-\operatorname{tr}[A^\top D_{k,+}(I-P_{k,+})A]\notag\\
    &=-\frac12\sum_e\frac1{d_{e,k}}
    \sum_{u,v\in e}\|U_{k,u}(t)-U_{k,v}(t)\|_2^2\le0,
    \label{eq:heat-energy-dissipation}
  \end{align}
  where $e$ ranges over retained columns and $d_{e,k}=(D_{e,k})_{ee}$.
  If the nonstationary subspace is nonzero, let
  \begin{equation}
    \lambda_{{\rm gap},k}=\min\{1-\lambda:
    \lambda\in\operatorname{spec}(P_{k,+}^{\rm sym}),\ \lambda<1\}>0,
    \quad P_{k,+}^{\rm sym}=D_{k,+}^{1/2}P_{k,+}D_{k,+}^{-1/2}.
    \label{eq:local-heat-gap}
  \end{equation}
  Then
  \begin{equation}
    \|U_{k,+}(t)-\mathsf Q_{k,+}U_{k,+}^0\|_{D_{k,+},F}
    \le e^{-\lambda_{{\rm gap},k}t}
    \|U_{k,+}^0-\mathsf Q_{k,+}U_{k,+}^0\|_{D_{k,+},F}.
    \label{eq:heat-gap-convergence}
  \end{equation}
  A zero nonstationary subspace gives zero residual. Isolated nodes stay fixed, and
  compatible completions with equal initial features have identical trajectories.
\end{theorem}
\begin{proof}
  The matrix exponential uniquely solves the finite linear system, and column invariance
  gives the common trajectory. On $V_{k,+}$,
  $D_{k,+}P_{k,+}=B_{k,+}D_{e,k}^{-1}B_{k,+}^\top$ is symmetric and
  $P_{k,+}\mathsf Q_{k,+}=\mathsf Q_{k,+}P_{k,+}=\mathsf Q_{k,+}$. Differentiating
  the energy and expanding the quadratic form gives
  \begin{equation}
    \operatorname{tr}[A^\top D_{k,+}(I-P_{k,+})A]
    =\frac12\sum_e\frac1{d_{e,k}}\sum_{u,v\in e}\|A_u-A_v\|_2^2.
    \label{eq:incidence-dirichlet-form}
  \end{equation}
  Its null space consists of componentwise constant matrices. In symmetric coordinates,
  $\mathsf Q_{k,+}$ is the orthogonal eigenvalue-one projection; every other heat
  multiplier is at most $e^{-\lambda_{{\rm gap},k}t}$. The spectral theorem proves
  convergence. The isolated block is the identity.
\end{proof}

\begin{corollary}[Ordered sampling of the diffusion trajectory]
  \label{cor:ordered-heat-orbit}
  Equation~\eqref{eq:ordered_heat_times} gives, for each coordinate $j$,
  \begin{equation}
    0=t_{0j}<t_{1j}<\cdots<t_{Rj},\qquad
    W_{k,\theta}^{(j)}=[U_{k,:j}(t_{0j}),\ldots,U_{k,:j}(t_{Rj})],
    \label{eq:ordered-orbit-samples}
  \end{equation}
  where $W_{k,\theta}^{(j)}$ selects columns $j,d+j,\ldots,Rd+j$ and
  $U_{k,:j}(t)=\mathsf K_k(t)U_{k,:j}^0$. The feasible grid $t_b=b\Delta$,
  $\Delta>\epsilon_{\rm ord}$, yields the dense algebra
  \begin{equation}
    \operatorname{span}\{1,e^{b\Delta(\lambda-1)}:b\ge1\}\subset C([0,1]).
    \label{eq:ordered-grid-algebra}
  \end{equation}
  Density over increasing finite widths follows from
  Corollary~\ref{cor:heat-spectral-universality}.
\end{corollary}

\paragraph{Masked prediction and target dynamics.}
Within round $u$, let $s$ index local steps, $o$ denote optimizer state,
$Y_{{\rm L},k}=(X_k,[B_k]_{\rm col})$, and $\xi_{k,u,s}$ contain masks, minibatches,
and optimizer randomness. The update has the form
\begin{equation}
  (\theta_{k,u,s},o_{k,u,s})=\Psi_{k,u,s}\!\left(
  \theta_{k,u,s-1},o_{k,u,s-1},\bar\vartheta_{k,u,s-1},
  Y_{{\rm L},k},\mathcal P,\xi_{k,u,s}\right).
  \label{eq:local-state-recursion}
\end{equation}
With $\vartheta_{k,u,s}=\operatorname{ctx}(\theta_{k,u,s})$ containing the encoder and
raw time increments, the target satisfies
\begin{equation}
  \bar\vartheta_{k,u,s}
  =\mu\bar\vartheta_{k,u,s-1}+(1-\mu)\vartheta_{k,u,s}
  =\mu^s\bar\vartheta_{k,u,0}+(1-\mu)\sum_{a=1}^s\mu^{s-a}\vartheta_{k,u,a}.
  \label{eq:ema-fiber-recursion}
\end{equation}
For $e_{k,u,s}=\vartheta_{k,u,s}-\bar\vartheta_{k,u,s}$ and
$\Delta_{k,u,s}=\vartheta_{k,u,s}-\vartheta_{k,u,s-1}$, subtraction and iteration give
\begin{equation}
  e_{k,u,s}=\mu e_{k,u,s-1}+\mu\Delta_{k,u,s},\qquad
  \|e_{k,u,s}\|\le\mu^s\|e_{k,u,0}\|
  +\mu\sum_{a=1}^s\mu^{s-a}\|\Delta_{k,u,a}\|.
  \label{eq:ema-lag-dynamics}
\end{equation}
Thus a frozen online context has lag contraction factor $\mu$. The input masking moments
in Equation~\eqref{eq:mask_variance} do not imply unbiased nonlinear encodings or gradients.

\begin{lemma}[Observation invariance throughout training]
  \label{lem:fiber-dynamics-closure}
  For fixed $\mathcal P$ and coupled $\xi$, the masked loss, online, optimizer and EMA
  states, uploads, broadcasts, selected checkpoint, and final logits are measurable in
  $(Y_{\rm L},\mathcal P,\xi)$ and agree across its compatibility class.
\end{lemma}
\begin{proof}
  This is Theorem~\ref{thm:fedastra-transcript-factorization} applied to the local recursions
  above and the aggregation
  \begin{equation}
    \theta_k^{u+1}=\mathsf U_{k,u}(\theta^u,Y_{{\rm L},k},\mathcal P,\xi_{k,u}),
    \qquad\theta^{u+1}=
    \frac{\sum_{k\in\mathcal C_u}n_k^{\rm tr}\theta_k^{u+1}}
    {\sum_{k\in\mathcal C_u}n_k^{\rm tr}}.
    \label{eq:federated-state-closure}
  \end{equation}
\end{proof}
Diffusion dissipation and EMA lag concern distinct dynamics; neither is a Lyapunov function
for joint nonlinear training. Observation invariance also supplies no formal privacy
guarantee for communicated updates.

\subsection{Downstream Risk with Additional Information}
\label{app:auxiliary-risk}

We compare fitted predictors under one joint law of data, training randomness, and held-out
query nodes. Conditional averaging and information comparison are classical
\citep{blackwell1953experiments,chen2020augmentation}; here they connect prediction variation
to the assignment geometry of Theorem~\ref{thm:completion_polytope}.

\paragraph{Information and fitted risk.}
For a queried node with class $y\in[C]$, $C\ge2$, let $O_k$ contain client features and
anonymous incidence, the query, the parameter-only transcript through validation selection,
and exogenous randomness. Held-out labels, test outcomes, auxiliary messages, and states
trained using those messages are excluded. Set $\mathcal S_0=\sigma(O_k)$; the FedASTRA
predictor $F$ is $\mathcal S_0$-measurable. This permits federated parameter exchange while
keeping structural computation local. Let $A_{\rm aux}$ include arbitrary additional
information and any training state needed for an enriched predictor $G$, measurable in
$\mathcal S_1=\mathcal S_0\vee\sigma(A_{\rm aux})$. The predictors may have different
learned parameters. A conditional structural audit may fix analyst releases and cost arrays
in the base, but such conditioning does not grant deployment access to them.

Define, with one-hot label $\boldsymbol e_y$,
\begin{gather}
  p_0=\mathbb E[\boldsymbol e_y\mid\mathcal S_0],\qquad
  p_1=\mathbb E[\boldsymbol e_y\mid\mathcal S_1],\qquad
  J_C=I_C-\mathbf1\mathbf1^\top/C,
  \label{eq:risk-information-fields}\\
  s_F=J_C\operatorname{logit}_F,\qquad s_G=J_C\operatorname{logit}_G,\qquad
  \bar s_G=\mathbb E[s_G\mid\mathcal S_0].
  \label{eq:risk-centered-logits}
\end{gather}
Assume square-integrable logits. Centering preserves softmax. Write
$\operatorname{lse}(s)=\log\sum_ce^{s_c}$,
$\mathcal R_{\rm CE}(s)=\mathbb E[\operatorname{lse}(s)-s_y]$, and
$\mathcal R_{\log}(q)=\mathbb E[-\log q_y]$ with the usual convention at zero.
Conditional log risk equals conditional entropy plus expected KL divergence. Consequently,
\begin{equation}
  \inf_{q\ \mathcal S_0\text{-measurable}}\mathcal R_{\log}(q)
  -\inf_{q\ \mathcal S_1\text{-measurable}}\mathcal R_{\log}(q)
  =\mathcal I_{\rm add}:=\mathbb E D_{\rm KL}(p_1\|p_0)\ge0.
  \label{eq:risk-information-monotonicity}
\end{equation}
Indeed, the minimizers are $p_0,p_1$ and the tower property identifies the entropy
difference. For fitted probability rules of finite log risk, put
$\mathcal E_F=\mathbb E D_{\rm KL}(p_0\|q_F)$ and
$\mathcal E_G=\mathbb E D_{\rm KL}(p_1\|q_G)$. Subtracting
$\mathcal R_{\log}(q_F)=H(y\mid\mathcal S_0)+\mathcal E_F$ from
$\mathcal R_{\log}(q_G)=H(y\mid\mathcal S_1)+\mathcal E_G$ proves
Equation~\eqref{eq:main-local-advantage}. Thus $F$ strictly improves on $G$ exactly when
$\mathcal E_G-\mathcal E_F>\mathcal I_{\rm add}$.

\begin{theorem}[Cross-entropy comparison with arbitrary additional information]
  \label{thm:auxiliary-ce-risk}
  In the model above, define
  \begin{gather}
    V_{\rm aux}=\mathbb E\|s_G-\bar s_G\|_2^2,\qquad
    \Delta_{\rm label}=\mathbb E\|p_1-p_0\|_2^2,
    \label{eq:auxiliary-variation-signal}\\
    \delta_F=\mathcal R_{\rm CE}(s_F)-\mathcal R_{\rm CE}(\bar s_G),\qquad
    C_{\rm label}=\mathbb E\langle p_1-p_0,s_G-\bar s_G\rangle.
    \label{eq:auxiliary-risk-residual}
  \end{gather}
  For $D_{\rm lse}(a,b)=\operatorname{lse}(a)-\operatorname{lse}(b)
  -\langle\nabla\operatorname{lse}(b),a-b\rangle$,
  \begin{equation}
    \mathcal R_{\rm CE}(s_G)-\mathcal R_{\rm CE}(s_F)
    =\mathbb E D_{\rm lse}(s_G,\bar s_G)-C_{\rm label}-\delta_F.
    \label{eq:auxiliary-exact-ce}
  \end{equation}
  If $\|s_G\|_\infty\le B_{\rm logit}<\infty$ almost surely, then
  \begin{equation}
    \mathcal R_{\rm CE}(s_G)-\mathcal R_{\rm CE}(s_F)
    \ge\frac{e^{-2B_{\rm logit}}}{2C}V_{\rm aux}
    -\sqrt{\Delta_{\rm label}V_{\rm aux}}-\delta_F.
    \label{eq:auxiliary-ce-bound}
  \end{equation}
  A positive right side gives strictly smaller risk for $F$. Under label sufficiency
  $p_1=p_0$, the alignment term $C_{\rm label}$ vanishes.
\end{theorem}
\begin{proof}
  Put $r=s_G-\bar s_G$, so $\mathbb E[r\mid\mathcal S_0]=0$. The Bregman linear
  term has zero expectation, and
  \begin{equation}
    \mathbb E\langle\boldsymbol e_y,r\rangle
    =\mathbb E\langle p_1,r\rangle
    =\mathbb E\langle p_1-p_0,r\rangle=C_{\rm label}.
    \label{eq:auxiliary-ce-orthogonality}
  \end{equation}
Subtracting risks proves Equation~\eqref{eq:auxiliary-exact-ce};
in particular, taking $s_F=\bar s_G$ (so $\delta_F=0$) gives
\[
\mathcal R_{\rm CE}(s_G)-\mathcal R_{\rm CE}(\bar s_G)
=
\mathbb E D_{\rm lse}(s_G,\bar s_G)-C_{\rm label}.
\]
  For $p=\operatorname{softmax}(s)$ with $\|s\|_\infty\le B_{\rm logit}$,
  $p_c\ge e^{-2B_{\rm logit}}/C$. For $v\perp\mathbf1$ and $a=p^\top v$,
  \begin{equation}
    v^\top\nabla^2\operatorname{lse}(s)v
    =\sum_cp_c(v_c-a)^2
    \ge\frac{e^{-2B_{\rm logit}}}{C}\sum_c(v_c-a)^2
    \ge\frac{e^{-2B_{\rm logit}}}{C}\|v\|_2^2.
    \label{eq:auxiliary-ce-curvature}
  \end{equation}
  The segment from $\bar s_G$ to $s_G$ remains bounded and centered. Integral Taylor
  remainder gives $D_{\rm lse}(s_G,\bar s_G)\ge e^{-2B_{\rm logit}}\|r\|_2^2/(2C)$,
  proving the bound and Equation~\eqref{eq:main-strict-local-criterion}.
\end{proof}
The Bregman term charges variation around the base-conditional mean; $C_{\rm label}$
credits alignment with new label signal, and the signed $\delta_F$ retains fitting quality.
If $A_{\rm aux}$ is base-measurable, both $V_{\rm aux}$ and $\Delta_{\rm label}$ vanish.

\begin{proposition}[Assignment geometry of auxiliary prediction variation]
  \label{prop:assignment-logit-variance}
  Condition on base information fixing the cost arrays and intercepts. In the exact strata
  of Theorem~\ref{thm:completion_polytope}, declare independent uniform candidate
  permutations $\varpi_h\in\operatorname{Sym}(m_h)$ and affine logits
  \begin{equation}
    [\operatorname{logit}_G]_c=b_c+
    \sum_h\sum_{i=1}^{m_h}[C_c^h]_{i,\varpi_h(i)},\qquad c\in[C].
    \label{eq:auxiliary-assignment-logits}
  \end{equation}
  Set $\bar C^h=C^{-1}\sum_cC_c^h$ and
  $\widetilde C_c^h=J_{m_h}(C_c^h-\bar C^h)J_{m_h}$. Then
  \begin{equation}
    \mathbb E[\|s_G-\bar s_G\|_2^2\mid\mathcal S_0]
    =\sum_{h:m_h\ge2}\frac1{m_h-1}\sum_{c=1}^C\|\widetilde C_c^h\|_F^2.
    \label{eq:auxiliary-assignment-variance}
  \end{equation}
  Averaging gives $V_{\rm aux}$. A scalar cost $C^h$ has variance
  $\|J_{m_h}C^hJ_{m_h}\|_F^2/(m_h-1)$ for $m_h\ge2$; singleton strata have zero variance.
\end{proposition}
\begin{proof}
  Double centering a fixed cost $C_*$ changes assignment sums by a constant.
  For $D_*=J_mC_*J_m$ and uniform $\varpi$, $S=\sum_i[D_*]_{i,\varpi(i)}$ has zero
  mean. Equal-row terms in $\mathbb E S^2$ contribute $\|D_*\|_F^2/m$; distinct-row terms give
  \begin{equation}
    \frac1{m(m-1)}\sum_{i\ne i'}\sum_{j\ne j'}[D_*]_{ij}[D_*]_{i'j'}
    =\frac{\|D_*\|_F^2}{m(m-1)}.
    \label{eq:auxiliary-assignment-second-moment}
  \end{equation}
  The equality uses $\sum_{i'\ne i,j'\ne j}[D_*]_{i'j'}=[D_*]_{ij}$, from zero row
  and column sums. Adding gives $\operatorname{Var}(S)=\|D_*\|_F^2/(m-1)$.
  Class centering replaces $C_c^h$ by $C_c^h-\bar C^h$; summing the scalar identities
  over classes and independent strata proves the claim.
\end{proof}
Thus the same interaction detects nonidentification and measures variation under the
declared matching law. A scalar assignment interval of width $w_*$ gives variance at most
$w_*^2/4$. Uniform endpoint bounds for $C_c^h-\bar C^h$, together with uniformly bounded
centered intercepts, can supply $B_{\rm logit}$ across base realizations. Class centering
is essential because common logit shifts do not affect predictions. Nonlinear auxiliary
rules remain covered by Theorem~\ref{thm:auxiliary-ce-risk}.

\begin{corollary}[A conditional advantage for the fitted FedASTRA predictor]
  \label{cor:parameter-only-risk-advantage}
  Under the bounded-logit assumptions above, let $F_R$ be the fitted parameter-only
  FedASTRA predictor with $R$ positive heat channels. If
  $\delta_{F_R}\le\varepsilon_R$ for a stated $\varepsilon_R\ge0$, then
  \begin{equation}
    \frac{e^{-2B_{\rm logit}}}{2C}V_{\rm aux}
    -\sqrt{\Delta_{\rm label}V_{\rm aux}}>\varepsilon_R
    \quad\Longrightarrow\quad
    \mathcal R_{\rm CE}(s_{F_R})<\mathcal R_{\rm CE}(s_G).
    \label{eq:parameter-only-risk-threshold}
  \end{equation}
  For a conditionally label-independent candidate channel, the square-root term is zero;
  affine assignment logits have the variance in Proposition~\ref{prop:assignment-logit-variance}.
\end{corollary}
\begin{proof}
  Transcript factorization makes $F_R$ base-measurable. Substitute the residual bound
  into Equation~\eqref{eq:auxiliary-ce-bound}; conditional label independence gives $p_1=p_0$.
\end{proof}

Under label sufficiency $p_1=p_0$, equivalently $y\perp A_{\rm aux}\mid\mathcal S_0$,
the two information levels have the same conditional entropy. Hence, for
$p_{F_R}=\operatorname{softmax}(s_{F_R})$ and $p_G=\operatorname{softmax}(s_G)$,
\begin{equation}
  \mathcal R_{\log}(p_G)-\mathcal R_{\log}(p_{F_R})
  =\mathbb E D_{\rm KL}(p_0\|p_G)-\mathbb E D_{\rm KL}(p_0\|p_{F_R}).
  \label{eq:label-sufficient-risk-floor}
\end{equation}
If the fixed-$R$ class approximates $p_0$ arbitrarily well in expected KL, its infimum equals
both information-specific Bayes risks. A fitted error at most $\varepsilon_R$ strictly
improves on a specified $G$ with larger KL error. Label sufficiency, fixed-model
approximation, and fitting are separate assumptions; spectral density and correspondence
invariance imply none of them.

\begin{proposition}[Target direction of the masked alignment component]
  \label{prop:masked-conditional-direction}
  Freeze the EMA context $\bar\vartheta_k$ for one local update. Let $O_{\rm mask}$ be the
  predictor's masked representation and $U_{\rm tar}$ its full-view, layer-normalized and
  unit-normalized target, with $\|U_{\rm tar}\|_2=1$ almost surely. Then
  \begin{equation}
    \inf_{u:\,\|u(O_{\rm mask})\|_2=1}
    \mathbb E[1-\langle u(O_{\rm mask}),U_{\rm tar}\rangle]
    =1-\mathbb E\|\mathbb E[U_{\rm tar}\mid O_{\rm mask}]\|_2.
    \label{eq:masked-conditional-direction}
  \end{equation}
  For $m=\mathbb E[U_{\rm tar}\mid O_{\rm mask}]\ne0$, the minimizing direction is
  $m/\|m\|_2$; when $m=0$, every unit direction is optimal.
\end{proposition}
\begin{proof}
  Conditional expected alignment is $\langle u,m\rangle\le\|m\|_2$ by
  Cauchy--Schwarz, with equality in the stated directions. Integrate the optimum.
\end{proof}
This concerns the frozen-target alignment component with unrestricted directions.
Anti-collapse regularization, two-view agreement, moving targets, and finite predictors
remain in the full objective~\eqref{eq:heat_jepa_loss} and its optimization analysis.

\paragraph{Classification metrics.}
For base-measurable $p_F=\operatorname{softmax}(s_F)$, let
$h_F=\arg\max_cp_{F,c}$ and $h_0=\arg\max_cp_{0,c}$ with a fixed tie rule.
Comparing their selected coordinates and applying Pinsker and Cauchy--Schwarz bounds
classification regret by
$\mathbb E\|p_0-p_F\|_1\le\sqrt{2\mathbb E D_{\rm KL}(p_0\|p_F)}$.
This is regret against the base Bayes rule; a surrogate-risk ordering alone need not order
two fitted classifiers' accuracies. If additionally $p_1=p_0$, the top-two conditional
probability gap is at least $\kappa_{\rm cls}>0$, and $h_{F_R}=h_0$ almost surely, then
$\operatorname{Acc}(F_R)-\operatorname{Acc}(G)\ge
\kappa_{\rm cls}\Pr(h_G\ne h_0)$ by conditioning on $\mathcal S_1$.
Macro-F1 also depends on classwise confusion probabilities: increased true positives and
decreased false positives and false negatives suffice, whereas cross-entropy ordering
alone does not determine it. Both metrics use the same held-out evaluation protocol.

\subsection{Randomized Recovery from Diffusion Features and Local Prediction Risk}
\label{app:local-sufficiency}

Heat features retain permitted local information when they recover the diffusion operator
and its pairwise projection, and that projection determines the observed incidence. These
are separate conditions, independent of unknown correspondence across clients. Incidence
is binary; column multiplicities are retained unless simplicity is assumed.

\begin{lemma}[Projection and recovery from diffusion features]
  \label{lem:heat-observability}
  Let $B$ have $n$ rows and nonempty columns, and let $P$ be
  Equation~\eqref{eq:local_transition}, including absorbing isolated nodes.
  For $i\ne j$, $P_{ij}>0$ exactly when a column contains both nodes, so its
  off-diagonal support determines the pairwise projection $G(B)$.
  Fix $\Delta>0$, $d,\ell\ge1$, $U\in\mathbb R^{n\times d}$, and
  $H_\Delta=e^{\Delta(P-I)}$. Define
  \begin{equation}
    Q_-=[U,H_\Delta U,\ldots,H_\Delta^{\ell-1}U],\qquad
    Q_+=[H_\Delta U,\ldots,H_\Delta^\ell U].
    \label{eq:heat-observability-matrices}
  \end{equation}
  If $\operatorname{rank}(Q_-)=n$, then in exact arithmetic
  \begin{equation}
    H_\Delta=Q_+Q_-^\dagger,\qquad
    P=I+\Delta^{-1}\log H_\Delta,
    \label{eq:heat-inverse}
  \end{equation}
  using the Moore--Penrose inverse and principal matrix logarithm.
  Times $b\Delta$ are permitted common-time ordered heat channels when
  $\Delta>\epsilon_{\rm ord}$.
\end{lemma}

\begin{proof}
  For positive-degree $i$, $P_{ij}=d_i^{-1}\sum_eB_{ie}B_{je}/|e|$;
  nonnegativity proves the support claim, and isolated rows add no pairs.
  Since $Q_+=H_\Delta Q_-$, full row rank proves the first inverse.
  On positive-degree nodes, $P$ is similar to the positive semidefinite
  symmetric matrix $D_v^{-1/2}BD_e^{-1}B^\top D_v^{-1/2}$ and is stochastic;
  its spectrum lies in $[0,1]$. The isolated block is the identity.
  Thus $P$ is diagonalizable with real spectrum and
  $\log e^{\Delta(P-I)}=\Delta(P-I)$.
\end{proof}

\paragraph{Randomized full rank.}
Retain $U=[X,Z]$, where $Z\in\mathbb R^{n\times d_{\rm pr}}$ has independent
standard Gaussian entries independent of the hypergraph, attributes, and labels. Let
$n_{\rm eig}$ count the distinct eigenvalues of $P$, with largest multiplicity $g_{\max}$.
Then $\ell\ge n_{\rm eig}$ and $d_{\rm pr}\ge g_{\max}$ give
$\operatorname{rank}(Q_-)=n$ almost surely. Indeed, in $P=S\Lambda S^{-1}$,
each eigenspace block of $S^{-1}Z$ has full row rank almost surely. A vector annihilating
the Krylov matrix yields a Vandermonde system in the distinct eigenvalues of $H_\Delta$;
invertibility forces every eigenspace coefficient, and hence the vector, to vanish.
In particular, $d_{\rm pr}\ge n,\ell=1$ suffices for every $P$. These probes remain local
and are absent from reported runs; the finite attribute encoder needs its own rank check.

\paragraph{A reconstructible random family.}
Let $\mathcal E_n\subseteq\mathcal A_n$ be any random simple hypergraph, where candidate
supports $e\in\mathcal A_n$ satisfy $|e|\ge2$. Let $\Omega_2$ contain unordered distinct
edge pairs with intersection size at least two, and $\Omega_3$ unordered triples supporting
a Berge triangle: distinct $u,v,w$ and an ordering $e,f,g$ with
$\{u,v\}\subset e$, $\{v,w\}\subset f$, $\{w,u\}\subset g$. Define
\begin{equation}
  \varepsilon_n^{\rm all}=
  \sum_{\{e,f\}\in\Omega_2}\Pr(e,f\in\mathcal E_n)
  +\sum_{\{e,f,g\}\in\Omega_3}\Pr(e,f,g\in\mathcal E_n).
  \label{eq:general-random-hypergraph-collisions}
\end{equation}
These are actual joint inclusion probabilities; edges need not be independent, uniform
in size, or bounded in degree. The condition concerns the local incidence law, since
partitioning can create singletons or duplicates.

\begin{proposition}[Nonuniform and dependent reconstruction family]
  \label{prop:dependent-hypergraph-recovery}
  A deterministic decoder recovers $[B]_{\rm col}$ from $G(B)$ with probability
  at least $1-\varepsilon_n^{\rm all}$. For every projected edge $uv$, it forms
  $\{u,v\}\cup(N_G(u)\cap N_G(v))$ and removes repeated sets.
  An adjacency-matrix implementation costs $O(n^3)$ operations.
\end{proposition}

\begin{proof}
  The union bound excludes both configurations with probability at least
  $1-\varepsilon_n^{\rm all}$. Then each projected pair $uv$ belongs to a unique
  edge $e$. Its other members are common neighbors. An external common neighbor
  $w$ would give edges containing $uw$ and $vw$: if equal they overlap $e$ twice,
  and if distinct they form a Berge triangle with $e$. Thus each candidate equals
  $e$. Every edge contains a pair; scanning at most $n^2$ pairs and $n$ neighbors
  proves recovery and the cost.
\end{proof}

Incidence forests have $\varepsilon_n^{\rm all}=0$, permitting dependent, nonuniform
edges and order-$n$ edge counts. More generally, $\varepsilon_n^{\rm all}\to0$ suffices.
The following transfer also accepts other decoders with proved failure bounds, including
reconstruction results for independent uniform hypergraphs \citep{bresler2025recovery}.

\begin{theorem}[Heat recovery and asymptotic preservation of local information]
  \label{thm:local-heat-sufficiency}
  Let $O_n=(X_n,[B_n]_{\rm col})$ have any joint law with $y_n\in[C]$, where
  $C\ge2$ is fixed. Suppose a decoder from $G(B_n)$ recovers $[B_n]_{\rm col}$
  with failure probability at most $\varepsilon_n^{\rm rec}$.
  Let $T_n=\Psi_n(O_n,\xi_n^{\rm pr})$, where $\Psi_n$ is fixed and measurable
  and $\xi_n^{\rm pr}$ is independent of $(O_n,y_n)$. Suppose $T_n$ retains
  the features in Lemma~\ref{lem:heat-observability} and enough information to
  recover $X_n$, with joint rank/attribute-recovery failure probability at most
  $\varepsilon_n^{\rm obs}$. A decoder from $(T_n,\xi_n^{\rm pr})$ then recovers
  $O_n$ with error at most $\varepsilon_n=\varepsilon_n^{\rm rec}+
    \varepsilon_n^{\rm obs}$. For $\varepsilon_n\le1/2$,
  \begin{equation}
    0\le\mathcal R^*_{\rm heat,n}-\mathcal R^*_{\rm loc,n}
    \le h_2(\varepsilon_n)+\varepsilon_n\log C=:d_n,
    \label{eq:heat-local-bayes-gap}
  \end{equation}
  where $\mathcal R^*_{\rm heat,n}=H(y_n\mid T_n,\xi_n^{\rm pr})$,
  $\mathcal R^*_{\rm loc,n}=H(y_n\mid O_n)$, and
  $h_2(u)=-u\log u-(1-u)\log(1-u)$, with $0\log0=0$.
  Hence $d_n=O(\varepsilon_n\log(1/\varepsilon_n))\to0$ when
  $\varepsilon_n\to0$.
\end{theorem}

\begin{proof}
  On simultaneous success, Equation~\eqref{eq:heat-inverse} recovers $P$,
  its positive off-diagonal support gives $G(B_n)$, and the incidence and attribute
  decoders give $O_n$. The union bound proves the error claim. Data processing
  gives the nonnegative entropy gap. If $E$ indicates decoder failure, $O_n$ is
  determined by $(T_n,\xi_n^{\rm pr},E=0)$, so
  \[
    I(y_n;O_n\mid T_n,\xi_n^{\rm pr})
    \le H(E)+\Pr(E=1)\log C
    \le h_2(\varepsilon_n)+\varepsilon_n\log C.
  \]
  Independence of the seed and deterministic generation of $T_n$ identify this
  mutual information with the entropy gap; Bayes cross-entropy equals conditional entropy.
\end{proof}

\paragraph{Why incidence recovery needs an assumption.}
On six vertices, $\mathcal E_1=\{123,145,246,356\}$ and
$\mathcal E_2=\{124,135,236,456\}$ have degree two and identical pair co-incidences.
Thus $B_1B_1^\top=B_2B_2^\top$ and $P_1=P_2=B_1B_1^\top/6$, despite different
column-permutation classes. Identical attributes give identical heat features at every
time. A uniform mixture, padded with common isolated nodes, has reconstruction error
at least $1/2$. Duplicating all columns also preserves $P$ for multihypergraphs.
Additional heat channels cannot remove either ambiguity.

\paragraph{Recovery at a queried node.}
For a simple hypergraph $\mathcal E$ with edge sizes at least two, let
$V_{\rm amb}(\mathcal E)$ contain endpoints of pairs lying in two distinct edges and
the three cycle vertices of each Berge triangle. Put
$p_{n,v}^{\rm query}=\Pr(v_*=v\mid\mathcal E_n)$ for the queried vertex $v_*$.

\begin{proposition}[Query-weighted recovery and prediction risk]
  \label{prop:query-recovery-risk}
  For $v\notin V_{\rm amb}(\mathcal E)$, its incident edge set satisfies
  \begin{equation}
    \widehat{\mathcal E}(v)=
    \operatorname{unique}\{\{v,u\}\cup(N_G(v)\cap N_G(u)):u\in N_G(v)\}
    =\mathcal E(v),
    \label{eq:queried-incidence-decoder}
  \end{equation}
  where $\mathcal E(v)=\{e\in\mathcal E:v\in e\}$ and unique removes repeats.
  Under any joint edge law, failure at $v_*$ has probability at most
  \begin{equation}
    \varepsilon_n^{\rm query}
    =\mathbb E\sum_{v=1}^n p_{n,v}^{\rm query}
    \mathbf1\{v\in V_{\rm amb}(\mathcal E_n)\},
    \label{eq:query-weighted-failure}
  \end{equation}
  equal to $\mathbb E|V_{\rm amb}(\mathcal E_n)|/n$ for an independent uniform query.
  More generally, let $O_n$ include all permitted local information and the query,
  and let $S_n^{\rm task}$ be measurable from $O_n$. If
  $T_n=\Psi_n(O_n,\xi_n^{\rm pr})$ for fixed measurable $\Psi_n$ and a seed
  independent of $(O_n,y_n)$, $y_n\in[C]$, and $(T_n,\xi_n^{\rm pr})$ recovers
  $S_n^{\rm task}$ with error at most $\delta_n\le1/2$, then
  \begin{equation}
    0\le H(y_n\mid T_n,\xi_n^{\rm pr})-H(y_n\mid O_n)
    \le I(y_n;O_n\mid S_n^{\rm task})+h_2(\delta_n)+\delta_n\log C.
    \label{eq:task-recovery-risk}
  \end{equation}
\end{proposition}

\begin{proof}
  For $v\notin V_{\rm amb}$ each pair $vu$ has a unique source edge $e$.
  An external common neighbor would create either a repeated pair through $v$
  or a Berge triangle with cycle vertex $v$, exactly as in
  Proposition~\ref{prop:dependent-hypergraph-recovery}. Thus the decoder returns
  precisely all incident edges, including the empty set at an isolated vertex.
  Conditioning on $\mathcal E_n$ proves the failure bound.
  The failure-indicator argument above gives
  $I(y_n;S_n^{\rm task}\mid T_n,\xi_n^{\rm pr})
    \le h_2(\delta_n)+\delta_n\log C$.
  Use $H(y_n\mid T_n,\xi_n^{\rm pr})\le H(y_n\mid S_n^{\rm task})+
    I(y_n;S_n^{\rm task}\mid T_n,\xi_n^{\rm pr})$, subtract $H(y_n\mid O_n)$,
  and apply measurability and data processing.
\end{proof}

If $S_n^{\rm task}$ contains the query, its incident edges and their member attributes
(including its own when isolated), the rank and attribute assumptions give
$\delta_n\le\varepsilon_n^{\rm query}+\varepsilon_n^{\rm obs}$. Only the probability
of querying an affected vertex must vanish; $I(y_n;O_n\mid S_n^{\rm task})$ separately
charges label information beyond this neighborhood. No locality of the dense heat
operator is assumed. For known two-node simple graphs, projection recovers every edge,
including those in cycles.

For a concrete separation, put $m_n\to\infty$, $m_n/n\to0$ disjoint six-node examples
among $n$ vertices, choose their incidences independently and uniformly, and put a fixed
incidence forest elsewhere. With fixed noninformative attributes, an encoder independent
of these choices, and a retained query index, all $2^{m_n}$ incidences have identical
heat features, even with independent probes. Full reconstruction succeeds with probability
at most $2^{-m_n}$, yet $\varepsilon_n^{\rm query}\le6m_n/n$. Label each block vertex
by its block choice and forest vertices by zero. Full incidence determines the labels;
for a uniform query the optimal full-matrix heat cross-entropy is
$(6m_n/n)\log2\to0$.

These are information bounds for the complete local heat matrix. Its decoder need not
be expressible by the implemented finite encoder and rowwise classifier. The next bound
retains that approximation error, shared-parameter cost, and training error; neither
private probes nor a reconstruction stage is added to the implementation or its payload.

\subsection{Stochastic Federated Optimization and Risk Transfer}
\label{app:stochastic-convergence}

\paragraph{Population risk and shared parameters.}
Let a random test node belong to client $k$ with probability $w_k^{\rm pop}$, and write
$\mathcal R_k(\theta)=\mathbb E[-\log q_{\theta,k}(y)]$,
$\mathcal R_{\rm fed}(\theta)=\sum_kw_k^{\rm pop}\mathcal R_k(\theta)$ for the actual
prediction rule. Evaluation weights may differ from FedAvg weights
$w_k^{\rm tr}=n_k^{\rm tr}/\sum_jn_j^{\rm tr}$. All posterior risks include the same
permitted training-side information and exclude held-out labels. If conditioning on a
training transcript, include the realized probe seed in both local and heat information:
the representation is then deterministic given these data, and recovery bounds must hold
under the conditional law. Unconditional probe independence need not survive this
conditioning. Expectations also average over training and algorithmic randomness.

\begin{theorem}[From local information to shared-model population risk]
  \label{thm:client-population-risk}
  Let $\Theta$ be a nonempty shared parameter domain with finite risks and
  $\widehat\theta\in\Theta$. For each client suppose the optimal local and heat
  risks satisfy $0\le\mathcal R^*_{{\rm heat},k}-\mathcal R^*_{{\rm loc},k}\le d_k$,
  using Theorem~\ref{thm:local-heat-sufficiency} or
  Proposition~\ref{prop:query-recovery-risk}. Define
  \begin{align}
    a_k&=\left[\inf_{\theta\in\Theta}\mathcal R_k(\theta)
      -\mathcal R^*_{{\rm heat},k}\right]_+,\nonumber\\
    a_{\rm share}
    &=\inf_{\theta\in\Theta}\sum_kw_k^{\rm pop}\mathcal R_k(\theta)
    -\sum_kw_k^{\rm pop}\inf_{\theta\in\Theta}\mathcal R_k(\theta)\ge0.
    \label{eq:shared-parameter-cost}
  \end{align}
  Here $[x]_+=\max\{x,0\}$; $a_k$ includes finite-encoder and rowwise-readout
  restrictions, and $a_{\rm share}$ charges shared parameters. For fixed empirical
  training cross-entropy $J_N$ on $N$ labels, suppose
  \begin{equation}
    \mathbb E\sup_{\theta\in\Theta}|J_N(\theta)-\mathcal R_{\rm fed}(\theta)|\le u_N,
    \qquad
    \mathbb E[J_N(\widehat\theta)-\inf_{\theta\in\Theta}J_N(\theta)]\le e_{N,T}.
    \label{eq:population-risk-deviations}
  \end{equation}
  Then, without independence across clients,
  \begin{equation}
    \mathbb E\!\left[\mathcal R_{\rm fed}(\widehat\theta)
      -\sum_kw_k^{\rm pop}\mathcal R^*_{{\rm loc},k}\right]
    \le\mathbb E\!\left[\sum_kw_k^{\rm pop}(d_k+a_k)+a_{\rm share}\right]
    +2u_N+e_{N,T}.
    \label{eq:client-population-risk}
  \end{equation}
  Expectations include any conditioning information on which the risks depend.
\end{theorem}

\begin{proof}
  Two uniform-discrepancy bounds give
  $\mathbb E\mathcal R_{\rm fed}(\widehat\theta)
    \le\mathbb E\inf_\Theta\mathcal R_{\rm fed}+2u_N+e_{N,T}$.
  Substitute the definition of $a_{\rm share}$ and
  $\inf_\Theta\mathcal R_k-\mathcal R^*_{{\rm loc},k}\le a_k+d_k$.
  For transcript-conditioned risks apply these inequalities conditionally and average.
\end{proof}

The random-test-node objective needs weighted recovery errors, not simultaneous recovery
at every client. A single parameter sequence approaching all positively weighted client
optima suffices for $a_{\rm share}=0$; separate local approximation does not.
If cross-entropies lie in $[0,M]$, $u_N$ is at most the training-weighted uniform sampling
deviation plus $M\|w^{\rm tr}-w^{\rm pop}\|_1/2$. Dependent hypergraphs supply no automatic
independent-node concentration rate.

\paragraph{Actual displacements and task progress.}
Let $d_{k,u,s}=\theta_{k,u,s}-\theta_{k,u,s+1}$ be the realized displacement at local
step $s$ of round $u$, with $S_{k,u}$ steps. Full-participation FedAvg gives
\begin{equation}
  D_u=\theta^u-\theta^{u+1}
  =\sum_kw_k^{\rm tr}\sum_{s=0}^{S_{k,u}-1}d_{k,u,s}.
  \label{eq:actual-fedavg-displacement}
\end{equation}
Masks, minibatches, dropout, clipping, optimizer moments, weight decay, numerical heat
evaluation, and target motion enter through these displacements. Different step counts
are allowed; subsampling uses normalized participating-client weights.

\begin{proposition}[Task-risk bound along the implemented trajectory]
  \label{prop:actual-update-risk}
  Let $J_N$ be evaluation-mode empirical cross-entropy with exact heat, and let
  $L_u\ge0$ bound its largest Hessian eigenvalue on
  $\{\theta^u-aD_u:0\le a\le1\}$. For every prefix of every finite realized trajectory,
  \begin{equation}
    J_N(\theta^t)\le J_N(\theta^0)-\sum_{u<t}
    \left\{\langle\nabla J_N(\theta^u),D_u\rangle
    -\tfrac12L_u\|D_u\|_2^2\right\}.
    \label{eq:actual-task-risk}
  \end{equation}
  A positive bracket guarantees strict reduction that round. The inequality holds
  simultaneously for all prefixes, including any validation-selected checkpoint;
  its right side minus $\inf_\Theta J_N$, followed by expectation, bounds $e_{N,T}$.
\end{proposition}

\begin{proof}
  Taylor's integral remainder is
  $\int_0^1(1-a)D_u^\top\nabla^2J_N(\theta^u-aD_u)D_u\,da
    \le L_u\|D_u\|_2^2/2$; sum over rounds. This is pathwise and requires neither
  unbiased gradients, a fixed teacher, nor gradient domination. The native evaluation
  network's linear maps, GELU, softplus, matrix exponential, positive-epsilon LayerNorm,
  and cross-entropy are smooth, giving finite curvature on each finite segment.
  The scale model also uses LayerNorm and a smooth closed-form heat operator.
  Predictive anti-collapse losses affect $D_u$ but are absent from $J_N$;
  baseline BatchNorm buffers are outside this proof.
\end{proof}

If implemented logits differ from exact logits by at most $\epsilon_u^{\rm logit}$ in
infinity norm at checkpoint $u$, the $2$-Lipschitz property of cross-entropy adds only
$2\epsilon_0^{\rm logit}+2\epsilon_t^{\rm logit}$ to the implemented-loss bound.
Gradients and curvatures refer to exact heat; intermediate numerical errors enter $D_u$.

\paragraph{Classifier, Adam, and target identities.}
Split parameters into context $\vartheta$, classifier $\theta_{\rm cls}$, and masked
predictor $\theta_{\rm pred}$. The predictive loss does not use the classifier, so
\begin{equation}
  \nabla_{\theta_{\rm cls}}\mathcal L_{\rm pred,k}=0,\qquad
  \nabla_{\theta_{\rm cls}}\mathcal L_k
  =\nabla_{\theta_{\rm cls}}\mathcal L_{\rm CE,k}.
  \label{eq:classifier-task-gradient}
\end{equation}
Joint clipping can rescale this gradient, and context updates change its inputs.
Suppress client/round indices, let $g_s$ be the realized stochastic gradient,
$c_s^{\rm clip}\in[0,1]$ its clipping multiplier, $\ell_s>0$ the learning rate,
and $\lambda_{\rm wd}\ge0$ weight decay. Coupled Adam uses
$z_s=c_s^{\rm clip}g_s+\lambda_{\rm wd}\theta_s$. From zero moments,
\begin{equation}
  \widehat m_s=z_s+\sum_{j<s}
  \frac{(1-\beta_1)\beta_1^{s-j}}{1-\beta_1^{s+1}}(z_j-z_s),\qquad
  d_s=\ell_s\operatorname{diag}(\sqrt{\widehat v_s}+\epsilon_{\rm Adam})^{-1}
  \widehat m_s,
  \label{eq:adam-actual-displacement}
\end{equation}
where $\widehat v_s=(1-\beta_2)\sum_{j\le s}\beta_2^{s-j}z_j^{\odot2}/
(1-\beta_2^{s+1})$, $\beta_1,\beta_2\in[0,1)$, $\epsilon_{\rm Adam}>0$, and
$\odot2$ is coordinatewise squaring. AdamW instead sets $z_s=c_s^{\rm clip}g_s$
and adds $\ell_s\lambda_{\rm wd}\theta_s$ to $d_s$.
The native optimizer resets each round. The target also resets to
$\bar\vartheta_0=\vartheta_0$, and post-step EMA gives
\begin{equation}
  \bar\vartheta_s=\vartheta_0-\sum_{j<s}(1-\mu^{s-j})d_j^{\rm ctx},\qquad
  \vartheta_s-\bar\vartheta_s=-\sum_{j<s}\mu^{s-j}d_j^{\rm ctx}.
  \label{eq:round-reset-teacher}
\end{equation}
These exact identities enter Equation~\eqref{eq:actual-task-risk}, even when moving-target
updates are not gradients of one fixed loss. EMA controls lag to the current context,
not distance to a fixed reference teacher. Neither these identities nor the architecture
guarantees task alignment; general Adam can fail even on convex problems
\citep{DBLP:conf/iclr/ReddiKK18}.

\paragraph{Conditional convergence through the affine output layer.}
Write $\theta=(\theta_{\rm rest},\theta_{\rm out})$, where $\theta_{\rm out}$ consists
of the final affine classification weights and bias. With $\theta_{\rm rest}$ fixed,
evaluation cross-entropy is convex in this block, which also satisfies
Equation~\eqref{eq:classifier-task-gradient}. This applies to both FedASTRA implementations
and the optional deeper classifier, whose preceding layers belong to $\theta_{\rm rest}$.

\begin{theorem}[Global task-risk convergence along actual updates]
  \label{thm:actual-global-convergence}
  Use $J_N,D_u,L_u$ from Proposition~\ref{prop:actual-update-risk}, let
  $J_N^*=\inf_{\theta\in\Theta}J_N(\theta)>-\infty$ and
  $E_u=J_N(\theta^u)-J_N^*$. Fix $0\le u_0<T$. For every $u=u_0,\ldots,T$,
  assume a comparator with actual and comparison parameters in $\Theta$ satisfying
  \begin{equation}
    J_N(\theta_{\rm rest}^u,\theta_{\rm out}^{u,\circ})\le J_N^*+a_{\rm ctx},
    \qquad
    \|\theta_{\rm out}^u-\theta_{\rm out}^{u,\circ}\|_2\le D_{\rm cls},
    \label{eq:actual-affine-adequacy}
  \end{equation}
  with $a_{\rm ctx}\ge0$ and $0<D_{\rm cls}<\infty$. This assumption concerns
  affine-head adequacy on each realized representation, separately from full-model
  approximation over $\Theta$. Define
  $\Delta_u=\langle\nabla J_N(\theta^u),D_u\rangle-L_u\|D_u\|_2^2/2$.
  Suppose $\gamma_u>0$, $\zeta_u\ge0$ satisfy
  \begin{equation}
    \Delta_u\ge\gamma_u\|\nabla_{\theta_{\rm out}}J_N(\theta^u)\|_2^2-\zeta_u,
    \qquad u=u_0,\ldots,T-1.
    \label{eq:actual-task-alignment}
  \end{equation}
  Then, pathwise,
  \begin{equation}
    E_{u+1}\le E_u-
    \frac{\gamma_u}{D_{\rm cls}^2}(E_u-a_{\rm ctx})_+^2+\zeta_u.
    \label{eq:actual-global-recurrence}
  \end{equation}
  If all deficits $\zeta_u$ vanish, then
  \begin{equation}
    E_T\le a_{\rm ctx}+r_{u_0,T},\qquad
    r_{u_0,T}:=\left[(E_{u_0}-a_{\rm ctx})_+^{-1}
    +D_{\rm cls}^{-2}\sum_{u=u_0}^{T-1}\gamma_u\right]^{-1},
    \label{eq:actual-global-inverse-rate}
  \end{equation}
  with $r_{u_0,T}=0$ if $(E_{u_0}-a_{\rm ctx})_+=0$.
  In particular, $\gamma_u\ge\gamma>0$ gives
  $r_{u_0,T}\le D_{\rm cls}^2/[\gamma(T-u_0)]$. For general deficits,
  \begin{equation}
    \min_{u_0\le u<T}(E_u-a_{\rm ctx})_+
    \le D_{\rm cls}\sqrt{
    \frac{E_{u_0}+\sum_{u=u_0}^{T-1}\zeta_u}{\sum_{u=u_0}^{T-1}\gamma_u}}.
    \label{eq:actual-global-deficit-rate}
  \end{equation}
\end{theorem}

\begin{proof}
  Affine-head convexity and Equation~\eqref{eq:actual-affine-adequacy} give
  \[
    E_u-a_{\rm ctx}\le
    \langle\nabla_{\theta_{\rm out}}J_N(\theta^u),
    \theta_{\rm out}^u-\theta_{\rm out}^{u,\circ}\rangle
    \le D_{\rm cls}\|\nabla_{\theta_{\rm out}}J_N(\theta^u)\|_2.
  \]
  Combining with $E_{u+1}\le E_u-\Delta_u$ proves the recurrence.
  For zero deficits, $E_u$ is nonincreasing. Put
  $\delta_u=(E_u-a_{\rm ctx})_+$ and $\kappa_u=\gamma_u/D_{\rm cls}^2$.
  Whenever $\delta_{u+1}>0$,
  $\delta_{u+1}\le\delta_u(1-\kappa_u\delta_u)$ with $\kappa_u\delta_u<1$,
  hence $\delta_{u+1}^{-1}\ge\delta_u^{-1}+\kappa_u$.
  A zero $\delta_u$ stays zero. Summing proves the reciprocal rate.
  For arbitrary deficits, summing the recurrence and using $E_T\ge0$ gives
  $D_{\rm cls}^{-2}\sum_u\gamma_u\delta_u^2\le E_{u_0}+\sum_u\zeta_u$;
  bounding the minimum by the weighted mean proves the last claim.
\end{proof}

The coefficients $\gamma_u$ measure realized task progress and are not Adam learning
rates. For any chosen $\gamma_u>0$, the deficit
$\zeta_u=[\gamma_u\|\nabla_{\theta_{\rm out}}J_N(\theta^u)\|_2^2-\Delta_u]_+$
makes the deficit bound applicable. The $O((T-u_0)^{-1})$ conclusion additionally requires
adequate realized contexts, zero deficits, and progress coefficients bounded away from zero.
The architecture supplies affine-head convexity; it does not establish these trajectory
conditions. No unbiased-gradient or full-network Polyak--Lojasiewicz condition is used.

\begin{corollary}[Global population convergence of the returned model]
  \label{cor:actual-global-population}
  Assume Theorems~\ref{thm:client-population-risk} and
  \ref{thm:actual-global-convergence}, with $\zeta_u=0$ on the stated interval.
  For the actual validation-selected $\widehat\theta_{N,T}$ define
  $s_{N,T}^{\rm sel}=[J_N(\widehat\theta_{N,T})-J_N(\theta^T)]_+$, and put
  $A_N=\mathbb E[\sum_kw_k^{\rm pop}(d_k+a_k)+a_{\rm share}]+2u_N$. Then
  \begin{equation}
    \mathbb E\!\left[\mathcal R_{\rm fed}(\widehat\theta_{N,T})
      -\sum_kw_k^{\rm pop}\mathcal R^*_{{\rm loc},k}\right]
    \le A_N+\mathbb E[a_{\rm ctx}+r_{u_0,T}+s_{N,T}^{\rm sel}].
    \label{eq:actual-global-population-rate}
  \end{equation}
  Expectations include all path-dependent quantities. If $\gamma_u\ge\gamma>0$
  and $\mathbb ED_{\rm cls}^2$ is uniformly bounded, the reciprocal term is
  $O((T-u_0)^{-1})$ in expectation. Vanishing $A_N$, $\mathbb Ea_{\rm ctx}$,
  $\mathbb Er_{u_0,T}$, and $\mathbb Es_{N,T}^{\rm sel}$ along a joint data-size
  and training-horizon limit give local-information Bayes consistency.
\end{corollary}

\begin{proof}
  By definition, $J_N(\widehat\theta_{N,T})-J_N^*\le E_T+s_{N,T}^{\rm sel}$.
  Insert Equation~\eqref{eq:actual-global-inverse-rate} into
  Theorem~\ref{thm:client-population-risk}; the stated rate and limit follow.
\end{proof}

The selection term covers the implementation's validation Accuracy/Macro-F1 rule.
Nonzero alignment deficits retain only Equation~\eqref{eq:actual-global-deficit-rate}'s
best-iterate guarantee, unless selection is controlled separately. These conditional rates
do not follow from finite experiments or from the choice of Adam, EMA, and FedAvg alone.

\paragraph{Predictive-surrogate mismatch.}
The same decomposition for one predictor class makes objective mismatch explicit.
\begin{corollary}[Conditional local Bayes consistency]
  \label{cor:local-bayes-consistency}
  Let $d_n$ be from Theorem~\ref{thm:local-heat-sufficiency}, and let $\mathcal Q_N$
  be an available predictor class measurable from retained heat information, with
  population cross-entropy $R(q)$ and approximation error
  $a_N=\inf_{q\in\mathcal Q_N}R(q)-\mathcal R^*_{\rm heat,n}\ge0$.
  Suppose $\mathbb E\sup_{q\in\mathcal Q_N}|\widehat R_N(q)-R(q)|\le u_N$ and
  the returned $\widehat q_{N,T}\in\mathcal Q_N$ satisfies
  \begin{equation}
    \mathbb E[\widehat R_N(\widehat q_{N,T})
      -\inf_{q\in\mathcal Q_N}\widehat R_N(q)]\le\kappa_N+e_{N,T},
    \label{eq:surrogate-calibration-defect}
  \end{equation}
  where $\kappa_N,e_{N,T}\ge0$ bound surrogate/readout mismatch and optimization.
  Then
  \begin{equation}
    \mathbb ER(\widehat q_{N,T})-\mathcal R^*_{\rm loc,n}
    \le d_n+a_N+2u_N+\kappa_N+e_{N,T}.
    \label{eq:full-local-risk-transfer}
  \end{equation}
  Thus $\varepsilon_n,a_N,u_N,\kappa_N,e_{N,T}\to0$ imply local Bayes consistency.
\end{corollary}

\begin{proof}
  Two uniform-deviation bounds yield
  $\mathbb ER(\widehat q_{N,T})\le\inf_{\mathcal Q_N}R+2u_N+\kappa_N+e_{N,T}$;
  substitute $a_N$ and Equation~\eqref{eq:heat-local-bayes-gap}.
  For a fixed-teacher joint objective
  $J_N=\widehat R_N+\lambda_{\rm pred}\widehat L_{\rm pred,N}$, let $q$ index
  the complete model, including auxiliary heads. If a comparator $q_N^\circ$ has
  CE gap at most $\kappa_N^{\rm CE}$ and predictive-loss gap at most
  $\kappa_N^{\rm pred}$, both relative to their infima over $\mathcal Q_N$, an
  $e_{N,T}$-optimal solution of $J_N$ satisfies
  Equation~\eqref{eq:surrogate-calibration-defect} with
  $\kappa_N=\kappa_N^{\rm CE}+\lambda_{\rm pred}\kappa_N^{\rm pred}$.
  This follows by subtracting the predictive-loss lower bound. It requires a fixed
  surrogate; moving-teacher updates do not automatically optimize that objective.
\end{proof}

Information recovery controls $d_n$; the finite representation/readout controls $a_N$;
the data law and training procedure must control $u_N,\kappa_N,e_{N,T}$. Under label
sufficiency in Section~\ref{sec:information-comparison}, the local Bayes limit equals
the enriched-information limit, and Equation~\eqref{eq:main-strict-local-criterion}
gives the criterion for beating a particular fitted enriched predictor.

\section{Experimental Protocols and Implementation Details}
\label{app:protocols}

\subsection{FedASTRA Implementation}
\label{app:algorithm}

\paragraph{Network and heat parameterization.}
The context encoder applies a linear map, layer normalization, and GELU to the raw attributes,
followed by encoder dropout. Let its width be $d$. 
The training forward pass uses the type-marginalized binary incidence matrix from Equation~\eqref{eq:local_transition}. Relation types are not used as model inputs, and repeated relation instances remain separate columns.
The standard experiments use $R=4$, while Trivago uses $R=8$. The time increments in Equation~\eqref{eq:ordered_heat_times} are initialized at zero, giving initial diffusion times close to $b\log 2$, with $\epsilon_{\rm ord}=10^{-4}$. 
The classifier receives $(R+1)d$ coordinates, applies layer normalization and dropout, and uses one linear
readout. The predictive head maps $(R+1)d$ to $2d$, applies layer normalization and GELU, and
maps back to $(R+1)d$. For each local update, two independent Bernoulli masks remove attribute
coordinates. The removal probability is $0.40$ on the established tasks and $0.30$ on Trivago,
with retained values divided by $0.60$ and $0.70$, respectively. Incidence indices, relation
columns, and node membership remain fixed across both online views and the target view. 
At the start of each round, the EMA target context is initialized from the broadcast context encoder, including the attribute map and ordered-time increments.
The target is updated after each local step with momentum $\mu=0.99$ on the standard experiments and $\mu=0.996$ on Trivago. It remains client-local and is discarded at the end of each round.

\paragraph{Mask variance and target memory.}
Write the coordinate-removal probability as $p<1$ and let $\delta_j\overset{\rm
    iid}{\sim}\operatorname{Bernoulli}(1-p)$. The inverted mask $\mathcal M_p(x)=\delta\odot
  x/(1-p)$ obeys exactly
\begin{equation}
  \mathbb E_\delta\mathcal M_p(x)=x,
  \qquad
  \mathbb E_\delta\|\mathcal M_p(x)-x\|_2^2
  =\frac{p}{1-p}\|x\|_2^2.
  \label{eq:mask_variance}
\end{equation}
Thus, the masking variance is determined directly by $p$. 
Within one client round, write $\vartheta_s=\operatorname{ctx}(\theta_s)$ and
$\bar\vartheta_s$ for the online and target context states after local step
$s$, with the round reset $\bar\vartheta_0=\vartheta_0$. The update
$\bar\vartheta_s=\mu\bar\vartheta_{s-1}+(1-\mu)\vartheta_s$ expands over the
complete local context history as
\begin{equation}
  \bar\vartheta_s=\mu^s\vartheta_0+(1-\mu)
  \sum_{j=1}^{s}\mu^{s-j}\vartheta_j .
  \label{eq:ema_memory}
\end{equation}
The EMA therefore weights earlier context states geometrically by $\mu^{s-j}$. The corresponding memory scale is approximately $(1-\mu)^{-1}$ local steps.
Both masking and EMA updates are client-local.  
For every fixed $p<1$ and $\mu\in[0,1)$, coupling their randomness preserves the complete observation invariance in Corollary~\ref{cor:completion-invariance}.  
In the masked-prediction ablation, $p$ and $\mu$ are held fixed across paired conditions so that only the predictive component changes.

\paragraph{Exact local objective.}
Each local update encodes two masked views online. At client $k$, let
$\vartheta_k=\operatorname{ctx}(\theta_k)$ and let $\bar\vartheta_k=(\bar\theta_{k,\rm
    enc},\bar\eta_k)$ denote the round-local EMA copy of the complete context encoder, including
its learned time increments. 
The target branch forms $Z_k^{\rm
  tar}=W_k(X_k\mid\bar\vartheta_k)$, and each masked online view produces a prediction $\widehat Z_k^{(a)}=p_{\theta_k}(W_k(\widetilde
  X_k^{(a)}\mid\vartheta_k))$ with $a\in\{1,2\}$. 
Here the row-wise predictor is
$p_{\theta_k}:\mathbb R^{(R+1)d}\to\mathbb R^{(R+1)d}$ and the row-wise classifier for $C$
classes is $c_{\theta_k}:\mathbb R^{(R+1)d}\to\mathbb R^C$.
Let $n_k=|V_k|$ and we use $\operatorname{LN}$ for row-wise layer
normalization, $\operatorname{sg}$ for stop gradient, and $\operatorname{cos}$ for
cosine similarity. 
Set $\mathcal A_k^{(a)}=n_k^{-1}\sum_{v\in V_k}
  [1-\operatorname{cos}(\operatorname{LN}(\widehat Z_{k,v}^{(a)}),
    \operatorname{sg}(\operatorname{LN}(Z_{k,v}^{\rm tar})))]$. 
With an anti-collapse regularizer $\mathcal R_{\rm ac}$ inspired by variance--invariance--covariance regularization (VICReg), and a fixed view-agreement weight $\gamma_{\rm view}\ge 0$, we define
\begin{align}
  \mathcal L_{\mathrm{pred},k}
  &=\frac12\sum_{a=1}^{2}[\mathcal A_k^{(a)}+
  \mathcal R_{\rm ac}(\operatorname{LN}(\widehat Z_k^{(a)}))]
  +\gamma_{\rm view}\!\left[1-\frac1{n_k}\sum_{v\in V_k}
    \operatorname{cos}(\widehat Z_{k,v}^{(1)},
    \widehat Z_{k,v}^{(2)})\right].
  \label{eq:heat_jepa_loss}
\end{align}
For $A\in\mathbb R^{n\times h}$, define
$\overline A_j=n^{-1}\sum_{i=1}^nA_{ij}$ and
$\operatorname{Var}_n(A_{\cdot j})=
  \sum_{i=1}^n(A_{ij}-\overline A_j)^2/\max(n-1,1)$. Thus the implementation
uses sample variance for $n>1$ and the zero variance for $n=1$. Define
\begin{align}
  \mathcal V(A)
  &=\frac1h\sum_{j=1}^{h}
  \left[1-\sqrt{\operatorname{Var}_n(A_{\cdot j})+10^{-4}}\right]_+,
  \label{eq:heat_variance_penalty}\\
  \mathcal C(A)
  &=\frac{1}{h}\left\|
  \operatorname{offdiag}\!\left(
  \frac{\widetilde A^\top\widetilde A}{\max(n-1,1)}
  \right)\right\|_F^2,
  \qquad
  \widetilde A_{\cdot j}=
  \frac{A_{\cdot j}-\overline A_j}
  {\sqrt{\operatorname{Var}_n(A_{\cdot j})+10^{-4}}}.
  \label{eq:heat_covariance_penalty}
\end{align}
The anti-collapse term in Equation~\eqref{eq:heat_jepa_loss} is
$\mathcal R_{\rm ac}(A)=0.10\mathcal V(A)+0.005\mathcal C(A)$ and receives
$A=\operatorname{LN}(\widehat Z_k^{(a)})$, the rowwise layer-normalized
predictor matrix defined
before Equation~\eqref{eq:heat_jepa_loss}.  The term $\mathcal A_k^{(a)}$
averages the nodewise target-alignment cosine losses.  The
view-agreement weight is $\gamma_{\rm view}=0.10$. The fixed predictive coefficient is
$\lambda_{\rm pred}=1$.  Supervised cross-entropy is evaluated only on the
training mask. Prediction uses all client-local nodes, attributes, and incidence.
The validation-selected
checkpoint uses Accuracy with Macro-F1 tie-breaking and is tested once.

\paragraph{Shared matrix powers across diffusion times.}
The required powers of $P_k$ applied to $U_k$ are computed recursively and shared across all diffusion times, with time-specific Poisson coefficients.
Let $L_{\rm tail}$ be the recurrence degree chosen by
Equation~\eqref{eq:poisson_tail_certificate}, let $\mathcal E_k^{\rm rel}$ be the retained
relation columns, and let $\Omega_k=\{(v,e):B_{k,ve}=1\}$ be the client incidence list. One application of $P_k$ costs $O(|\Omega_k|d)$ time and $O((|V_k|+|\mathcal E_k^{\rm rel}|)d)$
working storage. 
The complete basis computation therefore has the following time and storage
complexity:
\begin{equation}
  O\!\left(L_{\rm tail}|\Omega_k|d+
  L_{\rm tail}R|V_k|d\right)
  \quad\text{time},\qquad
  O\!\left((R+1)|V_k|d+|\mathcal E_k^{\rm rel}|d+|\Omega_k|\right)
  \quad\text{storage}.
  \label{eq:heat_complexity}
\end{equation}
Increasing $R$
adds coefficient accumulation and representation width, but does not require additional incidence traversals for each channel.  
The current maximum learned time and fixed
tolerance determine the degree $L_{\rm tail}$. Although the target heat operator has
infinite Poisson support, Equation~\eqref{eq:tail-collapsed-evaluator} evaluates it
at finite degree with certified forward and time-gradient errors. 

Table~\ref{tab:heat-runtime-residual} reports the recurrence degree, numerical error, and local evaluation cost under the finite-task implementation.
The evaluated client has 445 nodes, 485 anonymous relation columns, 774 incidences, and 64
feature channels. Each CPU median is computed over five timing runs. Across all settings, the maximum absolute residual is $2.38\times10^{-7}$. The results show that the heat-basis computation meets the prescribed error tolerance, while maintaining small numerical residuals and moderate local computation cost.

\begin{table}[H]
  \centering
  \caption{Heat-basis evaluation on one CoraCA client at $K=6$. $L_{\rm tail}$ is selected using the $10^{-6}$ tail criterion in Equation~\eqref{eq:poisson_tail_certificate}, and residuals are measured against a
reference evaluation with tolerance $10^{-12}$. CPU timings are reported in milliseconds (ms) on identical local inputs.}
  \label{tab:heat-runtime-residual}
  \scriptsize
  \setlength{\tabcolsep}{3.5pt}
  \begin{tabular}{cccccc}
    \toprule
    $R$ & $L_{\rm tail}$ & Forward tail & Time-gradient tail & Relative $\ell_2$ residual & CPU median (ms) \\
    \midrule
    2 & 11 & $2.95{\times}10^{-8}$ & $2.87{\times}10^{-7}$ & $3.07{\times}10^{-8}$ & 26.30 \\
    4 & 15 & $4.35{\times}10^{-8}$ & $2.98{\times}10^{-7}$ & $2.26{\times}10^{-8}$ & 33.72 \\
    6 & 18 & $9.34{\times}10^{-8}$ & $5.26{\times}10^{-7}$ & $1.48{\times}10^{-8}$ & 33.27 \\
    8 & 21 & $1.06{\times}10^{-7}$ & $5.34{\times}10^{-7}$ & $2.20{\times}10^{-8}$ & 42.99 \\
    \bottomrule
  \end{tabular}
\end{table}

\paragraph{Communication protocol and privacy boundary.}
Table~\ref{tab:protocol-payload-scope} summarizes the information available to the server and the state retained locally. Under this protocol, FedASTRA communicates only the online model parameters, while auxiliary training states remain client-local.
FedAvg aggregates client updates with weights proportional to their fixed labeled-training counts $n_k^{\mathrm{tr}}$.  The local EMA target is reset at each round, while learned node states, incidence data, and numerical heat buffers remain client-local. 
This communication protocol does not by itself provide end-to-end differential privacy. Model updates may still reveal information about attributes or labels, and secure aggregation alone does not eliminate the risk of aggregate inversion attacks \citep{shi2025scalemia}. Record-level privacy would require an explicit threat model together with additional mechanisms such as calibrated randomization, privacy accounting, and compatible secure aggregation.

\begin{table}[H]
  \centering
  \scriptsize
  \setlength{\tabcolsep}{4pt}
  \renewcommand{\arraystretch}{1.08}
  \caption{Server-visible information and client-local state for FedASTRA.}
  \label{tab:protocol-payload-scope}
  \begin{tabularx}{\textwidth}{@{}lX@{}}
    \toprule
    Protocol item & Declared scope \\
    \midrule
    Server receives & Online model parameters from each participating client. \\
    \hline
    Server broadcasts & The aggregated online model parameters. \\
     \hline
    Fixed aggregation metadata & Each client's labeled-training count, fixed before training. \\
     \hline
    client-local state & Raw attributes, labels, incidence tables, relation membership, relation identifiers, round-local optimizer state, heat powers,
    latent states, and the local EMA target context encoder reset at each round. \\
     \hline
    Auxiliary learned-state payload & Zero bytes.  No embeddings, relation summaries, or target
    parameters are transmitted. \\
     \hline
    Structural access claim & Direct cross-client node, edge, and source-hyperedge
    identifiers are absent. \\
    \hline
    Privacy extension & Secure aggregation and differential privacy require their own mechanism, parameters, and accountant. \\
    \bottomrule
  \end{tabularx}
\end{table}

\subsection{Graph Partitions and Lifts}
\label{app:partitions}

\paragraph{Main graph partition.}
For the main experiments, following ~\citep{yao2023fedgcn,luo2025fedhgl}, we fix the task split before client assignment and partition nodes across $K\in \{3,6,9\}$ clients using class-conditioned Dirichlet sampling. Client allocation probabilities are drawn from a symmetric Dirichlet distribution with concentration $10^4$, subject to the declared load cap and minimum client size. The resulting assignment is precomputed from the full class labels, fixed before model training, and shared by all compared methods. For graph datasets, this node partition is also fixed before constructing client-local one-hop relations from the induced subgraphs.

\paragraph{Truncation-controlled graph partition.}
We construct low-, reference-, and high-truncation partitions while preserving each client's load and label histogram. Each paired comparison begins from one balanced label allocation and permits swaps between nodes of the same class. Every proposal therefore preserves each client's
load and label histogram. With nonnegative fixed weights $\alpha_2,\alpha_3$ and
$\rho_k=n_k/\sum_jn_j$, the three objectives are
\begin{equation}
  \begin{aligned}
    J_{\rm low}(z)&=\sum_{q\in\{2,3\}}\alpha_q\tau_q(z),\\
    J_{\rm ref}(z)&=\sum_{q\in\{2,3\}}\alpha_q
    \left(\tau_q(z)-\left[1-\sum_k\rho_k^q\right]\right)^2,\\
    J_{\rm high}(z)&=-\sum_{q\in\{2,3\}}\alpha_q\tau_q(z),
  \end{aligned}
  \quad\text{s.t.}\quad
  |V_k|=n_k,\qquad \operatorname{hist}(Y_{V_k})=h_k.
  \label{eq:partition_objective}
\end{equation}
Here, $1-\sum_k\rho_k^q$ is the expected order-$q$ truncation level under independent allocation with client proportions $\rho_k$. We minimize the corresponding objective for the low-truncation, independent-allocation reference, and high-truncation partitions.  Each pass proposes
2,048 same-class cross-client swaps, accepts strict improvements, and stops after at most 16 passes or after a pass accepts zero swaps.  The three partitions thus
have identical client loads and client-by-label histograms while their truncation coordinates differ.  Hyperedges are treated as sets, relation multiplicity and type are retained, and every objective is invariant to stored member order. We assign equal weight to the pair and triple coordinates, $\alpha_2=\alpha_3=1$.  

\paragraph{Graph lifting.}
For graph datasets, we adopt a partition-then-lift protocol. Each client first receives the induced subgraph $G[V_k]$, from which one-hop relations are constructed without remote endpoints or cross-client edges. The following proposition clarifies its relation to the alternative lift-then-partition construction.
\begin{proposition}[Lift-then-partition and partition-then-lift protocols]
  \label{prop:graph_lift}
  Let the rooted global one-hop relation be $e_v=\{v\}\cup N_G(v)$.  For a local
  root $v\in V_k$, the two graph-lift protocols obey the exact identity below for
  every client and every locally assigned root under one-hop lifting:
  \begin{equation}
    e_v\cap V_k=\{v\}\cup N_{G[V_k]}(v).
  \end{equation}
  Thus a local-root one-hop fragment has the same members under both
  constructions.  Lift-then-partition additionally creates fragments on client
  $k$ for roots $v\notin V_k$.  Partition-then-lift contains locally rooted
  columns.  For multi-hop relations, even a local root can differ because a path
  may leave $V_k$ and return.  This identity locates the extra globally
  constructed information in lift-then-partition and isolates globally supplied
  fragments from edges removed by the induced client graph.
\end{proposition}

\subsection{Dataset Details}
\label{app:datasets_relation_channels}

\begin{table}[t]
\centering
\scriptsize
\caption{Statistics of the datasets used in the experiments.}
\label{tab:dataset_statistics}
\resizebox{\linewidth}{!}{
\begin{tabular}{lcccccccc}
\toprule
 & CoraCA & DBLP4k & IMDB4k & Trivago
 & Cora & CiteSeer & Facebook & ogbn-papers100M \\
\midrule
\#Nodes       & 2,708 & 4,057 & 4,278 & 172,738 & 2,708 & 3,327 & 22,470 & 111,059,956 \\
\#Graph edges       & -- & -- & -- & -- & 5,278 & 4,552 & 85,424 & 1,614,062,356 \\
\#Hyperedges  & 1,072 & 22,051 & 7,338 & 233,202 & 2,708 & 3,327 & 22,470 & 111,059,956 \\
\#Incidences  & 4,585 & 110,330 & 17,106 & 726,861 & 13,264 & 12,431 & 193,318 & 3,339,184,668 \\
\#Classes     & 7 & 4 & 3 & 160 & 7 & 6 & 4 & 172 \\
\#Features    & 1,433 & 334 & 3,066 & 300 & 1,433 & 3,703 & 4,714 & 128 \\
\bottomrule
\end{tabular}
}
\end{table}

We use two types of datasets for evaluation, including native hypergraph datasets and graph datasets. For graph datasets, \#Graph edges counts distinct non-self undirected edges used for the one-hop lift, with duplicate and reciprocal entries counted once. Table~\ref{tab:dataset_statistics} summarizes the statistics of the datasets used in the experiments.
\paragraph{Native hypergraph datasets.}
CoraCA, DBLP4k, and IMDB4k are loaded from the public dataset release accompanying HGNN and HGNN+ \citep{feng2019hypergraph,gao2023hgnnplus}. 
CoraCA uses coauthorship relations \citep{yadati2019hypergcn}. 
DBLP4k retains paper and term relations, and IMDB4k retains director and actor relations as separate channels \citep{sun2011pathsim,fu2020magnn}. 
Trivago follows the public DHG-Bench release \citep{li2026dhgbench}, where browsing sessions define hyperedges. 
Distinct source relation instances are retained even when they contain the same set of nodes. 
FedASTRA combines all released incidence columns in one observable operator.  Relation types are used only to define analysis strata and are not provided as categorical embeddings to the model.
\paragraph{Graph datasets.}
The graph-derived experiments use the Cora and CiteSeer paper citation networks
\citep{yang2016revisiting}, the Facebook page mutual-like network \citep{rozemberczki2021multiscale}, and the large-scale ogbn-papers100M paper citation network \citep{hu2020ogb}.
Following the partition-then-lift protocol in Appendix~\ref{app:partitions}, nodes are first assigned to clients and each local root $v\in V_k$ then defines the closed one-hop relation $\{v\}\cup(N_G(v)\cap V_k)$. Thus, every retained incidence is visible to its assigned client through the common released hyperedge. The ogbn-papers100M experiment follows its official implementation and processes all nodes and citations without sampling.

\subsection{Experimental Configurations}
\label{app:federated-protocol}

\paragraph{Training and evaluation protocols.}
Within each comparison, we use the same data splits, released inputs, communication rounds, local updates, and validation-selection rule. 
Following \citep{luo2025fedhgl}, training fractions are $0.10$, $0.06$, and $0.06$ for CoraCA, DBLP4k, and IMDB4k, respectively, with validation and test fractions of $0.20$ and $0.40$. The remaining unlabeled nodes retain their structural roles and, together with labeled nodes, provide attributes and released local structure to the transductive encoders. Trivago uses a per-class stratified $0.50/0.25/0.25$ train/validation/test split that assigns every node once following DHG-Bench \citep{li2026dhgbench}. For graph datasets, Cora and CiteSeer use the $0.10/0.20/0.40$ ratios, while Facebook uses $0.008/0.20/0.40$, following \citep{luo2025fedhgl}. For all datasets except ogbn-papers100M, the splits are stratified by class and fixed before client assignment. Ogbn-papers100M uses its official temporal split \cite{hu2020ogb}. Training labels are used for model fitting, validation performance is used for checkpoint selection, and the selected model is evaluated once on the test set. 
The standard experiments use 150 communication rounds with three local updates per round, while Trivago uses 30 rounds. Results use five repetitions, except for Fed-HNHN and Trivago experiments, which use three. The ogbn-papers100M scale experiment uses 20 communication rounds with one local epoch per round and is reported from a single full-data run.

Node classification is evaluated using Accuracy and Macro-F1. Macro-F1 averages the per-class F1 scores over the full label set, assigning zero to any class absent from the test split. 
The evaluator also computes $\tau_2$ and $\tau_3$ from the source hypergraph and the fixed experimental partition. 
These quantities measure pair- and triple-order truncation and are used only for evaluation, not for model fitting, checkpoint selection, or inference. For summaries over \(K\in\{3,6,9\}\), we first average the results across the three client counts within each repetition, and then report the mean and sample standard deviation across repetitions.

\paragraph{Baseline configurations.}
For the native-hypergraph comparison, all baselines use the same client-local attributes and anonymous incidence columns as FedASTRA, with independently indexed columns across clients. Only trainable model parameters are communicated. We report baseline results only when the dataset, task split, client assignment, observed structure, communication rounds, local updates, and evaluation protocol match those of FedASTRA. Fed-HNHN follows the FedHGL implementation with bias-free linear maps, mean aggregation, and hidden BatchNorm~\citep{luo2025fedhgl}. Its BatchNorm running statistics remain client-local and are not included in FedAvg.

For the graph comparison, all methods use the same split, client assignment, communication rounds, local updates, and evaluation protocol. FedSage, FedGCN (0-hop), and FedASTRA use only client-local structure. FedSage+ exchanges generated boundary features and reconstruction gradients without direct cross-client identifiers~\citep{zhang2021fedsage}. FedGCN (2-hop) relies on cross-client neighborhood information, using pre-aggregated neighbor features in the original protocol~\citep{yao2023fedgcn}, while FedCog additionally exchanges intermediate representations through known cross-client node and edge correspondences~\citep{lei2023fedcog}. These different information-access settings are retained when interpreting the predictive results.

For the FedASTRA configuration, Table~\ref{tab:heat-finite-config} summarizes the settings used in the standard experiments and on Trivago. The order $q$ is used only for the identification analysis and is not a model hyperparameter. The corresponding evaluator queries do not affect model training or checkpoint selection.

\begin{table}[H]
  \centering
  \scriptsize
  \caption{Configurations of FedASTRA in the standard experiments and on Trivago.}
  \label{tab:heat-finite-config}
  \setlength{\tabcolsep}{3pt}
  \begin{tabularx}{\textwidth}{@{}lXX@{}}
    \toprule
     & Standard experiments & Trivago \\
    \midrule
    Hyperedge input & Type-marginalized binary incidence, unit weights & Type-marginalized binary incidence, unit weights \\
    \hline
    Online model & Hidden dimension 64, dropout $0.35$ & Hidden dimension 128, dropout $0.20$ \\
    \hline
    Optimization & Adam, learning rate $10^{-2}$, weight decay $5\times10^{-4}$, gradient clip $5.0$ & Adam, learning rate $3\times10^{-3}$, weight decay $5\times10^{-4}$, gradient clip $5.0$ \\
    \hline
    Heat basis & $R=4$, increment floor $10^{-4}$, tail tolerance $10^{-6}$ & $R=8$, increment floor $10^{-4}$, tail tolerance $10^{-6}$ \\
    \hline
    Predictive path & Two masks with probability $0.40$, $\lambda_{\rm pred}=1$, EMA decay $0.99$ & Two masks with probability $0.30$, $\lambda_{\rm pred}=1$, EMA decay $0.996$ \\
    \bottomrule
  \end{tabularx}
\end{table}

\paragraph{Ablation settings.}
For the adaptive diffusion ablation, we evaluate $R\in\{0,1,2,4,8\}$ with $\lambda_{\rm pred}=1$ on CoraCA, DBLP4k, and Trivago at $K=6$. 
The $R=0$ control retains the feature encoder, classifier, masked predictor, and EMA target but removes diffusion. 
Three paired repetitions use the same node splits, client assignments, and training seeds across all five settings, giving 45 fitted models in total. 
Because increasing $R$ also increases the representation dimension and model size, we report the corresponding parameter count for each setting.
The parameter count includes the feature encoder, learned time increments, classifier, and predictive head, while excluding the client-local EMA target.
Let $d_x$ be input-feature dimension, $d$ the hidden dimension, and $C$ the number of classes. For the single-layer linear classifier in our implementation, the number of online parameters is
\begin{equation}
  N_R=d_xd+4d^2+(C+12)d+C+Rd(4d+C+4).
  \label{eq:r-width-parameter-count}
\end{equation}
The feature encoder contributes $d_xd+3d$, the time increments $Rd$, the classifier $(R+1)d(C+2)+C$, and the predictive head $4(R+1)d^2+(R+1)d+6d$. 
This accounting provides the model-size context for the adaptive diffusion ablation. For the masked prediction ablation, we evaluate $R\in\{1,4\}$ and $\lambda_{\rm pred}\in\{0,1\}$ on CoraCA and DBLP4k at $K=6$.

To measure the computation cost of the adaptive diffusion ablation, each dataset is run on one NVIDIA RTX A6000 GPU, with repetitions and $R$ settings executed sequentially. 
Elapsed time is measured with CUDA synchronization and includes model initialization, all training rounds, validation-based checkpoint selection, checkpoint restoration, and final post-training measurements, while excluding dataset loading, split construction, and the final test evaluation.
Figure~\ref{fig:r-width-complete} and Table~\ref{tab:r-width-complete} report predictive performance, parameter counts, and elapsed training time under this common measurement protocol.

\paragraph{Communication accounting.}
Let $b_\theta$ denote the size in bytes of the online model parameters and $K_u=|\mathcal C_u|$ be the number of participating clients in round $u$. FedASTRA transmits one model download and one model upload per participating client, giving
\begin{equation}
  b_{{\rm round},u}=K_ub_\theta+K_ub_\theta=2K_ub_\theta,
  \qquad b_{\rm auxiliary}=0.
  \label{eq:heat_communication}
\end{equation}
Model communication is reported as the logical payload of these uploads and downloads. Serialization, encryption, packet headers, retransmission, and other transport overheads are excluded for all methods. 
For the federated graph baselines, auxiliary communication accounts for additional method-specific information required by the corresponding implementation, including generated features and reconstruction gradients for FedSage+, two-hop neighborhood information for FedGCN (2-hop), and routed intermediate representations for FedCog. Model communication is reported per round, while auxiliary communication is accumulated over the full run.

\paragraph{Full-data ogbn-papers100M protocol without sampling.}
To support full-data training at this scale, ogbn-papers100M is processed with streaming minibatches. Each round traverses every training root once, and a 200,000-incidence budget controls batch size without truncating any rooted relation. Each minibatch performs one optimizer step. 
For FedASTRA, AdamW \citep{loshchilov2019decoupled} uses a learning rate of $2\times10^{-3}$, weight decay $10^{-4}$, dropout $0.2$, and gradient-norm clipping at $2.0$. The model uses hidden dimension 197, $R=4$, mask probability $0.40$,  $\lambda_{\rm pred}=1$, and EMA decay $0.99$. The hidden dimension is selected to match the approximately one-million-parameter baseline budget. It gives 976,504 trainable parameters, while dimension 198 would give 985,420 and exceed the 984,442-parameter ceiling.

In the minibatch implementation, each rooted example is represented independently. If a paper appears in multiple rooted relations, a separate copy is created for each relation, preventing message passing between different prediction examples. Each copied node therefore belongs to exactly one relation column, and different relation columns have disjoint support. This gives $D_v=I$ and $B^\top B=D_e$, so the occurrence-space operator $P_{\rm occ}=BD_e^{-1}B^\top$ is an orthogonal projection:
\begin{equation}
  P_{\rm occ}^{2}
  =BD_e^{-1}(B^\top B)D_e^{-1}B^\top
  =P_{\rm occ}.
  \label{eq:papers-projection-proof}
\end{equation}
The heat kernel can therefore be evaluated exactly for every minibatch as
\begin{equation}
  \mathsf K_{\rm occ}(t)=e^{-t}I+(1-e^{-t})P_{\rm occ}.
  \label{eq:papers-heat-closed-form}
\end{equation}
\begin{corollary}[Exact spectral form in the occurrence protocol]
  \label{cor:papers-projection-spectrum}
  For every scalar multiplier $f$ on
  $\operatorname{spec}(P_{\rm occ})\subseteq\{0,1\}$,
  \begin{equation}
    f(P_{\rm occ})=f(0)I+[f(1)-f(0)]P_{\rm occ}.
    \label{eq:papers-projection-calculus}
  \end{equation}
  Consequently, for any $t>0$,
  $\operatorname{span}\{I,\mathsf K_{\rm occ}(t)\}
    =\operatorname{span}\{I,P_{\rm occ}\}$. Therefore, the identity channel and any one positive-time heat channel span the space of all linear spectral filters of $P_{\rm occ}$.
\end{corollary}

The corollary follows from the spectral projectors $I-P_{\rm occ}$ and $P_{\rm occ}$, together
with $1-e^{-t}>0$. 
Additional learned times provide separately parameterized feature channels before nonlinear normalization and prediction, although they do not enlarge the linear spectral space of $P_{\rm occ}$. 
The closed form in Equation~(\ref{eq:papers-heat-closed-form}) relies on the disjoint-copy minibatch construction.
If node copies are shared across rooted relations, the resulting operator need not satisfy $P^2=P$. 
The general client-local transition in Equation~\eqref{eq:local_transition} need not satisfy $P_k^2=P_k$ and is therefore evaluated
using the recurrence in Appendix~\ref{app:proof-heat-stability}. 
In float32, the scale implementation gives $\max|P_{\rm occ}^2-P_{\rm occ}|=0$ and a maximum
absolute difference of $1.91\times10^{-6}$ between Equation~\eqref{eq:papers-heat-closed-form}
and a direct matrix exponential. 

\section{Additional Experimental Results}
\label{app:results}

This section provides additional experimental results by reporting extended performance evaluations, controlled analyses of the proposed mechanisms, and detailed scalability measurements under the corresponding experimental protocols.

\subsection{Performance on Native Hypergraphs}
\label{app:results-native}
Table \ref{tab:appendix-heat-native-by-client} reports the complete native-hypergraph results for each client count $K\in\{3,6,9\}$. On CoraCA, the performance gap between FedASTRA and Fed-HGNN increases as the number of clients grows from 3 to 9, from 4.83\% to 7.75\% accuracy and from 6.54\% to 9.84\% Macro-F1. This trend is consistent with the increasing fragmentation of local observations under matched client inputs.

\begin{table}[H]
  \centering
  \caption{Complete native-hypergraph results under different number of clients.}
  \label{tab:appendix-heat-native-by-client}
  \scriptsize
  \setlength{\tabcolsep}{3pt}
  \begin{tabular}{lcccccc}
    \toprule
    Method & \multicolumn{2}{c}{$K=3$} & \multicolumn{2}{c}{$K=6$} & \multicolumn{2}{c}{$K=9$} \\
    & Acc. (\%) & Macro-F1 (\%) & Acc. (\%) & Macro-F1 (\%) & Acc. (\%) & Macro-F1 (\%) \\
    \midrule
    \multicolumn{7}{l}{\textit{CoraCA}} \\
    FedASTRA & 76.27 $\pm$ 0.54 & 74.33 $\pm$ 0.72 & 74.97 $\pm$ 0.63 & 72.79 $\pm$ 0.93 & 74.27 $\pm$ 0.73 & 71.52 $\pm$ 0.97 \\
    Fed-HGNN & 71.44 $\pm$ 1.90 & 67.79 $\pm$ 2.12 & 67.97 $\pm$ 1.38 & 63.57 $\pm$ 1.50 & 66.52 $\pm$ 2.32 & 61.68 $\pm$ 3.02 \\
    Fed-HNHN & 57.30 $\pm$ 0.82 & 54.40 $\pm$ 1.21 & 54.16 $\pm$ 2.46 & 49.08 $\pm$ 3.29 & 52.74 $\pm$ 1.13 & 47.77 $\pm$ 2.57 \\
    Fed-HyperGCN & 67.38 $\pm$ 1.01 & 63.08 $\pm$ 1.64 & 63.84 $\pm$ 2.16 & 57.65 $\pm$ 2.56 & 60.85 $\pm$ 2.56 & 53.47 $\pm$ 3.07 \\
    Fed-AllDeepSets & 45.90 $\pm$ 3.43 & 37.62 $\pm$ 4.62 & 36.54 $\pm$ 2.11 & 24.37 $\pm$ 2.86 & 36.45 $\pm$ 1.45 & 19.26 $\pm$ 6.62 \\
    Fed-AllSetTransformer & 30.22 $\pm$ 0.00 & 6.63 $\pm$ 0.00 & 30.22 $\pm$ 0.00 & 6.63 $\pm$ 0.00 & 30.22 $\pm$ 0.00 & 6.63 $\pm$ 0.00 \\
    Fed-UniGCNII & 63.57 $\pm$ 1.89 & 56.52 $\pm$ 5.53 & 58.89 $\pm$ 1.52 & 49.01 $\pm$ 3.67 & 56.16 $\pm$ 2.07 & 44.80 $\pm$ 3.52 \\
    Fed-ED-HNN & 61.90 $\pm$ 1.51 & 57.27 $\pm$ 1.53 & 59.28 $\pm$ 2.38 & 52.11 $\pm$ 3.62 & 57.69 $\pm$ 2.22 & 50.46 $\pm$ 3.45 \\
    \midrule
    \multicolumn{7}{l}{\textit{DBLP4k}} \\
    FedASTRA & 78.52 $\pm$ 1.26 & 77.84 $\pm$ 1.04 & 79.75 $\pm$ 0.82 & 79.15 $\pm$ 0.75 & 79.40 $\pm$ 1.35 & 78.84 $\pm$ 1.24 \\
    Fed-HGNN & 77.62 $\pm$ 1.48 & 76.98 $\pm$ 1.39 & 77.73 $\pm$ 1.62 & 77.00 $\pm$ 1.37 & 77.55 $\pm$ 1.30 & 76.96 $\pm$ 1.16 \\
    Fed-HNHN & 72.23 $\pm$ 2.64 & 71.44 $\pm$ 2.66 & 68.43 $\pm$ 2.55 & 67.70 $\pm$ 2.67 & 68.08 $\pm$ 1.23 & 67.22 $\pm$ 0.97 \\
    Fed-HyperGCN & 74.74 $\pm$ 1.29 & 73.43 $\pm$ 1.63 & 72.29 $\pm$ 0.60 & 70.88 $\pm$ 0.90 & 71.04 $\pm$ 1.85 & 70.06 $\pm$ 1.80 \\
    Fed-AllDeepSets & 39.25 $\pm$ 0.68 & 31.74 $\pm$ 1.79 & 39.36 $\pm$ 0.72 & 31.96 $\pm$ 2.36 & 39.45 $\pm$ 1.18 & 33.63 $\pm$ 1.75 \\
    Fed-AllSetTransformer & 30.87 $\pm$ 0.97 & 22.50 $\pm$ 2.61 & 30.82 $\pm$ 1.20 & 21.30 $\pm$ 1.64 & 31.28 $\pm$ 1.73 & 21.62 $\pm$ 1.93 \\
    Fed-UniGCNII & 79.03 $\pm$ 2.12 & 77.85 $\pm$ 2.47 & 77.60 $\pm$ 2.49 & 76.18 $\pm$ 2.98 & 75.44 $\pm$ 3.10 & 73.92 $\pm$ 3.27 \\
    Fed-ED-HNN & 71.92 $\pm$ 3.52 & 71.19 $\pm$ 3.46 & 72.30 $\pm$ 2.30 & 71.59 $\pm$ 2.40 & 72.09 $\pm$ 2.36 & 71.30 $\pm$ 2.37 \\
    \midrule
    \multicolumn{7}{l}{\textit{IMDB4k}} \\
    FedASTRA & 51.99 $\pm$ 1.85 & 51.92 $\pm$ 1.53 & 51.05 $\pm$ 1.53 & 50.88 $\pm$ 1.34 & 50.34 $\pm$ 0.92 & 50.37 $\pm$ 0.85 \\
    Fed-HGNN & 50.14 $\pm$ 1.38 & 50.10 $\pm$ 1.39 & 48.36 $\pm$ 1.47 & 48.25 $\pm$ 1.46 & 48.05 $\pm$ 1.49 & 47.41 $\pm$ 2.11 \\
    Fed-HNHN & 46.14 $\pm$ 0.62 & 44.75 $\pm$ 0.57 & 43.81 $\pm$ 2.13 & 41.84 $\pm$ 2.95 & 44.92 $\pm$ 1.93 & 42.54 $\pm$ 1.81 \\
    Fed-HyperGCN & 47.36 $\pm$ 1.87 & 46.33 $\pm$ 2.22 & 43.26 $\pm$ 1.50 & 41.35 $\pm$ 2.30 & 41.81 $\pm$ 1.64 & 34.90 $\pm$ 5.71 \\
    Fed-AllDeepSets & 45.82 $\pm$ 1.87 & 43.96 $\pm$ 4.83 & 37.73 $\pm$ 1.67 & 30.24 $\pm$ 1.83 & 37.42 $\pm$ 1.35 & 28.15 $\pm$ 3.67 \\
    Fed-AllSetTransformer & 37.04 $\pm$ 0.03 & 18.04 $\pm$ 0.05 & 37.06 $\pm$ 0.05 & 18.06 $\pm$ 0.10 & 37.03 $\pm$ 0.11 & 18.21 $\pm$ 0.22 \\
    Fed-UniGCNII & 51.12 $\pm$ 2.34 & 49.88 $\pm$ 3.45 & 37.09 $\pm$ 0.87 & 21.95 $\pm$ 5.33 & 36.66 $\pm$ 0.87 & 19.87 $\pm$ 3.96 \\
    Fed-ED-HNN & 48.93 $\pm$ 1.98 & 48.61 $\pm$ 2.06 & 40.75 $\pm$ 0.94 & 30.57 $\pm$ 1.29 & 40.71 $\pm$ 0.93 & 30.58 $\pm$ 1.22 \\
    \midrule
    \multicolumn{7}{l}{\textit{Trivago}} \\
    FedASTRA & 33.55 $\pm$ 1.60 & 9.64 $\pm$ 0.35 & 27.70 $\pm$ 1.14 & 7.55 $\pm$ 0.32 & 26.01 $\pm$ 1.29 & 6.70 $\pm$ 0.15 \\
    Fed-HGNN & 22.99 $\pm$ 0.72 & 5.00 $\pm$ 0.71 & 22.08 $\pm$ 0.77 & 4.05 $\pm$ 0.59 & 20.48 $\pm$ 0.87 & 3.75 $\pm$ 0.67 \\
    Fed-HNHN & 29.10 $\pm$ 1.36 & 9.19 $\pm$ 0.19 & 24.87 $\pm$ 1.06 & 7.37 $\pm$ 0.35 & 23.08 $\pm$ 0.70 & 6.59 $\pm$ 0.24 \\
    Fed-HyperGCN & 9.40 $\pm$ 0.65 & 0.15 $\pm$ 0.06 & 9.00 $\pm$ 0.02 & 0.11 $\pm$ 0.01 & 9.10 $\pm$ 0.21 & 0.13 $\pm$ 0.04 \\
    Fed-AllDeepSets & 11.37 $\pm$ 0.11 & 0.30 $\pm$ 0.01 & 10.82 $\pm$ 0.15 & 0.28 $\pm$ 0.03 & 10.45 $\pm$ 0.22 & 0.24 $\pm$ 0.01 \\
    Fed-AllSetTransformer & 8.99 $\pm$ 0.00 & 0.10 $\pm$ 0.00 & 8.99 $\pm$ 0.00 & 0.10 $\pm$ 0.00 & 8.99 $\pm$ 0.00 & 0.10 $\pm$ 0.00 \\
    Fed-UniGCNII & 9.25 $\pm$ 0.23 & 0.16 $\pm$ 0.05 & 9.14 $\pm$ 0.09 & 0.16 $\pm$ 0.01 & 9.07 $\pm$ 0.10 & 0.13 $\pm$ 0.03 \\
    Fed-ED-HNN & 9.54 $\pm$ 0.86 & 0.16 $\pm$ 0.09 & 9.05 $\pm$ 0.06 & 0.13 $\pm$ 0.03 & 8.99 $\pm$ 0.01 & 0.10 $\pm$ 0.00 \\
    \bottomrule
  \end{tabular}
\end{table}

\subsection{Performance Against Baselines with Varying Information Access}
\label{app:results-graph}

Table \ref{tab:appendix-heat-graph-by-client} reports comparisons against federated graph learning baselines under different cross-client information-access settings, while Table \ref{tab:heat-graph-access-references} summarizes methods according to their available structural information and communication requirements. FedASTRA consistently outperforms identifier-free baselines across client counts and also surpasses FedSage+ despite the latter introducing additional exchanged information. In contrast, FedGCN (2-hop) and FedCog access explicit cross-client correspondence and therefore exploit a strictly richer observation setting. These results suggest that additional cross-client information is not universally beneficial: its value depends on whether the introduced information can be reliably exploited under the correspondence assumptions.

\begin{table}[H]
  \centering
  \caption{Complete results under different cross-client information-access settings.}
  \label{tab:appendix-heat-graph-by-client}
  \scriptsize
  \setlength{\tabcolsep}{3pt}
  \begin{tabular}{lcccccc}
    \toprule
    Method & \multicolumn{2}{c}{$K=3$} & \multicolumn{2}{c}{$K=6$} & \multicolumn{2}{c}{$K=9$} \\
    & Acc. (\%) & Macro-F1 (\%) & Acc. (\%) & Macro-F1 (\%) & Acc. (\%) & Macro-F1 (\%) \\
    \midrule
    \multicolumn{7}{l}{\textit{Cora}} \\
    FedSage & 70.35 $\pm$ 1.96 & 67.77 $\pm$ 1.67 & 63.18 $\pm$ 0.97 & 58.59 $\pm$ 0.87 & 58.65 $\pm$ 1.85 & 52.17 $\pm$ 2.81 \\
    FedGCN (0-hop) & 74.07 $\pm$ 1.57 & 71.24 $\pm$ 1.86 & 65.90 $\pm$ 2.05 & 61.96 $\pm$ 3.34 & 62.07 $\pm$ 1.79 & 57.44 $\pm$ 2.91 \\
    FedSage+ & 69.67 $\pm$ 0.71 & 66.87 $\pm$ 0.30 & 63.11 $\pm$ 1.66 & 57.78 $\pm$ 1.63 & 61.02 $\pm$ 1.40 & 55.05 $\pm$ 3.65 \\
    FedGCN (2-hop) & 84.05 $\pm$ 1.17 & 83.01 $\pm$ 1.23 & 84.05 $\pm$ 1.16 & 83.12 $\pm$ 1.20 & 84.12 $\pm$ 0.92 & 82.92 $\pm$ 0.78 \\
    FedCog & 84.29 $\pm$ 1.50 & 83.15 $\pm$ 1.64 & 84.34 $\pm$ 1.06 & 83.35 $\pm$ 1.00 & 83.92 $\pm$ 0.78 & 82.73 $\pm$ 0.66 \\
    FedASTRA & 77.54 $\pm$ 1.08 & 76.06 $\pm$ 1.31 & 74.45 $\pm$ 1.60 & 72.13 $\pm$ 1.81 & 73.33 $\pm$ 1.22 & 71.83 $\pm$ 1.31 \\
    \midrule
    \multicolumn{7}{l}{\textit{CiteSeer}} \\
    FedSage & 65.08 $\pm$ 0.96 & 59.56 $\pm$ 1.42 & 63.73 $\pm$ 0.74 & 57.00 $\pm$ 1.51 & 62.93 $\pm$ 1.72 & 55.67 $\pm$ 1.25 \\
    FedGCN (0-hop) & 68.68 $\pm$ 1.50 & 61.60 $\pm$ 1.32 & 65.83 $\pm$ 0.56 & 59.06 $\pm$ 0.52 & 65.14 $\pm$ 1.80 & 58.10 $\pm$ 1.68 \\
    FedSage+ & 63.16 $\pm$ 1.17 & 57.75 $\pm$ 1.59 & 62.53 $\pm$ 1.66 & 56.48 $\pm$ 2.61 & 62.95 $\pm$ 1.62 & 57.69 $\pm$ 2.10 \\
    FedGCN (2-hop) & 72.20 $\pm$ 1.44 & 65.64 $\pm$ 2.26 & 72.27 $\pm$ 1.53 & 64.88 $\pm$ 1.44 & 72.38 $\pm$ 1.50 & 64.85 $\pm$ 2.07 \\
    FedCog & 72.21 $\pm$ 1.49 & 65.84 $\pm$ 1.77 & 72.38 $\pm$ 1.04 & 65.09 $\pm$ 1.03 & 72.51 $\pm$ 1.42 & 65.11 $\pm$ 1.94 \\
    FedASTRA & 67.61 $\pm$ 2.32 & 63.62 $\pm$ 1.61 & 68.00 $\pm$ 1.86 & 63.67 $\pm$ 2.07 & 67.92 $\pm$ 1.88 & 63.65 $\pm$ 2.13 \\
    \midrule
    \multicolumn{7}{l}{\textit{Facebook}} \\
    FedSage & 72.63 $\pm$ 0.73 & 69.00 $\pm$ 1.08 & 65.64 $\pm$ 1.81 & 61.85 $\pm$ 2.39 & 60.41 $\pm$ 1.61 & 53.46 $\pm$ 4.16 \\
    FedGCN (0-hop) & 75.14 $\pm$ 0.95 & 72.62 $\pm$ 0.74 & 67.87 $\pm$ 1.71 & 64.39 $\pm$ 2.24 & 64.30 $\pm$ 1.22 & 59.44 $\pm$ 2.83 \\
    FedSage+ & 68.17 $\pm$ 0.33 & 65.41 $\pm$ 0.64 & 61.85 $\pm$ 0.95 & 57.62 $\pm$ 1.32 & 58.56 $\pm$ 1.33 & 54.02 $\pm$ 1.99 \\
    FedGCN (2-hop) & 83.21 $\pm$ 0.76 & 82.02 $\pm$ 0.87 & 82.84 $\pm$ 0.50 & 81.58 $\pm$ 0.67 & 82.73 $\pm$ 0.85 & 81.51 $\pm$ 1.04 \\
    FedCog & 82.46 $\pm$ 0.69 & 81.24 $\pm$ 0.73 & 82.50 $\pm$ 0.84 & 81.20 $\pm$ 1.09 & 82.15 $\pm$ 0.94 & 80.85 $\pm$ 1.13 \\
    FedASTRA & 76.88 $\pm$ 1.44 & 75.58 $\pm$ 1.60 & 75.57 $\pm$ 0.94 & 73.71 $\pm$ 0.86 & 75.12 $\pm$ 1.49 & 73.30 $\pm$ 1.76 \\
    \bottomrule
  \end{tabular}
\end{table}

\begin{table}[H]
  \centering
  \caption{Summarized results of federated graph baselines under different cross-client information-access settings. Results are aggregated across client counts $K\in\{3,6,9\}$.}
  \label{tab:heat-graph-access-references}
  \small
  \setlength{\tabcolsep}{4pt}
  \begin{tabularx}{\textwidth}{lXcc}
    \toprule
    Method & Exchanged or accessed information & Accuracy (\%) & Macro-F1 (\%) \\
    \midrule
    \multicolumn{4}{l}{\textit{Cora}} \\
    FedASTRA & Model parameters only, with no direct cross-client identifiers & 75.11 $\pm$ 0.99 & 73.34 $\pm$ 1.13 \\
    FedSage+ & Locally generated boundary features and cross-client reconstruction gradients, with no direct identifiers & 64.60 $\pm$ 1.11 & 59.90 $\pm$ 1.71 \\
    FedGCN (2-hop) & Cross-client node IDs, edges, and pre-aggregated neighbor features & 84.07 $\pm$ 1.03 & 83.02 $\pm$ 1.04 \\
    FedCog & Cross-client node IDs, edges, and intermediate embeddings & 84.18 $\pm$ 1.04 & 83.08 $\pm$ 1.03 \\
    \addlinespace[2pt]
    \multicolumn{4}{l}{\textit{CiteSeer}} \\
    FedASTRA & Model parameters only, with no direct cross-client identifiers & 67.84 $\pm$ 1.66 & 63.64 $\pm$ 1.83 \\
    FedSage+ & Locally generated boundary features and cross-client reconstruction gradients, with no direct identifiers & 62.88 $\pm$ 1.02 & 57.31 $\pm$ 1.74 \\
    FedGCN (2-hop) & Cross-client node IDs, edges, and pre-aggregated neighbor features & 72.28 $\pm$ 1.46 & 65.12 $\pm$ 1.68 \\
    FedCog & Cross-client node IDs, edges, and intermediate embeddings & 72.37 $\pm$ 1.26 & 65.35 $\pm$ 1.44 \\
    \addlinespace[2pt]
    \multicolumn{4}{l}{\textit{Facebook}} \\
    FedASTRA & Model parameters only, with no direct cross-client identifiers & 75.86 $\pm$ 1.02 & 74.20 $\pm$ 1.07 \\
    FedSage+ & Locally generated boundary features and cross-client reconstruction gradients, with no direct identifiers & 62.86 $\pm$ 0.71 & 59.02 $\pm$ 1.06 \\
    FedGCN (2-hop) & Cross-client node IDs, edges, and pre-aggregated neighbor features & 82.92 $\pm$ 0.61 & 81.70 $\pm$ 0.72 \\
    FedCog & Cross-client node IDs, edges, and intermediate embeddings & 82.37 $\pm$ 0.73 & 81.10 $\pm$ 0.90 \\
    \bottomrule
  \end{tabularx}
\end{table}

\subsection{Validation on Proposed Analysis and Mechanism}
\label{app:results-mechanism}

Table \ref{tab:heat-ablation} isolates the contributions of the adaptive diffusion basis and the masked prediction objective. Removing masked prediction $\lambda_{\mathrm{pred}}=0$ consistently reduces both Accuracy and Macro-F1 on CoraCA and DBLP4k, indicating that the predictive alignment component provides complementary representation learning signals. Increasing the heat basis width from $R=1$ to $R=4$ further improves performance when the predictive objective is enabled, while its effect is smaller without masked prediction. These results support the complementary roles of adaptive diffusion and masked representation alignment in FedASTRA.

\begin{table}[H]
  \centering
  \caption{Ablation study of adaptive diffusion and masked prediction at $K=6$. This study is separate from the heat-width sweep in Appendix~\ref{app:protocols}.}
  \label{tab:heat-ablation}
  \small
  \begin{tabular}{llrcc}
    \toprule
    Dataset & Reported condition & Parameters & Accuracy (\%) & Macro-F1 (\%) \\
    \midrule
    CoraCA & $R=1,\ \lambda_{\mathrm{pred}}=0$ & 126,407 & 67.26 $\pm$ 1.48 & 62.94 $\pm$ 1.75 \\
    CoraCA & $R=1,\ \lambda_{\mathrm{pred}}=1$ & 126,407 & 74.27 $\pm$ 0.56 & 71.96 $\pm$ 0.58 \\
    CoraCA & $R=4,\ \lambda_{\mathrm{pred}}=0$ & 177,671 & 69.32 $\pm$ 0.69 & 65.31 $\pm$ 1.15 \\
    CoraCA & $R=4,\ \lambda_{\mathrm{pred}}=1$ & 177,671 & 74.97 $\pm$ 0.63 & 72.79 $\pm$ 0.93 \\
    DBLP4k & $R=1,\ \lambda_{\mathrm{pred}}=0$ & 55,684 & 77.54 $\pm$ 0.70 & 76.77 $\pm$ 0.78 \\
    DBLP4k & $R=1,\ \lambda_{\mathrm{pred}}=1$ & 55,684 & 79.58 $\pm$ 0.67 & 78.99 $\pm$ 0.49 \\
    DBLP4k & $R=4,\ \lambda_{\mathrm{pred}}=0$ & 106,372 & 78.23 $\pm$ 0.51 & 77.60 $\pm$ 0.56 \\
    DBLP4k & $R=4,\ \lambda_{\mathrm{pred}}=1$ & 106,372 & 79.75 $\pm$ 0.82 & 79.15 $\pm$ 0.75 \\
    \bottomrule
  \end{tabular}
\end{table}

Table~\ref{tab:heat-controlled} reports the controlled-truncation results. The controlled partitions exhibit the expected ordering of retained higher-order co-membership mass. The observed pair/triple mass is 72.44/65.11\%, 16.67/2.78\%, and 9.68/1.10\% for the low, reference, and high CoraCA partitions, respectively. The corresponding DBLP4k values are 27.21/9.24\%, 16.67/2.78\%, and 16.33/2.70\%. These results confirm that the three partition conditions differ in structural truncation while preserving client loads and label histograms across paired comparisons. The ordering applies to co-membership mass only. Prediction performance also depends on feature and label alignment, so the truncation level alone does not determine Accuracy or Macro-F1.

\begin{table}[H]
  \centering
  \caption{Controlled-truncation results at $K=6$.}
  \label{tab:heat-controlled}
  \small
  \begin{tabular}{llcc}
    \toprule
    Dataset & Reported condition & Accuracy (\%) & Macro-F1 (\%) \\
    \midrule
    CoraCA & Low-truncation & 77.99 $\pm$ 0.58 & 76.40 $\pm$ 0.74 \\
    CoraCA & Reference & 75.58 $\pm$ 1.33 & 73.27 $\pm$ 1.07 \\
    CoraCA & High-truncation & 73.97 $\pm$ 0.80 & 71.34 $\pm$ 0.80 \\
    DBLP4k & Low-truncation & 78.23 $\pm$ 1.76 & 77.54 $\pm$ 1.59 \\
    DBLP4k & Reference & 79.46 $\pm$ 1.23 & 78.85 $\pm$ 1.13 \\
    DBLP4k & High-truncation & 79.24 $\pm$ 1.49 & 78.66 $\pm$ 1.49 \\
    \bottomrule
  \end{tabular}
\end{table}

Table~\ref{tab:identification-audit} and Table~\ref{tab:identification-absolute-width}
summarize the two-client structural identification audit. The audit uses five fixed
$K=2$, $c=1$ diagnostic partitions per source, exact hyperedge types, and unit weights.
Held-out parent tokens mark cross-client rows and verify the oracle
$\Phi_q(\mathcal H)=\sum_e\binom{|e|}{q}$ after the partition is fixed. Sorting fragment
sizes gives both attainable endpoints. Labels are used only by the partition constructor,
and no learner is trained. These evaluator partitions are independent of the reported
$K\in\{3,6,9\}$ training partitions. Across all evaluated cells, no higher-order
co-membership target is point identified, and the median relative endpoint width
increases from $36.37\%$ at $q=2$ to $93.71\%$ at $q=5$. These results empirically
support the predicted ambiguity caused by latent hyperedge correspondence, especially for higher-order structural targets.

\begin{table}[H]
  \centering
  \caption{Two-client structural identification audit on four native datasets. For each
order $q$, the audit contains 20 releases and 30 eligible strata. Relative endpoint width is
$(\Phi_q^+-\Phi_q^-)/\Phi_q^+$. By Corollary~\ref{cor:learner-fiber-lower-bound}, half of this normalized width gives the
corresponding minimax lower bound for observable point estimation of the same order-$q$
mass.}
  \label{tab:identification-audit}
  \scriptsize
  \setlength{\tabcolsep}{4pt}
  \begin{tabular}{ccccccc}
    \toprule
    $q$ & Releases & Eligible strata & Point identified & Median relative width & $q$-mass floor & Oracle coverage \\
    \midrule
    2 & 20 & 30 & 0/30 & 36.37\% & 18.18\% & 20/20 \\
    3 & 20 & 30 & 0/30 & 70.56\% & 35.28\% & 20/20 \\
    4 & 20 & 30 & 0/30 & 86.62\% & 43.31\% & 20/20 \\
    5 & 20 & 30 & 0/30 & 93.71\% & 46.85\% & 20/20 \\
    \bottomrule
  \end{tabular}
\end{table}

\begin{table}[H]
  \centering
  \caption{Median absolute sharp-endpoint gaps over five two-client diagnostic partitions, measured in order-$q$ tuple units. The values preserve the original
dataset scale, complementing the normalized endpoint widths reported in Table~\ref{tab:identification-audit}.}
  \label{tab:identification-absolute-width}
  \small
  \setlength{\tabcolsep}{7pt}
  \begin{tabular}{lrrrr}
    \toprule
    Dataset & $q=2$ & $q=3$ & $q=4$ & $q=5$ \\
    \midrule
    CoraCA & 6,714 & 69,204 & 469,120 & 2,631,843 \\
    DBLP4k & 3,752,114 & 983,644,998 & $2.1971{\times}10^{11}$ & $4.2294{\times}10^{13}$ \\
    IMDB4k & 18,587 & 205,338 & 1,581,947 & 10,490,694 \\
    Trivago & 410,593 & 3,263,206 & 23,744,190 & 209,428,691 \\
    \bottomrule
  \end{tabular}
\end{table}

To further verify the correspondence-invariant property of the learned representation,
we conduct an evaluator-only permutation audit on the released relation columns. A separate CPU float32 audit permutes the 743 released relation columns of CoraCA client 0 at $K=3$ and protocol seed 0 a total of 32 times. The fixed instance has 892 nodes, 1,499
incidences, a 128-dimensional projected input, hidden width 16, and $R=4$. It uses neither labels nor training. The maximum differences in $P_k$, all heat-filtered features, and fixed-model logits are each zero, verifying the algebra used by Corollary~\ref{cor:completion-invariance} under this fixed implementation audit with shared
numerical inputs.

Table~\ref{tab:identification-misspecification} reports evaluator-only audits under relaxed observation assumptions. Type coarsening removes a matching restriction while preserving the released row inventory, and the numeric containment check confirms that
the coarsened endpoint hull contains the exact hull in all audited cells. In contrast, status removal and one-sided deletion alter the released inventory, so finite compatible
classes require additional public bounds beyond the exact equal-count endpoint formula. Training and model selection are absent from all 120 audit outcomes in this evaluator-only audit.

\begin{table}[H]
  \centering
  \caption{Evaluator-only audit of relaxed observation assumptions on CoraCA and DBLP4k. The 40 cells consist of five fixed partitions and $q\in\{2,3,4,5\}$ for each dataset.}
  \label{tab:identification-misspecification}
  \small
  \setlength{\tabcolsep}{5pt}
  \begin{tabularx}{\textwidth}{@{}lccX@{}}
    \toprule
    Release change & Cells & Numeric endpoint & Audited conclusion \\
    \midrule
    Merge observed type strata & 40 & Available & The coarsened endpoint hull contains the exact hull in 40/40 cells. \\
    Remove local/cross marks & 40 & Bound required & Finite endpoints require a public cross-row count and support bound. \\
    Delete 10\% of one-side fragments & 40 & Bound required & Finite endpoints require public missing-row count, weight, and support bounds. \\
    \bottomrule
  \end{tabularx}
\end{table}

Figure~\ref{fig:r-width-complete} and Table~\ref{tab:r-width-complete} report the heat-width
sweep over $R\in\{0,1,2,4,8\}$ on CoraCA, DBLP4k, and Trivago.
CoraCA reaches 76.68\% Accuracy and 74.35\% Macro-F1 at $R=2$, compared with 70.67\% and
67.85\% at $R=0$. DBLP4k has closely grouped means over $R=1,2,4$, within 0.09 Accuracy
points and 0.10 Macro-F1 points, followed by lower means at $R=8$. Trivago reaches its highest
Accuracy of 27.66\% at $R=4$ and its highest Macro-F1 of 7.32\% at $R=8$. The latter change
from $R=4$ adds 0.30 Macro-F1 points, with parameters increasing from 472,224 to 818,336 and
fitting time from 235.3 to 417.9 seconds.

The sweep identifies a dataset-dependent finite-capacity tradeoff. Moderate widths improve the
attribute-only control, while additional channels increase parameters and runtime and can
change the two classification metrics differently. Spectral density describes approximation by
a growing family of heat filters, but finite-sample accuracy also depends on capacity and fitting;
therefore, density alone does not impose a monotone ordering on these empirical curves.
Each width's runtime and parameter count follows the same measurement protocol described in Appendix~\ref{app:federated-protocol}.

\begin{figure}[H]
  \centering
  \includegraphics[width=\textwidth]{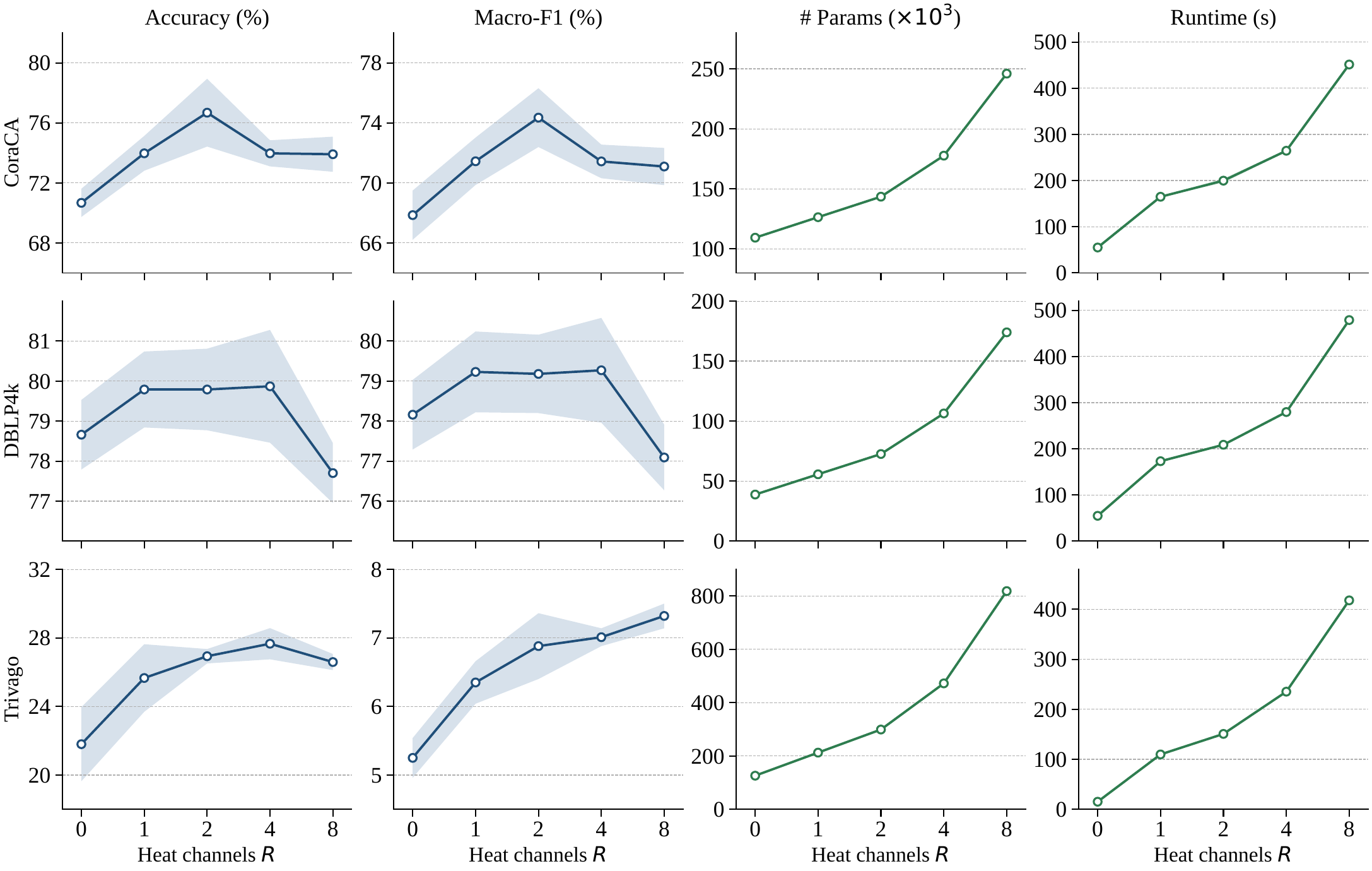}
  \caption{\textbf{Predictive utility and cost of the heat-basis width}. We fix $K=6$ and $\lambda_{\mathrm{pred}}=1$ on all three datasets. The $R=0$ control retains masked prediction and removes diffusion.}
  \label{fig:r-width-complete}
\end{figure}

\begin{table}[H]
  \centering
  \caption{Heat-basis width at $K=6$ and $\lambda_{\rm pred}=1$. Parameter counts include the online
model and predictive head. Runtime uses the synchronized fitting boundary described in Appendix~\ref{app:protocols}.}
  \label{tab:r-width-complete}
  \small
  \setlength{\tabcolsep}{4pt}
  \begin{tabular}{llrrrr}
    \toprule
    Dataset & $R$ & Accuracy (\%) & Macro-F1 (\%) & Parameters & Runtime (s) \\
    \midrule
    CoraCA & 0 & 70.67 $\pm$ 0.94 & 67.85 $\pm$ 1.64 & 109,319 & 55.0 $\pm$ 1.7 \\
    & 1 & 73.97 $\pm$ 1.16 & 71.44 $\pm$ 1.58 & 126,407 & 164.9 $\pm$ 0.3 \\
    & 2 & \textbf{76.68 $\pm$ 2.26} & \textbf{74.35 $\pm$ 1.96} & 143,495 & 199.8 $\pm$ 4.7 \\
    & 4 & 73.97 $\pm$ 0.87 & 71.43 $\pm$ 1.12 & 177,671 & 264.6 $\pm$ 1.8 \\
    & 8 & 73.91 $\pm$ 1.17 & 71.09 $\pm$ 1.24 & 246,023 & 451.5 $\pm$ 2.9 \\
    \midrule
    DBLP4k & 0 & 78.66 $\pm$ 0.87 & 78.16 $\pm$ 0.87 & 38,788 & 54.8 $\pm$ 0.4 \\
    & 1 & 79.79 $\pm$ 0.95 & 79.23 $\pm$ 1.01 & 55,684 & 173.1 $\pm$ 2.1 \\
    & 2 & 79.79 $\pm$ 1.02 & 79.18 $\pm$ 0.98 & 72,580 & 208.7 $\pm$ 2.4 \\
    & 4 & \textbf{79.87 $\pm$ 1.41} & \textbf{79.27 $\pm$ 1.31} & 106,372 & 279.6 $\pm$ 0.3 \\
    & 8 & 77.70 $\pm$ 0.75 & 77.09 $\pm$ 0.82 & 173,956 & 478.9 $\pm$ 2.8 \\
    \midrule
    Trivago & 0 & 21.80 $\pm$ 2.16 & 5.25 $\pm$ 0.29 & 126,112 & 15.2 $\pm$ 0.7 \\
    & 1 & 25.66 $\pm$ 1.97 & 6.35 $\pm$ 0.31 & 212,640 & 109.8 $\pm$ 0.1 \\
    & 2 & 26.93 $\pm$ 0.42 & 6.88 $\pm$ 0.48 & 299,168 & 150.8 $\pm$ 0.3 \\
    & 4 & \textbf{27.66 $\pm$ 0.91} & 7.01 $\pm$ 0.13 & 472,224 & 235.3 $\pm$ 0.6 \\
    & 8 & 26.59 $\pm$ 0.47 & \textbf{7.32 $\pm$ 0.18} & 818,336 & 417.9 $\pm$ 0.1 \\
    \bottomrule
  \end{tabular}
\end{table}

\subsection{Scalability on Large-scale Data}
\label{app:results-scale}

Table~\ref{tab:appendix-heat-papers100m-system} reports the full-scale evaluation on the complete graph
assets and official splits. It contains one full-data execution rather than repeated runs. It records
predictive quality, communication, wall-clock time, throughput, and peak device allocation with complete root and citation coverage for every compared method. FedASTRA achieves the highest predictive performance in this full-data setting while maintaining parameter-only communication.

\begin{table}[H]
  \centering
  \caption{Results on full-data ogbn-papers100M under the official split. Wall-clock training time covers 20 rounds and is reported in seconds (s). Throughput counts training roots per second. Peak memory is reported in gibibytes (GiB), where $1\,\mathrm{GiB}=2^{30}$ bytes. Communication counts bidirectional trainable-parameter
traffic, excluding client-local BatchNorm buffers for all methods.}
  \label{tab:appendix-heat-papers100m-system}
  \scriptsize
  \setlength{\tabcolsep}{3pt}
  \begin{tabularx}{\textwidth}{lXXXXX}
    \toprule
    Method & Parameters & \shortstack{Parameter traffic\\per round (MiB)} & Wall-clock (s) & Roots/s & Peak GPU (GiB) \\
    \midrule
    Fed-MLP & 983,380 & 60.02 & 47.2 & 511,594 & 0.58 \\
    Fed-SIGN & 984,442 & 60.09 & 47.6 & 506,887 & 0.68 \\
    Fed-HGNN & 982,672 & 59.98 & 54.6 & 442,051 & 1.52 \\
    Fed-HNHN & 34,672 & 2.12 & 312.5 & 77,259 & 0.54 \\
    Fed-HyperGCN & 984,442 & 60.09 & 406.6 & 59,386 & 8.89 \\
    Fed-GCN & 983,380 & 60.02 & 61.8 & 390,637 & 2.92 \\
    FedASTRA & 976,504 & 59.60 & 1069.7 & 22,570 & 19.53 \\
    \bottomrule
  \end{tabularx}
  \par\smallskip
  \begin{minipage}{\linewidth}
    \scriptsize
    All requested neighborhoods are evaluated without truncation.
  \end{minipage}
\end{table}

\end{document}